\documentclass[twoside]{article}

\usepackage[accepted]{aistats2025-arxiv}

\usepackage{natbib}
\usepackage[colorlinks=true, citecolor=blue]{hyperref}

\usepackage[toc,page,header]{appendix}
\usepackage{minitoc}
\usepackage{silence}
\def\safedef#1{%
   \ifx#1\undefined
      \expandafter\def\expandafter#1%
   \else
      \errmessage{The \string#1 is defined already}%
      \expandafter\def\expandafter\tmp
   \fi
}

\usepackage{amsthm,amsmath,bbm,amsfonts,amssymb}
\usepackage{dsfont} 
\usepackage{mathtools}
\usepackage{commath}
\usepackage{eqnarray}
\mathtoolsset{showonlyrefs=false}
\allowdisplaybreaks 

\usepackage{xspace}
\usepackage[T1]{fontenc}
\usepackage[utf8]{inputenc}
\usepackage[usenames,dvipsnames]{xcolor} 

\usepackage{hyperref,url}

\usepackage{wrapfig}
\usepackage[export]{adjustbox}
\usepackage{graphicx,tabularx}
\usepackage{multirow,hhline}
\usepackage{booktabs}
\usepackage{pbox} 
\usepackage{algorithm,algorithmic}

\usepackage[export]{adjustbox}

\usepackage{enumitem}
\usepackage[normalem]{ulem}

\newcommand{\kj}[1]{{\color{RedOrange}[#1]}}

\newcommand{\gray}[1]{{ \color[rgb]{.6,.6,.6} #1 }}

{\color{kjsaved}}%

\usepackage{framed}
\usepackage[most]{tcolorbox}
\definecolor{kjgray}{rgb}{.7,.7,.7}

\newtheoremstyle{kjstyle}
{1ex} 
{\topsep} 
{\itshape} 
{} 
{\bfseries} 
{.} 
{.5em} 
{} 

\newtheoremstyle{kjstyle2}
{.0em} 
{.0em} 
{\itshape} 
{} 
{\bfseries} 
{.} 
{.5em} 
{} 

\newtheoremstyle{kjstylenoitalic}
{1ex} 
{\topsep} 
{} 
{} 
{\bfseries} 
{.} 
{.5em} 
{} 

\tcbset{kjboxstyle/.style={title={},breakable,colback=white,enhanced jigsaw,boxrule=1.3pt,sharp corners,colframe=kjgray,boxsep=0pt,coltitle={black},attach title to upper={},left=.8ex,bottom=.4em}}    

\tcbset{kjboxstylec/.style={title={},breakable,colback=white,enhanced jigsaw,boxrule=1.3pt,sharp corners,colframe=kjgray,boxsep=0pt,coltitle={black},attach title to upper={},before skip=1.5ex,left=.8ex,bottom=.4em,top=0.4em,enlarge top by=-0.3em,enlarge bottom by=-0.3em}}    

\newtheorem{theorem}{Theorem}

\theoremstyle{kjstyle}
\AtBeginEnvironment{thm}{\begin{tcolorbox}[kjboxstyle]}%
  \AtEndEnvironment{thm}{\end{tcolorbox}}%
\theoremstyle{kjstyle}
\AtBeginEnvironment{lem}{\begin{tcolorbox}[kjboxstyle]}%
  \AtEndEnvironment{lem}{\end{tcolorbox}}%
\theoremstyle{kjstyle}
\AtBeginEnvironment{prop}{\begin{tcolorbox}[kjboxstyle]}%
  \AtEndEnvironment{prop}{\end{tcolorbox}}%
\theoremstyle{kjstyle}
\AtBeginEnvironment{cor}{\begin{tcolorbox}[kjboxstyle]}%
  \AtEndEnvironment{cor}{\end{tcolorbox}}%

\theoremstyle{kjstyle}
\AtBeginEnvironment{defn}{\begin{tcolorbox}[kjboxstyle]}%
  \AtEndEnvironment{defn}{\end{tcolorbox}}%

\theoremstyle{kjstyle}
\AtBeginEnvironment{ass}{\begin{tcolorbox}[kjboxstyle]}%
  \AtEndEnvironment{ass}{\end{tcolorbox}}%

\theoremstyle{kjstyle}
\AtBeginEnvironment{conj}{\begin{tcolorbox}[kjboxstyle]}%
  \AtEndEnvironment{conj}{\end{tcolorbox}}%

\theoremstyle{kjstyle}
\AtBeginEnvironment{question}{\begin{tcolorbox}[kjboxstyle]}%
  \AtEndEnvironment{question}{\end{tcolorbox}}%

\theoremstyle{kjstylenoitalic}\newtheorem{remark}{Remark}
\AtBeginEnvironment{remark}{\begin{tcolorbox}[kjboxstyle]}%
  \AtEndEnvironment{remark}{\end{tcolorbox}}%
\theoremstyle{kjstylenoitalic}
\AtBeginEnvironment{openp}{\begin{tcolorbox}[kjboxstyle]}%
  \AtEndEnvironment{openp}{\end{tcolorbox}}%

\theoremstyle{kjstylenoitalic}
\AtBeginEnvironment{ex}{\begin{tcolorbox}[kjboxstyle]}%
  \AtEndEnvironment{ex}{\end{tcolorbox}}%

\theoremstyle{kjstylenoitalic}

\usepackage{mdframed}
\usepackage{lipsum}
\definecolor{kjgray}{rgb}{.7,.7,.7}
\makeatletter



\makeatletter
\renewcommand{\paragraph}{%
  \@startsection{paragraph}{4}%
  {\z@}{0.50ex \@plus 1ex \@minus .2ex}{-1em}%
  {\normalfont\normalsize\bfseries}%
}
\makeatother

\usepackage[customcolors,shade]{hf-tikz} 

\usepackage{tabularx} 
\newcolumntype{P}[1]{>{\centering\arraybackslash}p{#1}}
\newcolumntype{M}[1]{>{\centering\arraybackslash}m{#1}}

\def\ddefloop#1{\ifx\ddefloop#1\else\ddef{#1}\expandafter\ddefloop\fi}

\def\ddef#1{\expandafter\def\csname #1#1\endcsname{\ensuremath{\mathbb{#1}}}}
\ddefloop ABCDFGHIJKLMNORSTUWXYZ\ddefloop 

\def\ddef#1{\expandafter\def\csname c#1\endcsname{\ensuremath{\mathcal{#1}}}}
\ddefloop ABCDEFGHIJKLMNOPQRSTUVWXYZ\ddefloop

\def\ddef#1{\expandafter\def\csname b#1\endcsname{\ensuremath{{\mathbf{#1}}}}}
\ddefloop ABCDEFGHIJKLMNOPQRSTUVWXYZ\ddefloop  
\def\ddef#1{\expandafter\def\csname b#1\endcsname{\ensuremath{{\boldsymbol{#1}}}}}
\ddefloop abcdeghijklmnopqrtsuvwxyz\ddefloop  

\def\ddef#1{\expandafter\def\csname h#1\endcsname{\ensuremath{\hat{#1}}}}
\ddefloop ABCDEFGHIJKLMNOPQRSTUVWXYZabcdefghijklmnopqrsuvwxyz\ddefloop 
\def\ddef#1{\expandafter\def\csname hc#1\endcsname{\ensuremath{\hat{\mathcal{#1}}}}}
\ddefloop ABCDEFGHIJKLMNOPQRSTUVWXYZ\ddefloop
\def\ddef#1{\expandafter\def\csname hb#1\endcsname{\ensuremath{\hat{\mathbf{#1}}}}}
\ddefloop ABCDEFGHIJKLMNOPQRSTUVWXYZ\ddefloop %
\def\ddef#1{\expandafter\def\csname hb#1\endcsname{\ensuremath{\hat{\boldsymbol{#1}}}}}
\ddefloop abcdefghijklmnopqrstuvwxyz\ddefloop %

\def\ddef#1{\expandafter\def\csname t#1\endcsname{\ensuremath{\tilde{#1}}}}
\ddefloop ABCDEFGHIJKLMNOPQRSTUVWXYZabcdefgijklmnpqtsuvwxyz\ddefloop 
\def\ddef#1{\expandafter\def\csname tc#1\endcsname{\ensuremath{\tilde{\mathcal{#1}}}}}
\ddefloop ABCDEFGHIJKLMNOPQRSTUVWXYZ\ddefloop
\def\ddef#1{\expandafter\def\csname tb#1\endcsname{\ensuremath{\tilde{\mathbf{#1}}}}}
\ddefloop ABCDEFGHIJKLMNOPQRSTUVWXYZ\ddefloop
\def\ddef#1{\expandafter\def\csname tb#1\endcsname{\ensuremath{\tilde{\boldsymbol{#1}}}}}
\ddefloop abcdefghijklmnopqrstuvwxyz\ddefloop %

\def\ddef#1{\expandafter\def\csname bar#1\endcsname{\ensuremath{\bar{#1}}}}
\ddefloop ABCDEFGHIJKLMNOPQRSTUVWXYZabcdefghijklmnopqrtsuvwxyz\ddefloop
\def\ddef#1{\expandafter\def\csname barc#1\endcsname{\ensuremath{\bar{\mathcal{#1}}}}}
\ddefloop ABCDEFGHIJKLMNOPQRSTUVWXYZ\ddefloop
\def\ddef#1{\expandafter\def\csname barb#1\endcsname{\ensuremath{\bar{\mathbf{#1}}}}}
\ddefloop ABCDEFGHIJKLMNOPQRSTUVWXYZ\ddefloop
\def\ddef#1{\expandafter\def\csname barb#1\endcsname{\ensuremath{\bar{\boldsymbol{#1}}}}}
\ddefloop abcdefghijklmnopqrstuvwxyz\ddefloop %

\def\ddef#1{\expandafter\def\csname war#1\endcsname{\ensuremath{\overline{#1}}}}
\ddefloop ABCDEFGHIJKLMNOPQRSTUVWXYZabcdefghijklmnopqrtsuvwxyz\ddefloop
\def\ddef#1{\expandafter\def\csname warc#1\endcsname{\ensuremath{\overline{\mathcal{#1}}}}}
\ddefloop ABCDEFGHIJKLMNOPQRSTUVWXYZ\ddefloop
\def\ddef#1{\expandafter\def\csname warb#1\endcsname{\ensuremath{\overline{\mathbf{#1}}}}}
\ddefloop ABCDEFGHIJKLMNOPQRSTUVWXYZ\ddefloop
\def\ddef#1{\expandafter\def\csname warb#1\endcsname{\ensuremath{\overline{\boldsymbol{#1}}}}}
\ddefloop abcdefghijklmnopqrstuvwxyz\ddefloop %

\def\sig{\sigma}

\def\dt{\delta}
\def\gam{\gamma}
\def\lam{\lambda}

\def\eps{\varepsilon}
\def\epsilon{\varepsilon}
\def\th{\theta}

\def\Dt{\Delta}

\def\Th{\Theta}

\usepackage{pgffor}
\def\greeksymbols{alpha,beta,gamma,gam,delta,dt,eps,epsilon,zeta,eta,theta,th,iota,kappa,kap,lambda,lam,mu,nu,xi,pi,rho,sigma,sig,tau,phi,chi,psi,omega,om,Gamma,Gam,Delta,Dt,Theta,Th,Lambda,Lam,Pi,Sigma,Sig,Phi,Psi,Omega,Om}
\def\greeksymbolsnoeta{alpha,beta,gamma,gam,delta,dt,eps,epsilon,zeta,theta,th,iota,kappa,kap,lambda,lam,mu,nu,xi,pi,rho,sigma,sig,tau,phi,chi,psi,omega,om,Gamma,Gam,Delta,Dt,Theta,Th,Lambda,Lam,Pi,Sigma,Sig,Phi,Psi,Omega,Om} 

\foreach \x in \greeksymbolsnoeta{\expandafter\xdef\csname b\x\endcsname{\noexpand\ensuremath{\noexpand\boldsymbol{\csname \x\endcsname}}}}

\foreach \x in \greeksymbols{\expandafter\xdef\csname h\x\endcsname{\noexpand\ensuremath{\noexpand\hat{\csname \x\endcsname}}}}
\foreach \x in \greeksymbolsnoeta{\expandafter\xdef\csname hb\x\endcsname{\noexpand\ensuremath{\noexpand\hat{\noexpand\boldsymbol{ \csname \x\endcsname}}}}}

\foreach \x in \greeksymbols{\expandafter\xdef\csname bar\x\endcsname{\noexpand\ensuremath{\noexpand\bar{\csname \x\endcsname}}}}
\foreach \x in \greeksymbolsnoeta{%
\expandafter\xdef\csname barb\x\endcsname{\noexpand\ensuremath{\noexpand\bar{\noexpand\boldsymbol{ \csname \x\endcsname}}}}
}

\foreach \x in \greeksymbols{\expandafter\xdef\csname t\x\endcsname{\noexpand\ensuremath{\noexpand\tilde{\csname \x\endcsname}}}}
\foreach \x in \greeksymbolsnoeta{\expandafter\xdef\csname tb\x\endcsname{\noexpand\ensuremath{\noexpand\tilde{\noexpand\boldsymbol{ \csname \x\endcsname}}}}}

\providecommand{\normz}[2][-1]{
\ensuremath{\mathinner{
\ifthenelse{\equal{#1}{-1}}{ 
\!\left\|#2\right\|}{}
\ifthenelse{\equal{#1}{0}}{ 
\|#2\|}{}
\ifthenelse{\equal{#1}{1}}{ 
\bigl\|#2\bigr\|}{}
\ifthenelse{\equal{#1}{2}}{ 
\Bigl\|#2\Bigr\|}{}
\ifthenelse{\equal{#1}{3}}{ 
\biggl\|#2\biggr\|}{}
\ifthenelse{\equal{#1}{4}}{ 
\Biggl\|#2\Biggr\|}{}
}} 
}  

\providecommand{\floor}[2][-1]{
\ensuremath{\mathinner{
\ifthenelse{\equal{#1}{-1}}{ 
\!\left\lfloor#2\right\rfloor}{}
\ifthenelse{\equal{#1}{0}}{ 
\lfloor#2\rfloor}{}
\ifthenelse{\equal{#1}{1}}{ 
\!\bigl\lfloor#2\bigr\rfloor}{}
\ifthenelse{\equal{#1}{2}}{ 
\!\Bigl\lfloor#2\Bigr\rfloor}{}
\ifthenelse{\equal{#1}{3}}{ 
\!\biggl\lfloor#2\biggr\rfloor}{}
\ifthenelse{\equal{#1}{4}}{ 
\!\Biggl\lfloor#2\Biggr\rfloor}{}
}} 
}

\providecommand{\ceil}[2][-1]{
\ensuremath{\mathinner{
\ifthenelse{\equal{#1}{-1}}{ 
\!\left\lceil#2\right\rceil}{}
\ifthenelse{\equal{#1}{0}}{ 
\lceil#2\rceil}{}
\ifthenelse{\equal{#1}{1}}{ 
\!\bigl\lceil#2\bigr\rceil}{}
\ifthenelse{\equal{#1}{2}}{ 
\!\Bigl\lceil#2\Bigr\rceil}{}
\ifthenelse{\equal{#1}{3}}{ 
\!\biggl\lceil#2\biggr\rceil}{}
\ifthenelse{\equal{#1}{4}}{ 
\!\Biggl\lceil#2\Biggr\rceil}{}
}} 
}

\newcommand{\fr}[2]{ { \frac{#1}{#2} }}

\newcommand{\T}{\top}

\def\cd{\cdot}

\def\la{\langle}
\def\ra{\rangle}

\def\larrow{\ensuremath{\leftarrow}}

\definecolor{mygrn}{rgb}{0,.8,0}
\definecolor{myred}{rgb}{.8,0,0}

\DeclareMathOperator{\EE}{\mathbb{E}} 

\DeclareMathOperator{\PP}{\mathbb{P}}

\DeclareMathOperator*{\argmax}{arg~max}
\DeclareMathOperator*{\argmin}{arg~min}

\DeclareMathOperator{\supp}{{\mathrm{supp}}}

\DeclareMathOperator{\one}{\mathds{1}\hspace{-.1em}}
\providecommand{\onec}[2][-1]{
\ensuremath{\mathinner{
\one\cbr[#1]{#2}
} 
}  
}

\DeclarePairedDelimiterX{\inp}[2]{\langle}{\rangle}{#1, #2}

\newcommand\declareop[3]{%
  \newcommand#1{%
    \mskip\muexpr\medmuskip*#2\relax
    {#3}%
    \mskip\muexpr\medmuskip*#2\relax
}}
\declareop\capprox{1}{{\sr{\const}{\approx}}} 
\declareop\logapprox{1}{{\sr{\mathsf{log}}{\approx}}} 

\def\opt{{\mathsf{opt}}}

\def\const{\mathsf{const}}

\def\poly{\operatorname{poly}}
\def\polylog{{\normalfont\text{polylog}}}

\def\Reg{{\mathsf{Reg}}}

\usepackage{pifont}
\newcommand{\sr}{\stackrel}

\makeatletter
\newcommand{\vast}{\bBigg@{3}}
\newcommand{\Vast}{\bBigg@{4}}
\makeatother

\newcommand{\edit}[2]{{\xspace\textcolor{blue}{\sout{#1}}}{ \textcolor{red}{#2}}}

\newenvironment{talign*}
 {\csname align*\endcsname}
 {\endalign}

\def\chrulefill{\leavevmode\leaders\hrule height 0.7ex depth \dimexpr0.4pt-0.7ex\hfill\kern0pt}

\def \AII {6Bd  \log \left( 1 + \frac{2}{\lambda}\right) }
\def \BII { n\cdot B \cdot 2^{-L} \cdot \onec{B\cdot2^{-L} >\Delta}  }
\def \DII { \frac{12dB}{2^{-L}\eps} \log \left( 1 + \frac{2}{\lambda 2^{-2L}\eps } \right) + 6dB\log\left( 1 + \frac{2}{\lambda}\right)}
\def \FIFI {4B n\exp\left(- \frac{B^2}{ 16\eps \cdot  \beta_{{n}}(\delta_n)  }\right)}
\def \FIFII {\frac{1}{\alpha_{\mathrm{emp}}} \left(\DII \right)}
\def \FI {\frac{1}{\alpha_{\mathrm{emp}}}\cdot \FIFI + \frac{1}{\alpha^2_{\mathrm{emp}}} \left(\DII \right)  }
\def \FIIFII {B \log (n+1)}
\def \FIIFI {\frac{192  \beta^{*}_{{n}}(\delta_n)d}{B \cdot 2^{-L}} \log\left( 1 + \frac{32\beta^{*}_{{n}}(\delta_n)}{\lambda B^2 \cdot 2^{-2L}}\right) + 6 dB  \log \left( 1 + \frac{2}{\lambda}\right)}
\def \FII {\FIIFII + \FIIFI}
\def \FIIIFII {\FIIFII}
\def \FIIIFI {\frac{512 H_{\mathrm{max}}\cdot C_{\mathrm{opt}}\cdot d\log(d)}{\alpha_{\mathrm{opt}} B \cdot 2^{-L}} \left( \frac{\lambda (S^*)^2}{2} + \sigma_*^2 d \log \left( 1 + \frac{n}{d\lambda}\right) \right)}
\def \FIII {\FIIIFII + \FIIIFI}

\def \REGLEMMAFOUR { 6dB \left(3 + \frac{1}{\alpha_{\mathrm{emp}}^2}\right)\log \left(1 + \frac{2}{\lambda}\right) + 2B \log (n+1) +   \frac{192\beta_{n}(\delta)  \log (n) d}{2^{-L}B}  \left(1 +  \frac{1}{\alpha_{\mathrm{emp}}^2}\right) \log \left( 1 + \frac{32 \beta_{n}(\delta)  \log (n) }{\lambda 2^{-2L} B^2 }\right)\\
	&\spacex + \frac{192  \beta^{*}_{{n}}(\delta_n)d}{B \cdot 2^{-L}} \log\left( 1 + \frac{32\beta^{*}_{{n}}(\delta_n)}{\lambda B^2 \cdot 2^{-2L}}\right) + \frac{512 H_{\mathrm{max}}\cdot C_{\mathrm{opt}}\cdot d \log (d)}{\alpha_{\mathrm{opt}} B \cdot 2^{-L}} \left( \frac{\lambda (S^*)^2}{2} + \sigma_*^2 d \log \left( 1 + \frac{n}{d\lambda}\right) \right)\\
	&\spacex + \BII + \frac{ 4B}{\alpha_{\mathrm{emp}}} }

\usepackage{enumitem}
\setlist{itemsep=.1em}

\usepackage{caption}
\usepackage{subcaption}

\usepackage[utf8]{inputenc} 
\usepackage[T1]{fontenc}    
\usepackage{hyperref}       
\usepackage{url}            
\usepackage{booktabs}       
\usepackage{amsfonts}       
\usepackage{nicefrac}       
\usepackage{microtype}      
\usepackage{xcolor}         

\def\ApproxDesign{{\mathrm{ApproxDesign}}} 
\def\ApproxDesignAugmented{{\mathrm{ApproxDesignAugmented}}} 
\def\opt{\mathrm{opt}}

\newcommand{\spacex}{\;\;\;\;}

\definecolor{bg}{rgb}{0.91, 0.91, 0.91}

\def\Sq{\mathrm{Sq}}

\usepackage[normalem]{ulem}
\newcommand{\stkout}[1]{\ifmmode\text{\sout{\ensuremath{#1}}}\else\sout{#1}\fi}

\def\Reg{{\mathrm{Reg}}}

\newcommand{\kb}[1]{{\color{RubineRed}[#1]}} 

{\color{kjsaved}}%

\graphicspath{{Figures/}}

\newif\ifFINAL
\FINALtrue 

\ifFINAL
    
    \def\guide#1{}
    \def\gray#1{}
    \def\kj#1{}
    \def\kb#1{}
\else
    \usepackage[inline]{showlabels}
    
\fi

\ifx\theorem\undefined
\newtheorem{theorem}{Theorem}
\fi
\ifx\example\undefined
\newtheorem{example}{Example}
\fi
\ifx\property\undefined
\newtheorem{property}{Property}
\fi
\ifx\lemma\undefined
\newtheorem{lemma}{Lemma}
\fi
\ifx\proposition\undefined
\newtheorem{proposition}[theorem]{Proposition}
\fi
\ifx\remark\undefined
\newtheorem{remark}[theorem]{Remark}
\fi
\ifx\corollary\undefined
\newtheorem{corollary}[theorem]{Corollary}
\fi
\ifx\definition\undefined
\newtheorem{definition}[theorem]{Definition}
\fi
\ifx\conjecture\undefined
\newtheorem{conjecture}[theorem]{Conjecture}
\fi
\ifx\fact\undefined
\newtheorem{fact}[theorem]{Fact}
\fi
\ifx\claim\undefined
\newtheorem{claim}{Claim}
\fi
\ifx\assum\undefined
\newtheorem{assumption}{Assumption}
\fi

\usepackage{etoolbox}

\def\mygapin{1}
\def\mygapnew{0.5}

\begin{document}

\textfloatsep=0.5em
\setlength{\abovedisplayskip}{3pt}%
\setlength{\belowdisplayskip}{4pt}%
\setlength{\abovedisplayshortskip}{3pt}%
\setlength{\belowdisplayshortskip}{4pt}

%

%

\doparttoc 
\faketableofcontents 

\runningtitle{Minimum Empirical Divergence for Sub-Gaussian Linear Bandits}

\twocolumn[

\aistatstitle{Minimum Empirical Divergence for Sub-Gaussian Linear Bandits}

\aistatsauthor{ Kapilan Balagopalan  \And Kwang-Sung Jun  }

\aistatsaddress{ University of Arizona \\ kapilanbgp@arizona.edu \And University of Arizona \\ kjun@cs.arizona.edu } ]

\begin{abstract}
\vspace{-1em}
	We propose a novel linear bandit algorithm called LinMED (Linear Minimum Empirical Divergence), which is a linear extension of the MED algorithm that was originally designed for multi-armed bandits.
	LinMED is a randomized algorithm that admits a closed-form computation of the arm sampling probabilities, unlike the popular randomized algorithm called linear Thompson sampling.
	Such a feature proves useful for off-policy evaluation where the unbiased evaluation requires accurately computing the sampling probability.
	We prove that LinMED enjoys a near-optimal regret bound of $d\sqrt{n}$ up to logarithmic factors where $d$ is the dimension and $n$ is the time horizon.
	We further show that LinMED enjoys a $\frac{d^2}{\Delta}\left(\log^2(n)\right)\log\left(\log(n)\right)$ problem-dependent regret where $\Delta$ is the smallest sub-optimality gap.
	Our empirical study shows that LinMED has a competitive performance with the state-of-the-art algorithms.
\end{abstract}

\section{INTRODUCTION}\label{main-section:intro-section}

The multi-armed bandit problem represents a stateless reinforcement learning framework with numerous real-world applications. One of its most prominent applications is in recommendation systems, which are extensively employed by e-commerce platforms~\citep{elena21survey}, digital streaming services~\citep{elena21survey,mary15bandits}, news portals~\citep{li10acontextual}, and a variety of other platforms experiencing significant economic growth. The multi-armed bandit problem has spawned several important variants, including stochastic linear bandits, adversarial bandits, and best-arm identification, all of which share a common underlying structure but are adapted to different environments to achieve distinct goals. This has fostered a rich and robust area of research.

In bandit problems, the main objective is to minimize cumulative regret by learning to select optimal arms over time. 
A key challenge is maintaining a balance between exploration (gathering information about the mean rewards of various arms) and exploitation (leveraging gathered information to take arms with large estimated rewards). Focusing exclusively on either strategy would be sub-optimal for minimizing cumulative regret.

Stochastic linear bandits, in particular, generalize the classical multi-armed bandit framework.
At each time step $t\in\{1,2,\ldots\}$, the learner selects an arm $A_t$ from a set of arms $\mathcal{A}_t \subset \mathbb{R}^d$, and observes a reward given by $Y_t = \langle A_t, \theta^* \rangle + \eta_t$, where $\th^*\in\RR^d$ is an unknown parameter and $\eta_t$ is a zero-mean noise.
The learner's objective is to minimize the cumulative (pseudo-)regret over the time horizon $n$, which is defined by:
\begin{align}\label{eq:def-regret}
	\Reg_n :=  \sum_{t=1}^{n} \max_{a \in \mathcal{A}_t}  \langle a, \theta^*\rangle - \langle A_t, \theta^*\rangle~.
\end{align}

Since linear bandits generalize multi-armed bandits, many linear bandit algorithms have been derived by adapting algorithmic principles from multi-armed bandits. For instance, LinUCB~\citep{dani08stochastic,ay11improved} extends the optimism principle from UCB~\citep{auer02finite}, linear Thompson sampling~\citep{agrawal14thompson} extends Thompson sampling~\citep{thompson33onthelikelihood,agrawal17nearoptimal}, and LinIMED~\citep{bian24indexed} extends IMED~\citep{honda15non}.
Therefore, it is natural to explore the stochastic linear bandit version of the MED framework~\citep{honda11asymptotically}.
\looseness=-1

Evaluating bandit algorithms for recommendation systems typically requires running it live with the customers.
However, this severely costs the user experience if the algorithm has a poor performance.
Off-policy evaluation (OPE)~\citep{precup00eligibility} aims to address this issue by evaluating an algorithm (i.e., target policy) using the data collected by another algorithm (i.e., logging policy).
The standard method for OPE is inverse propensity weighting  (IPW)~\citep{horvitz52generalization}, which provides an unbiased estimator for the target policy's performance.
For this to work, the logging policy is required to satisfy two properties.
First, it must assign a nonzero sampling probability to every arm because otherwise we obtain no information on the zero-probability arms, disallowing counterfactual inference on their rewards.
This automatically excludes any algorithm that makes a deterministic arm selection conditioning on previous observations.
Second, the logging policy must allow accurate computation of the sampling probability.
This is because IPW uses inverse sampling probability as an importance weight to scale the observed rewards.
Thus, when the sampling probability is small, even a small error can be detrimental.
We remark that by accurate computation we mean the extra computation \textit{in addition to} running the algorithm itself.
For example, any algorithm that computes the assigned probability for each arm first and then samples an arm would require zero extra computation.
On the other hand, (linear) Thompson sampling~\citep{thompson33onthelikelihood,agrawal14thompson}  itself does not compute the sampling probability, so extra computation is needed.
We call algorithms that satisfy the two properties above to be \textit{OPE-friendly}.
\looseness=-1


In this paper, we propose a novel linear bandit algorithm called LinMED (Linear Minimum Empirical Divergence), which is a linear version of Minimum Empirical Divergence~\citep{honda11asymptotically}. 
LinMED has numerous merits.

First, LinMED is OPE-friendly.
This is in stark contrast to the popular randomized algorithm linear Thompson sampling (LinTS)~\citep{agrawal14thompson} that is not OPE-friendly.
LinTS does not have a known closed-form solution or efficient methods for computing arm sampling probabilities and may  assign zero probabilities to many arms.
Note that using Monte Carlo sampling for estimating the probability up to the target precision has the time complexity of $O(1/\sqrt{\mathrm{precision}})$ where precision is the desired floating point precision, which is quite large for a numerical approximation method (vs, say, $\log(1/\text{precision})$ of the bisection method)
In particular, as discussed above, the error in probability goes to the denominator and further amplifies the error for IPW.
We numerically verify such a phenomenon in Figure~\ref{main-fig:offline-eval-figure}. Specifically, we take the uniform policy, which assigns equal probabilities to each arm, as the target policy. We evaluate IPW based on two logging policies, LinMED and LinTS, respectively. 
Given the logged data \( (A_t, p_t(A_t), Y_t)_{t=0}^{n} \) where $p_t(A_t)$ is probability of sampling arm $A_t$ from the logging policy, the IPW score for the uniform policy is defined as
\[
\text{IPW score} = \frac{1}{n} \sum_{t=1}^{n} \frac{\frac{1}{|\mathcal{A}|}}{p_t(A_t)} \cdot Y_t
\]
	
As discussed for LinTS, the logged probability \( p_t(A_t) \) must be estimated via Monte Carlo sampling. We selected the arm set \( \mathcal{A} = \{a_1 = (1, 0)^\top, a_2 = (0.6, 0.8)^\top \} \) and \( \theta^* = (1, 0)^\top \). It is expected that an OPE-friendly algorithm would yield an IPW score as an unbiased estimator of the expected reward of the uniform policy, which is $0.8$. 
As shown in Figure \ref{main-fig:offline-eval-figure}, the mean IPW w.r.t. LinMED is almost identical to the true value 0.8 while that w.r.t. LinTS exhibits a significant bias.
\looseness=-1

	\begin{figure}[t]
		\centering
		\includegraphics[width=1\linewidth]{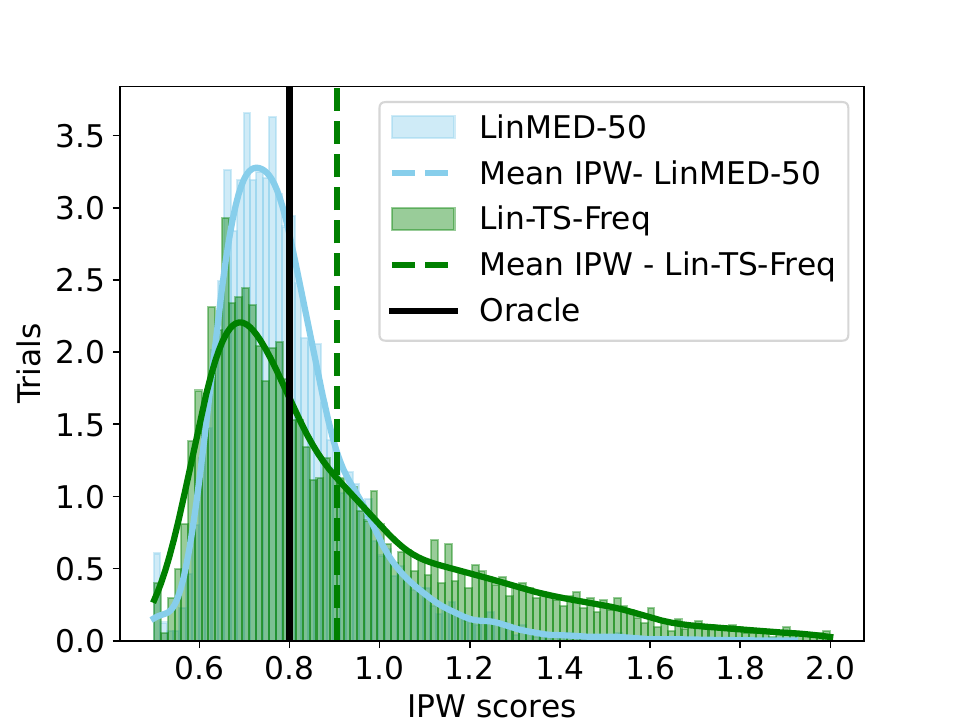}
		\vspace{-.8em}
		\caption{
			IPW scores of the uniform policy when the logging policy is LinMED and LinTS respectively.
				We used 1,000 Monte Carlo samples to estimate the sampling probabilities of LinTS.
				Oracle denotes the expected reward of the uniform policy.
				LinTS shows a nontrivial amount of bias, unlike LinMED (mean of LinMED is exactly aligned with the oracle, thus invisible in the plot).
				See Appendix \ref{app-subsection:offline-eval-exp-subsection} for details. }
		%
		\label{main-fig:offline-eval-figure}
	\end{figure}

Second, LinMED  achieves not only a near-optimal minimax regret bound of $\tO(d\sqrt{n})$ \citep{dani08stochastic}, where $\tO$ omits logarithmic factors, but also an instance-dependent regret bound of $O(\fr{d^2}{\Delta}\log^2(n))$ where $\Delta$ is the smallest gap as defined in \eqref{main-eq:smallest-gap-definition}.
To our knowledge, the only existing linear bandit algorithm with an nonasymptotic instance-dependent regret bound is OFUL~\citep{ay11improved}. 
LinMED stands out even more when compared against randomized algorithms that allow closed-form computation of sampling probability, namely SquareCB~\citep{foster20beyond}, EXP2 \citep{bubeck12towards}, and SpannerIGW~\citep{zhu22contextual}, because they provably have sub-optimal instance-dependent regret of $\Omega(\Delta\sqrt{n})$, as we show later in Theorems \ref{main-thm:exp2-inst-dep-bound-theorem} and \ref{main-thm:spannerigw-inst-dep-bound-theorem} and numerically confirm in Section \ref{main-section:emp-studies-section}.
We summarize the comparison of LinMED with other methods in Table~\ref{main-table:comparison-table}.

Third, our analysis reveals that LinMED enjoys sub-linear regret bounds even if the sub-Gaussian noise parameter $\sig_*^2$ is \textit{under}-specified in the algorithm, albeit with an extra factor that grows with the degree of under-specification.
This is in stark contrast to existing algorithms and their analyses that only provides a valid regret bound when the sub-Gaussian parameter is \textit{over}-specified~\citep{ay11improved,agrawal14thompson}.
A more detailed discussion is provided in Section \ref{main-section:thm-statemnt-section}.

Finally, LinMED demonstrates outstanding empirical performance across various challenging scenarios, including delayed reward settings (see Appendix~\ref{app-subsection:delayed-reward-exp-subsection}) and ``end of optimism'' instance. 
A comprehensive discussion of these results is provided in Section~\ref{main-section:emp-studies-section}.

\begin{table*}[t]
	\centering
	\begin{tabular}{ |>{\centering\arraybackslash}p{6cm}|>{\centering\arraybackslash}p{2.5cm}|>{\centering\arraybackslash}p{2cm}| >{\centering\arraybackslash}p{2cm}|>{\centering\arraybackslash}p{2cm}| } 
		\hline
		Algorithms & Minimax regret & Instance-dependent regret & Efficiently computable probability & Probability assigned for all arms\\
		\hline
		OFUL \citep{ay11improved} & $\tilde{O}(d\sqrt{n})$ &$O(\frac{d^2}{\Delta}\log^3n)$  & N/A & No\\
		LinIMED \citep{bian24indexed}& $\tilde{O}(d\sqrt{n})$ & Unknown  & N/A & No\\
		LinTS \citep{agrawal14thompson}& $\tilde{O}(d\sqrt{dn})$ & Unknown  & No  & No\\
		RandUCB \citep{vaswani20olddog} &  $\tilde{O}(d\sqrt{n})$ &Unknown & No & No \\
		SquareCB \citep{foster20beyond} & $\tilde{O}(\sqrt{Kdn})$ & Unknown & Yes & No  \\
		E2D \citep{foster23tight} &  $\tilde{O}(d\sqrt{n})$ &Unknown& Yes & No* \\
		SpannerIGW \citep{zhu22contextual} & $\tilde{O}(d\sqrt{n})$ &$\Omega(\Delta\sqrt{n})$ & Yes & No*\\
		EXP2 \citep{bubeck12towards} & $O(\sqrt{dn\log K })$ &$\Omega(\Delta\sqrt{n})$ & Yes & Yes \\
		LinMED (ours) & $\tilde{O}(d\sqrt{n})$ &$O(\frac{d^2}{\Delta}\log^2n)$ & Yes & Yes \\
		\hline
	\end{tabular}
	\caption{
		Comparison of linear bandit algorithms. `No*' means that the algorithm can be modified to assign a nonzero probability to every arm. The term ``efficiently computable probability'' refers to the efficiency in extra computation in addition to running the algorithm. 
	}
	\label{main-table:comparison-table}
 \vspace{-1em}
\end{table*}

\textbf{Organization.}
In Section \ref{problemstatement}, we introduce the problem formulation and key notations. This is followed by the presentation of a warm up version of LinMED in Section~\ref{main-section:warm-up} where we also provide a brief discussion on its connection to Maillard sampling~\citep{bian22maillard} and  SpannerIGW~\citep{zhu22contextual} highlighting the importance of optimal experimental design for large arm sets. This is followed by the presentation of LinMED algorithm in Section \ref{main-section:algo-section}. Next, we move to the main results in Section \ref{main-section:thm-statemnt-section} where we establish the regret bounds of our algorithm. Additionally, in Section \ref{main-section:inst-dep-low-bound-spannerigw-exp2-section}, we discuss the instance-dependent lower bounds for SpannerIGW and EXP2. Finally, Section \ref{main-section:emp-studies-section} presents empirical studies to support our theoretical findings.

\section{PRELIMINARIES} \label{main-section:prob-def-prem-section}
\label{problemstatement}
\vspace{-1em}

\textbf{Notations.} For any $d$ dimensional vector $x \in \mathbb{R}^d$ and a $d \times d$ positive definite matrix $A$, we use $\lVert x \rVert_A$ to denote the Mahalanobis norm $\sqrt{x^\T A x}$ and we use $\lVert x \rVert$ to denote the Euclidean norm. We use $a \wedge b$ (resp. $a \vee b$) to denote the minimum (resp. maximum) of two real numbers $a$ and $b$. 
For a set $\cB \subset \RR^d$, denote by $\Delta(\cB)$ the set of all probability measures on $\cB$. 
The notation $\tilde{O}(\cdot)$ omits the logarithmic factors from the standard big-O notation $O(\cd)$.
For example , $A \log B = \tilde{O}(A)$.
For any event $\mathcal{E}$, the complement of the event is denoted by $\overline{\mathcal{E}}$.
We denote $a_i, a_{i+1}, \ldots, a_j$ by $a_{i:j}$.
\looseness=-1

\textbf{The Stochastic Linear Bandit Model.} 
In the stochastic linear bandit model, the learner chooses an arm $A_t$ in each round $t$ from the arm set $\mathcal{A}_t  \subset \mathbb{R}^d$.
After choosing arm $A_t$, the environment reveals a reward 
\begin{align*}
	Y_t = \langle \theta^*, A_t\rangle + \eta_t
\end{align*}
to the learner where $\theta^* \in \mathbb{R}^d$ is an unknown coefficient of the linear model, $\eta_t$ is a $\sigma_*^2$-sub-Gaussian noise conditioned on $A_{1:t}$ and $Y_{1:t-1}$.
That is, for any $\lambda \in \mathbb{R}$, almost surely,
\begin{align*}
	\EE\sbr{\exp{\left(\lambda\eta_t\right)} \mid A_{1:t},Y_{1:t-1}} \leq \exp\left( \frac{\lambda^2 \sigma_*^2}{2}\right).
\end{align*}
Further, denote by $a^*_t := \argmax_{a \in \mathcal{A}_t} \langle \theta^*, a \rangle$ the arm with the largest mean reward at time $t$. The goal of the learner is to minimize the cumulative (pseudo-)regret over the horizon $n$, which is precisely defined in~\eqref{eq:def-regret}.
Throughout the paper, we focus on analyzing the expected (pseudo-)regret $\EE \Reg_n$ rather than a high probability bound. We also assume the following,
\begin{assumption}\label{main-assump:env-assumption}
For all $ t \geq 1$, every arm $a\in \cA_t$ satisfies $\|a\|_2\le1$. Furthermore, for some constant $B$, $\forall t \geq 1$ $\Delta_{a,t}:=  \langle \theta^*, a^*_t \rangle - \langle \theta^*, a \rangle   \leq B, \forall a \in \mathcal{A}_t$.
\end{assumption}

Note that prior linear bandit studies make the assumption of knowing the value of $\sigma$ and $S$ such that $\sigma^2_*\le \sigma^2$ and $\|\theta^*\|\le S$, which  accounts for the case of over-specification but not under-specification\footnote{
  \citet{gales22norm} adapt to the unknown norm $\|\th^*\|$ but not the sub-Gaussian parameter $\sig_*^2$.
}.
Instead, we analyze the regret of our proposed algorithm for arbitrarily given $\sigma$ and $S$ as guesses about $\sigma_*$ and $\|\theta^*\|$, accounting for both over- and under-specification.

\section{WARMUP: A LINEAR EXTENSION OF MINIMUM EMPIRICAL DIVERGENCE} \label{main-section:warm-up}

In multi-armed bandits, we are given $K$ arms and required to repeatedly choose an arm $A_t\in[K]$ to pull and observe its stochastic reward to maximize the cumulative reward.
MED~\citep{honda11asymptotically} is a randomized multi-armed bandit algorithm that is optimized for bounded rewards (and achieved improved regret bounds compared to those that are optimized for sub-Gaussian rewards, which is a larger class of reward distributions) where the algorithm principle can be instantiated for (sub-)Gaussian rewards, which appeared first in \citet{maillard13apprentissage} and further analyzed in \citet{bian22maillard}.
Maillard sampling, or sub-Gaussian MED, samples arm from the following distribution:
\begin{align*}
    p_{t,a} &= \frac{\exp\left( - \frac{N_{t-1,a}}{2} \hat{\Delta}_{a,t-1}^2\right)}{\sum_{ b  \in \mathcal{A}} \exp\left( - \frac{N_{t-1,b}}{2} \hat{\Delta}_{b,t-1}^2\right)}
\end{align*}
where $\hat{\Delta}_{a,t-1} = \max_{a^{'} \in  [K]} \hat{\mu}_{t-1,a^{'}} - \hat{\mu}_{t-1,a}$ is the estimated reward gap at $t$, and $\hmu_{t,a}$ and $N_{t,a}$ are the empirical mean reward of arm $a$ based on past rewards from arm $a$ and the pull count of arm $a$, respectively. 
Throughout, we simply say MED for this instance and refer both of them interchangeably since they are the same in spirit.

Towards a linear extension of MED, one may consider the following counterparts for the linear model:
\begin{align}\label{main-eq:ms-to-linmed}
    \hat{\Delta}_{a,t} \!\rightarrow\!  \max_{a' \in  \mathcal{A}_t}  \langle \hat{\theta}_{t-1}, a' - a \rangle   \text{ and } N_{t-1,a} \!\rightarrow\! \frac{1}{\lVert a \rVert^2_{V_{t-1}^{-1}}}.
\end{align}
The second term in \eqref{main-eq:ms-to-linmed} is justified since the leverage score  ($\lVert a \rVert^2_{V_{t-1}^{-1}}$) decreases with amount of exploration performed in the direction of $a$. 
This leads to an algorithm that we call LinMEDNOPT (\textbf{Lin}ear \textbf{M}inimum \textbf{E}mpirical \textbf{D}ivergence with \textbf{N}o \textbf{OPT}imal design of experiment) with the sampling distribution given by
\begin{align}\label{main-eq:prb-linmednopt}
    p_t^{\text{LinMEDNOPT }}(a)= \frac{f_{t}(a)}{\sum_{b \in \mathcal{A}_t}  f_{t}(b)},
\end{align}
where $f(t)$ is defined in \eqref{eq:main-f_tm1}.
Our attempts to analyze the regret of this algorithm resulted in a polynomial dependence on $K$, which is undesirable since the strength of linear bandits is the ability to handle a large or even an infinite number of arms.
\looseness=-1

Indeed, one can find a problem where the regret scales with $K$ as follows.
Specifically, consider a 2-dimensional problem where the best arm and $\theta^*$ are both $(1,0)\in \RR^2$. 
The rest of the $K-1$ arms are all $(0,1)\in \RR^2$; i.e., all the sub-optimal arms share the same feature representation. 
In the beginning, after a few arm pulls, LinMEDNOPT could misjudge one of the sub-optimal arms as the best arm with a constant probability (imagine $\hat\theta$ being around $(-1,0)$).
Then, it assigns the same constant probability for choosing one of the sub-optimal arms in the next time step. 
Since there are $K-1$ such sub-optimal arms, the total probability assigned to them will be high.
Consequently, this significantly reduces the probability assigned to the true optimal arm (at most $1/K$), resulting in not exploring in the direction of the optimal arm. 
Since pulling an arm in the direction of the suboptimal arm $(0,1)$ provides zero information on the best arm $(1,0)$, it will be not until the algorithm pulls the optimal arm a few times that it can recover from this undesirable state.
The waiting time for this is $\Omega(K)$ during which we suffer a linear regret.
This indeed happens and leads to an order $K$ regret numerically as can be seen in Figure~\ref{main-fig:k-dependency-figure}.

\begin{figure}[t]
	\centering
	\includegraphics[width=1\linewidth]{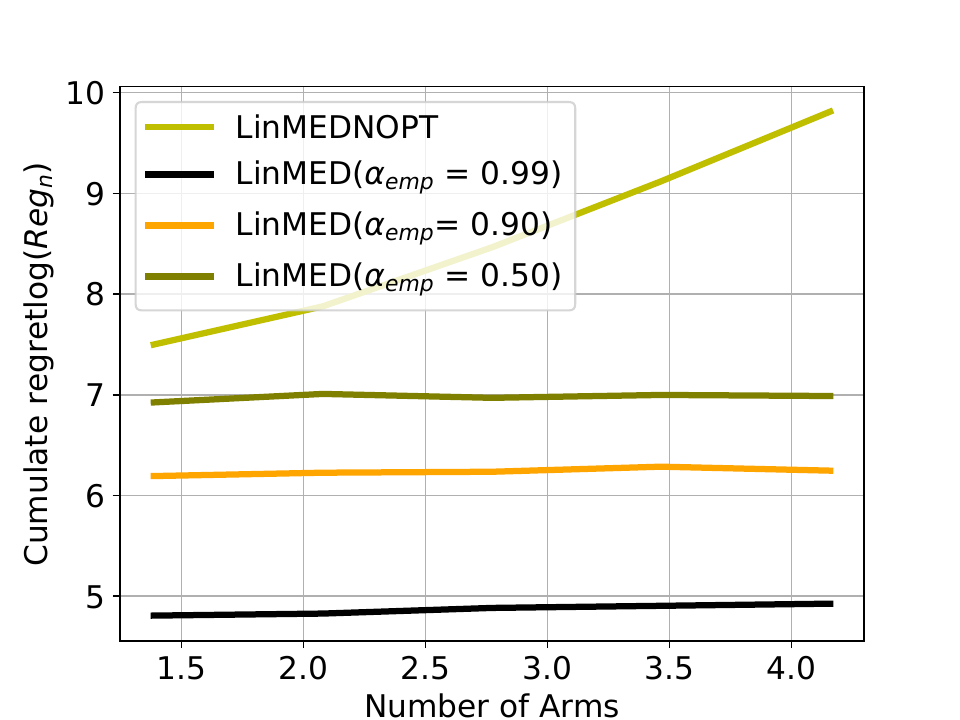}
	\caption{ LinMED vs LinMEDNOPT, with $\sigma^2 = \sigma_*^2 = 3$ for $K \in \{4,8,16,32,64\}$, and  $(\alpha_{\text{emp}},\alpha_{\text{opt}})  \in \{(0.99,0.005),(0.90,0.05),(0.5,0.25)\}$, $n = 20000$.  }
	\label{main-fig:k-dependency-figure}
\end{figure}

Inspired by SpannerIGW \citep{zhu22contextual}, we leverage the G-optimal design to avoid the dependence on $K$ and propose an algorithm called LinMED in the next section.
The key idea is that G-optimal design assigns probabilities to arms such that it will be informative for the linear model structure.
Specifically, in the example above, G-optimal design will assign probabilities as if there are only two arms $(1,0)$ and $(0,1)$.
This way, the probability will be assigned to these two arms almost equally at the beginning, ensuring that the waiting time to recover from the bad state discussed above is $\Th(1)$ with respect to $K$ rather than $\Th(K)$.
Our proposed algorithm LinMED will have a hyper-parameter $\alpha_{\text{opt}} \in (0,1)$ that controls how much we rely on the optimal design.
Figure~\ref{main-fig:k-dependency-figure} shows that LinMED with various choices of $\alpha_\opt$ results in regret independent of $K$.
\looseness=-1

\section{LINEAR MINIMUM EMPIRICAL DIVERGENCE (LINMED)} \label{main-section:algo-section}

In this section, we describe our proposed algorithm Linear Minimum Empirical Divergence (LinMED; Algorithm~\ref{main-algo:LinMED-algo}).
LinMED takes in guesses $\sigma^2$ and $S$ on the unknown problem parameters $\sigma_*^2$ and $\|\th^*\|$, but we do not require that these guesses are over-specified respectively, as we discussed in Section~\ref{main-section:prob-def-prem-section}.
At each time step $t$, the algorithm has maintained a ridge regression estimator $\hth_{t-1}$ computed with a ridge parameter $\lambda$ based on the samples collected up to time step $t-1$; see Algorithm~\ref{main-algo:LinMED-algo} for their precise definitions. 
Let
\begin{align}\label{eq:beta}
    \beta_{t}(\delta_t) := \left(\sigma \sqrt{\log \left(\frac{\det V_{t}}{\det V_0} \right) + 2\log\frac{1}{\delta_t}} + \sqrt{\lambda}S\right)^2
\end{align}
where $V_{t} = \lambda I + \sum_{s=1}^{t} A_s A_s^T$.

\begin{algorithm}[h!]
\setlength{\abovedisplayskip}{3pt}%
\setlength{\belowdisplayskip}{3pt}%
\setlength{\abovedisplayshortskip}{3pt}%
\setlength{\belowdisplayshortskip}{3pt}

	\textbf{Input:} regularization $\lam$, failure rates $\{ \delta_t\}_{t=0}^{\infty}$,
	optimal design fraction $\alpha_{\text{opt}}$,
	empirical best fraction $\alpha_{\text{emp}}$, $\mathrm{ver} \in \{0,1\}$,
	$S$ (guess for $\lVert \theta^* \rVert_2$), and $\sigma^2$ (guess for $\sig_*^2$)
	\begin{algorithmic}[1] 
		\STATE  Initialize $\hat{\theta}_0 = 0$, $V_0 = \lambda I $.
		\FOR{$t=1,2,\ldots $}
		\STATE Observe arm set $\mathcal{A}_t$.
			\STATE Estimate  $\hat{a}_t = \max_{a^{'} \in \mathcal{A}_{t}}\langle \hat{\theta}_{t-1}, a^{'} \rangle$.
		\STATE Estimate $\hat{\Delta}_{a,t} :=  \langle \hat{\theta}_{t-1}, \hat{a}_t - a \rangle \spacex \forall a \in \mathcal{A}_t $. \label{aaa}
		\STATE Define $\forall a \in \mathcal{A}_t $
		\begin{align}\label{eq:main-f_tm1} 
			f_{t}(a) = \exp \del[4]{- \frac{\hat\Delta_{a,t}^2}{\beta_{t-1}(\delta_{t-1})  \lVert \hat{a}_{t} - a \rVert^2_{V_{t-1}^{-1}} }} 
		\end{align}
		where we take $\frac{0}{0} = 0$ and $\beta_t(\dt_t)$, defined in \eqref{eq:beta}, is a function of $S$ and $\sig^2$.
		\STATE Compute a design: \\ ~~~~$q^\opt_{t} = \mathrm{ApproxDesignAugmented}(\mathcal{A}_t, f_{t}, \mathrm{ver})$.\vspace{.5em} 
		\STATE  Let  $  \forall a \in \mathcal{A}_t $
		\begin{align}\label{main-eq:q_t}
			q_t(a) &= \alpha_{\text{opt}} \cdot q^\opt_{t}(a) + \alpha_{\text{emp}} \cdot \one\cbr{a = \hat{a}_{t}} \notag\\
			&\spacex + (1- \alpha_{\text{opt}} - \alpha_{\text{emp}}) \cdot \frac{1}{\lvert \mathcal{A}_t \rvert}.
		\end{align}\vspace{-.5em}
		\STATE Compute $p^{'}_{t}(a)$:
		\begin{align}\label{eq:p_tm1} 
			p^{'}_{t}(a) =  \frac{q_t(a) f_{t}(a)}{\sum_{b \in \mathcal{A}_t} q_t(b) f_{t}(b)}.
		\end{align}
		\STATE Define 
		\begin{align}\label{main-eq:event-b}
			\mathcal{B}_t = \{ a \in \mathcal{A}_t : \lVert a \rVert^2_{V_{t-1}^{-1} } > 1\}.
		\end{align}\vspace{-.5em}
		\IF {$\lvert \mathcal{B}_t \rvert > 0$} 
    \STATE $\forall a \in \mathcal{A}_t, \spacex p_t(a) = \frac{1}{2}p^{'}_t(a) + \frac{1}{2}\one\cbr{a = B_t}$ where $B_t$ is an arbitrarily chosen action $\in \mathcal{B}_t$.
		\ELSE 
		\STATE  
       $\forall a \in \mathcal{A}_t \spacex  p_t(a) = p^{'}_t(a)$.
		\ENDIF
		\STATE Take action $A_t \sim p_t$. 
		\STATE Observe the reward $ Y_t$ and update
      \[V_t = V_{t-1} + A_tA^\T_t \text{ and } \spacex \hat{\theta}_t = V_t^{-1} \sum_{s=1}^{t} A_sY_s. \]
		\ENDFOR
	\end{algorithmic}
	\caption{LinMED}
	\label{main-algo:LinMED-algo}
\end{algorithm}

LinMED first transforms the original arm set $\mathcal{A}_t$ into an augmented arm set $\overline{\mathcal{A}}_{(t)}$, see Algorithm \ref{main-algo:LinMED-sub-routine}. Although we present two different versions of LinMED, the version where the augmented arm set is generated by eliminating highly sub-optimal arms—while simpler to analyze—cannot be extended to cases where the true sub-Gaussian noise parameter is under-specified. Therefore, the main focus of this paper is the version 0, although detailed proofs for both version 0 and version 1 are provided in Appendix \ref{app-section:gmain-proof}. 
In version 0, the arms are rescaled as follows: 
\begin{align*}
	\overline{\mathcal{A}}_{(t)} = \{ \sqrt{f_{t}(a)} \cd a \mid a \in \mathcal{A}_t\}
\end{align*}
where $f_{t}(a)$ is an exponential weight defined in \eqref{eq:main-f_tm1}.


\begin{algorithm}[t]
\setlength{\abovedisplayskip}{3pt}%
\setlength{\belowdisplayskip}{3pt}%
\setlength{\abovedisplayshortskip}{3pt}%
\setlength{\belowdisplayshortskip}{3pt}
	\textbf{Input:} $\mathcal{A}_t$, $f_t$, $\mathrm{ver} \in \{0,1\}$
\begin{algorithmic}
	\IF{$\mathrm{ver} = 0$ }
	\STATE Re-scale the arms: $$\overline{\mathcal{A}}_{(t)} = \{ \sqrt{f_{t}(a)} \cd a \mid a \in \mathcal{A}_t\}$$
 \vspace{-.7em}
	\ELSE
	\STATE Eliminate highly sub-optimal arms:  $$\overline{\mathcal{A}}_{(t)} = \{a \in \mathcal{A}_t : f_{t}(a) \geq \frac{1}{e}\}$$.
 \vspace{-.7em}
	\ENDIF
	\STATE Compute $q_t^\opt = \ApproxDesign(\overline{\mathcal{A}}_{(t)})$.
        \RETURN $q_t^\opt$
\end{algorithmic}
\caption{ApproxDesignAugmented}
\label{main-algo:LinMED-sub-routine}
\end{algorithm}

In order to compute the arm sampling probability, we leverage the G-optimal design of experiments~\citep{kiefer60theequivalence}.
Specifically, we assume that we have access to a computation oracle denoted by
\begin{align*}
	\mathrm{ApproxDesign}(\cB)
\end{align*} that takes in a set of vectors $\cB$ and outputs a distribution over the set $\cB$.
We assume that $\ApproxDesign()$ satisfies the following two assumptions.
	%
		
		\begin{assumption} (The design optimality)\label{main-assump:opt-lev-scr-assumption} 
			Given a set of vectors $\cB \subset \RR^d$, the oracle $\ApproxDesign(\cB)$ returns a $C_{\mathrm{opt}}$-optimal design $q\in\Delta(\cB)$ for the set $\cB$; i.e.
			\begin{align*}
				\lVert b \rVert_{V^{-1}(q)}^2 &\leq C_\opt d \log(d), \forall b \in \cB~
			\end{align*}
		\end{assumption}
		where $V(q) := \sum_{b \in \cB} q_b b b^\T$ for $q \in \Delta(\cB)$.
		Furthermore, we assume that the support size of the design is small as follows:
		\begin{assumption}(Cardinality of design)\label{main-assump:opt-cardinality-assumption} 
			Given a set of vectors $\cB \subset \RR^d$, the oracle $\ApproxDesign(\cB)$ returns a design $q\in\Delta(\cB)$ for the set $\cB$ such that
			\begin{align*}
				\lvert \supp(q) \rvert &= \tilde{\mathcal{O}}(d) ~.
			\end{align*}
		\end{assumption}
		 Existence of such an oracle satisfying Assumptions \ref{main-assump:opt-lev-scr-assumption} and \ref{main-assump:opt-cardinality-assumption} is guaranteed by Kiefer–Wolfowitz~\citep{kiefer60theequivalence}, and there are efficient algorithms for solving it~\citep{todd16minimum}. We present one such $\mathrm{ApproxDesign}()$ algorithm in Appendix \ref{app-section:approx-design-section}.

        We compute $q_t^\opt = \ApproxDesignAugmented(\bar\cA_{(t)})$. \kj{[D ] we need to choose between $\ApproxDesign$ and $\text{ApproxDesignAugmented}$} Subsequently, $q_t$ is calculated as outlined in \eqref{main-eq:q_t}, wherein a weight of $\alpha_{\mathrm{opt}}$ is allocated to $q_t^{\mathrm{opt}}$, $\alpha_{\mathrm{emp}}$ is assigned to the empirical best arm, and the remaining weight is distributed among all the arms in $\mathcal{A}_t$. We then sample arm $A_t$ according to the distribution $p^{'}_t$ defined in \eqref{eq:p_tm1} whenever the set $\mathcal{B}_t$ defined in \eqref{main-eq:event-b} is empty, otherwise we delegate one half of the probability to an arbitrarily chosen arm from $\mathcal{B}_t$.
		Finally, we observe the reward and update the estimator $\hat{\theta}_t$ for the next round.

		\section{MAIN RESULTS} \label{main-section:thm-statemnt-section}

        We now provide regret guarantees of LinMED.
        For the instance-dependent regret bound, we will use the following assumption.
		\begin{assumption}(Lower bound for sub-optimality gap)\label{main-assump:osub-opt-assumption} 
            There exists a constant $\Dt>0$ such that
		\begin{align}\label{main-eq:smallest-gap-definition}
			\Delta \leq \min_{t \in [n], a \in\cA_t: \Delta_{a,t} > 0 } \Delta_{a,t}, \spacex \text{almost surely.}
		\end{align}
	\end{assumption}
		
        Furthermore, we define the true confidence radius
			\begin{align}\label{eq:beta-star}
				\beta^{*}_{t}(\delta_t) := \del[4]{\sigma_* \sqrt{\log \left(\frac{\det V_{t}}{\det V_0} \right) + 2\log\frac{1}{\delta_t}} + \sqrt{\lambda}S_*}^2
			\end{align}
		where $S_* := \|\th^*\|_2 $.
        We define
		\begin{align*}
			H_{\mathrm{max}} := \max_{t \in [n]}\exp \left(\frac{\beta^{*}_{t-1}(\delta_{t-1})}{\beta_{t-1}(\delta_{t-1})}\right).
		\end{align*}
		
	We first state two generic theorems guaranteeing the regret bound of LinMED for any input $\lambda$, $\sig^2$, and $S$, followed by a more concise results with a particular choices of $\lambda$ under the over-specification and under-specification (of $\sigma_*^2$ and $S_*$) cases respectively.

    Furthermore, for all upcoming instance-dependent results, we ignore all logarithmic factors except those related to $n$ and omit terms that do not involve $\polylog(n)$ or $\frac{1}{\Delta}$. Similarly, for all upcoming minimax results, we ignore logarithmic factors except those related to $n$ and omit terms that  do not involve $\poly(n)$.

		\begin{theorem}[Instance-dependent bound]\label{main-thm:inst-dep-reg-bound-theorem}
			Under Assumptions \ref{main-assump:env-assumption}, \ref{main-assump:opt-lev-scr-assumption}, and \ref{main-assump:opt-cardinality-assumption}, with $\delta_t = \frac{1}{t+1}$, LinMED satisfies, $\forall n \geq 1$,
			\begin{align*}
				&\EE\Reg_n =\\
				& O\bigg( \frac{1}{\Delta} d \log (n) \bigg(\left(\sigma^2 d\log(n) + \lambda S^2 \right)\log \left(\log n  \right) +  \\
				&\spacex  \left(\sigma_*^2 d\log(n) + \lambda S_*^2 \right) H_{\mathrm{max} } \bigg) \bigg)
			\end{align*}
			\normalsize
		\end{theorem}
		\begin{theorem}[Minimax bound]\label{main-thm:mini-max-reg-bound-theorem}
			Under Assumptions \ref{main-assump:env-assumption}, \ref{main-assump:opt-lev-scr-assumption}, and \ref{main-assump:opt-cardinality-assumption}, with $\delta_t = \frac{1}{t+1}$, LinMED satisfies, $\forall n \geq 1$, 
			\begin{align*}
				&\EE\Reg_n =
				O\bigg( \sqrt{n} \bigg( \log^{\frac{1}{2}}(n) \left(d \sigma \log(n) + \frac{\lambda S^2}{\sigma} \right) + \\
				&\spacex\hspace{12ex}  \frac{ H_{\mathrm{max}}  }{  \sigma \log^{\frac{3}{2}}(n) } \left( d \sigma_*^2 \log(n) + \lambda S_*^2 \right) \bigg)   \bigg).
			\end{align*}
		\end{theorem}
		It is important to emphasize that in general the learner does not have access to the true sub-Gaussian parameter ($\sigma_*^2$) of the noise and $S_*$. The input sub-Gaussian parameter ($\sigma^2$) and $S$ may either over-specified or under-specified with respect to their true values. Nevertheless, our algorithm provides a regret bound that remains valid across all such scenarios \kj{[D ] make it very clear}, a feature absent in the analysis of most of the state-of-the art algorithms such as OFUL, LinTS, and LinIMED. This constitutes one of the novel contributions of our analysis. It is noteworthy that, at first glance, one might be misled into believing that selecting smaller values for \( S \) and \( \sigma \) results in a smaller regret bound. However, this is not the case, as \( H_{\mathrm{max}} \) increases exponentially as \( S \) and \( \sigma \) decrease.
		
		Consider the case where the true sub-Gaussian parameter ($\sigma_*^2$) and $S_*$ are over-specified, $H_{\mathrm{max}}$ tends to be less than $\exp(1)$, leading to the following corollaries:
		
		\begin{corollary}[Instance-dependent bound]\label{main-cor:inst-dep-reg-bound-corollary}
			Under Assumptions \ref{main-assump:env-assumption}, \ref{main-assump:opt-lev-scr-assumption}, and \ref{main-assump:opt-cardinality-assumption}, assuming $\sigma^2 \geq \sigma_*^2$, $S \geq S_*$ with $\lambda = \frac{\sigma^2}{S^2}$ and $\delta_t = \frac{1}{t+1}$,  LinMED satisfies, $\forall n \geq 1$,
			\begin{align*}
				&\EE\Reg_n = O\left( \sigma^2 \frac{d^2}{\Delta}  \log^2 (n) \log \left(\log n\right)\right) .
			\end{align*}
			\normalsize
		\end{corollary}
		\begin{corollary}[Minimax bound]\label{main-cor:mini-max-reg-bound-corollary}
		Under Assumptions \ref{main-assump:env-assumption}, \ref{main-assump:opt-lev-scr-assumption}, and \ref{main-assump:opt-cardinality-assumption}, assuming $\sigma^2 \geq \sigma_*^2$, $S \geq S_*$ and with $\lambda = \frac{\sigma^2}{S^2}$ and $\delta_t = \frac{1}{t+1}$, LinMED satisfies, $\forall n \geq 1$,
			\begin{align*}
				&\EE\Reg_n = O\left( \sig d\sqrt{n}\log^{\frac{3}{2}}(n)  \right).
			\end{align*}
		\end{corollary}
	Instance-dependent bound of LinMED showcases a $\log(n)$ improvement over the instance-dependent bound of OFUL \citep{ay11improved} and LinMED guarantees an optimal minimax bound up to logarithmic factors. 

    Next,  we consider the case where the $\sigma_*^2$ is under-specified and $S_*$ is over-specified.
  
  \begin{corollary}[Minimax bound]\label{main-cor:mini-max-reg-bound-corollary-under-est}
  	Under Assumptions \ref{main-assump:env-assumption}, \ref{main-assump:opt-lev-scr-assumption}, and \ref{main-assump:opt-cardinality-assumption}, assuming $\sigma^2 < \sigma_*^2$, $S \geq S_*$ and with $\lambda = \frac{\sigma^2}{S^2}$ and $\delta_t = \frac{1}{t+1}$, $\forall n \geq 1$, LinMED satisfies
  	\begin{align*}
  		\EE\Reg_n =
  		O \del[4]{\frac{\sigma d  \sqrt{n} }{ \log^{\frac{1}{2}}(n) }\del[3]{ \log^{2}(n) + \frac{\sigma_{*}^2}{\sigma^2} \exp\del[2]{\frac{\sigma_*^2}{\sigma^2}} } }
  	\end{align*}
  \end{corollary}

	One can derive the instance-dependent bound and bounds for under-specified $S_*$ in a similar fashion.
		Proofs of the theorems and corollaries are deferred to the appendix. 
		
		\textbf{The key steps of the proof of Theorem~\ref{main-thm:inst-dep-reg-bound-theorem} and~\ref{main-thm:mini-max-reg-bound-theorem}}.
		Conceptually, our proof structure for Lemma \ref{app-lemma:regretlemma} closely follows the framework of the Maillard sampling proof by \citet{bian22maillard}. We define the following events:
		$\mathcal{U}_{t-1,\ell}(A_t) = \cbr[1]{ \lVert A_t \rVert^2_{V_{t-1}^{-1}} \geq \eps_{\ell}}$, $\mathcal{V}_{t-1}(A_t) = \cbr[1]{\hat{\Delta}_{A_t,t} \geq \frac{\Delta_{A_t,t}}{1+c}}$, $\mathcal{W}_{t-1,\ell}  = \cbr[1]{ \max_{a^{'} \in \mathcal{A}_{t}} \langle \hat{\theta}_{t-1}, a^{'} \rangle \geq \langle \theta^*, a^*_t \rangle - \eps_{2, \ell} }$, where $\ell, \eps_{\ell}$, $\eps_{2, \ell}$ are parameters to be tuned.
		At an abstract level, the regret can be decomposed as follows:
		\begin{align*}
			\mathrm{Reg}_n &= \EE\sbr[3]{\sum_{t=1}^{n} \Delta_{A_t,t}  } = \EE\sbr[3]{\sum_{t=1}^{n} \Delta_{A_t,t} \one \cbr{\mathcal{U}_{t-1,\ell}(A_t)} } \tag{Term 1}\\
			&\spacex + \EE\sbr[3]{\sum_{t=1}^{n} \Delta_{A_t,t} \one \cbr{\overline{\mathcal{U}}_{t-1,\ell}(A_t)} \one\cbr{\mathcal{V}_{t-1}(A_t)} } \tag{Term 2}\\
			& \spacex + \EE\sbr[3]{ \sum_{t=1}^{n} \Delta_{A_t,t}  \one\cbr{\overline{\mathcal{V}}_{t-1}(A_t)} \one\sbr{\mathcal{W}_{t-1,\ell}  } } \tag{Term 3} \\
			& \spacex + \EE\sbr[3]{ \sum_{t=1}^{n} \Delta_{A_t,t}  \one\sbr{\overline{\mathcal{W}}_{t-1,\ell}  } }. \tag{Term 4} \\
		\end{align*}
		We bound Term 1 and Term 3 using the elliptical potential count (EPC), as shown in Lemma \ref{app-lemma:epc-lemma}. Term 2 is bounded by noting that the probability of selecting sub-optimal arms is small when the events $ \one \cbr[1]{\overline{\mathcal{U}}_{t-1,\ell} (A_t)}$ and $\one\cbr{\mathcal{V}_{t-1}(A_t)} $ occur. 
    Term 4 is the most challenging one where drawing an analogue from Maillard sampling's proof is nontrivial.
    Detailed proof is presented in Appendix \ref{app-section:gmain-proof}.
		
		\section{INSTANCE-DEPENDENT LOWER BOUNDS FOR SPANNERIGW AND EXP2} \label{main-section:inst-dep-low-bound-spannerigw-exp2-section}

		In this section, we analyze the instance-dependent regrets for EXP2 and SpannerIGW.  We show that there are instances for which the above two algorithms have an instance-dependent bound of $\Omega(\Delta \sqrt{n})$. 
        Hence, LinMED stands out as a leading randomized algorithm with closed-form arm sampling probabilities, achieving a logarithmic instance-dependent regret bound.
		
		\begin{theorem}\label{main-thm:exp2-inst-dep-bound-theorem}
			There exists a linear bandit problem for which the EXP2 algorithm satisfies
			\begin{align*}
				\EE\Reg_n \geq \Omega(\Delta \sqrt{n}).
			\end{align*}
		\end{theorem}

		\begin{theorem}\label{main-thm:spannerigw-inst-dep-bound-theorem}
			There exists a linear bandit problem for which the SpannerIGW algorithm satisfies
			\begin{align*}
				\EE \Reg_n \geq \Omega(\Delta \sqrt{n}).
			\end{align*}
		\end{theorem}
		
		The proofs are deferred to the Appendix \ref{app-section:low-bound-args-section}.

			\begin{figure}[t]
				\centering
				\begin{subfigure}{0.5\textwidth}
					\centering
					\includegraphics[width=\mygapin\linewidth]{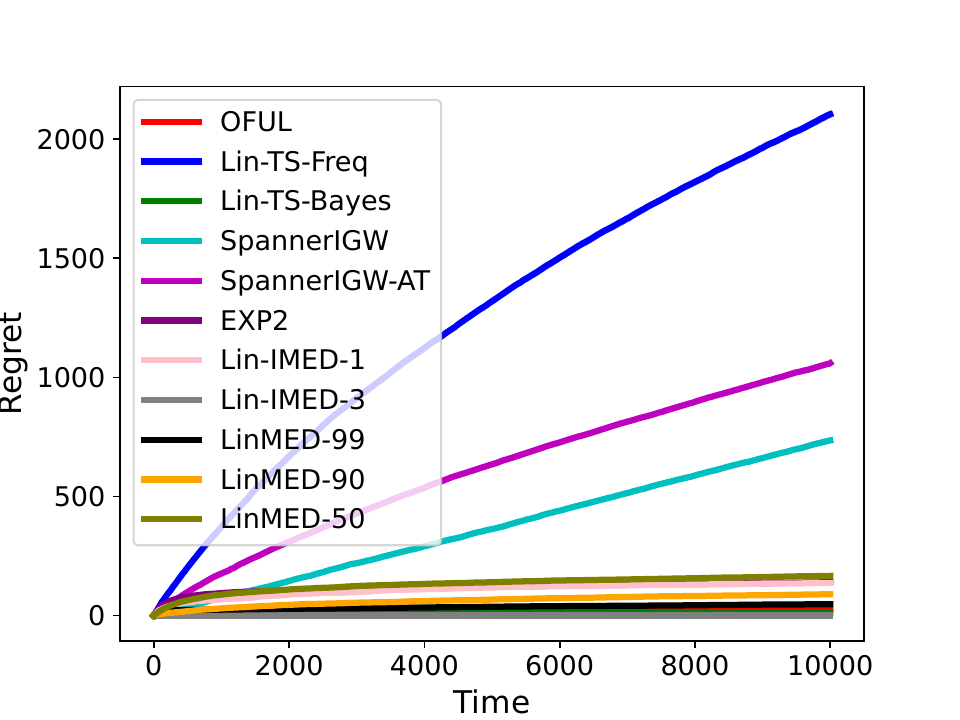}
					\caption{}
					\label{main-subfigure:inst-dep-sims-subfigure-1}
				\end{subfigure}%
				
				\begin{subfigure}{0.5\textwidth}
					\centering
					\includegraphics[width=\mygapin\linewidth]{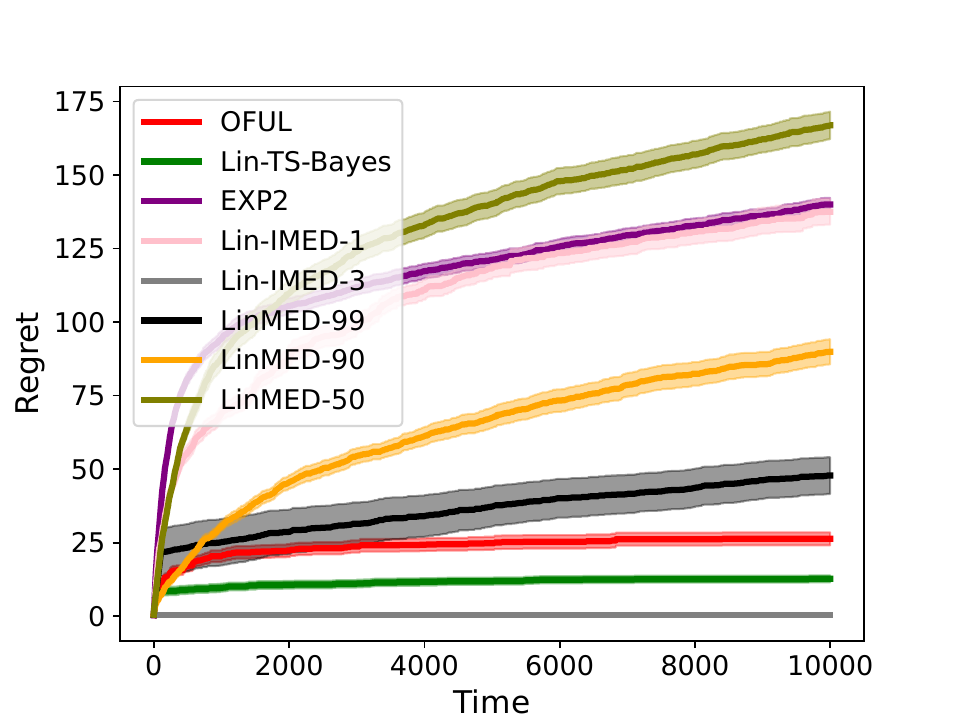}
					\caption{}
					\label{main-subfigure:inst-dep-sims-subfigure-2}
				\end{subfigure}
				\caption{Large gap instance experiments}
				\label{fig:fig1}
			\end{figure}

			\begin{figure*}[t!]
				\centering
				\begin{subfigure}[b]{0.5\textwidth}
					\centering
					\includegraphics[width=\mygapin\textwidth]{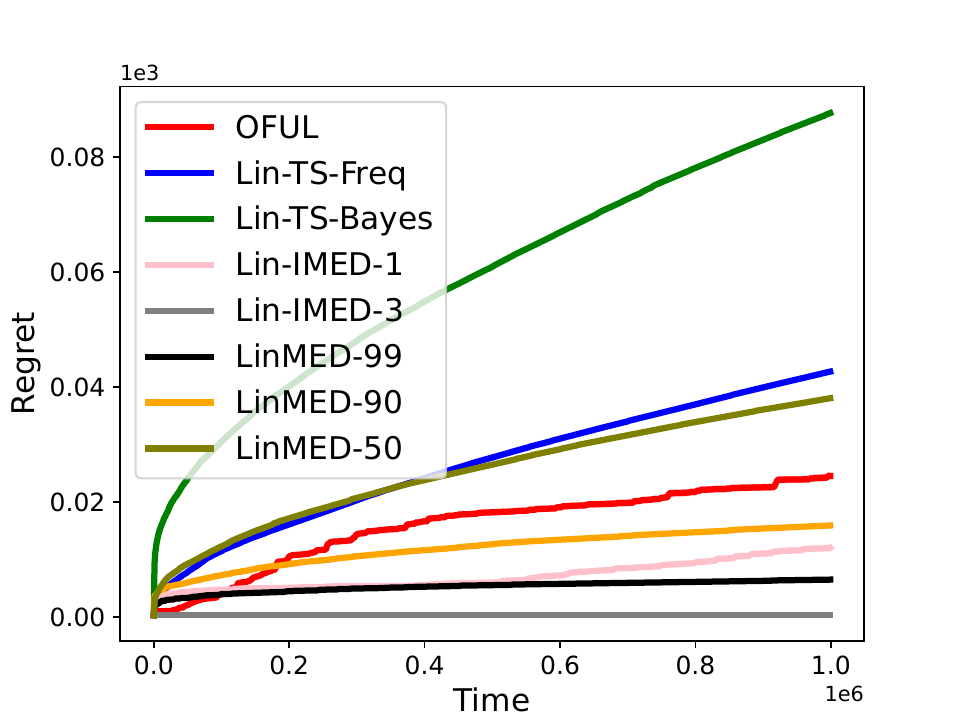}
					\caption{ $\eps = 0.005$, $\sigma^2 = \sigma^2_* $}
				\end{subfigure}%
				\begin{subfigure}[b]{0.5\textwidth}
					\centering
					\includegraphics[width=\mygapin\textwidth]{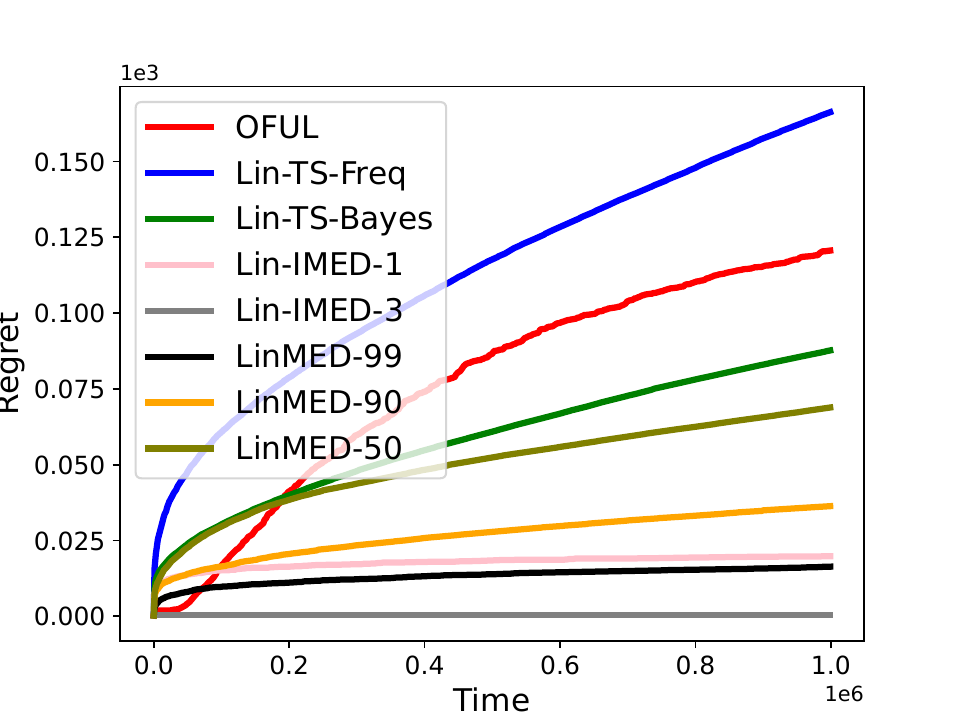}
					\caption{ $\eps = 0.005$, $\sigma^2 = 2\cdot \sigma^2_* $}
				\end{subfigure}
				
				\begin{subfigure}[b]{0.5\textwidth}
					\centering
					\includegraphics[width=\mygapin\textwidth]{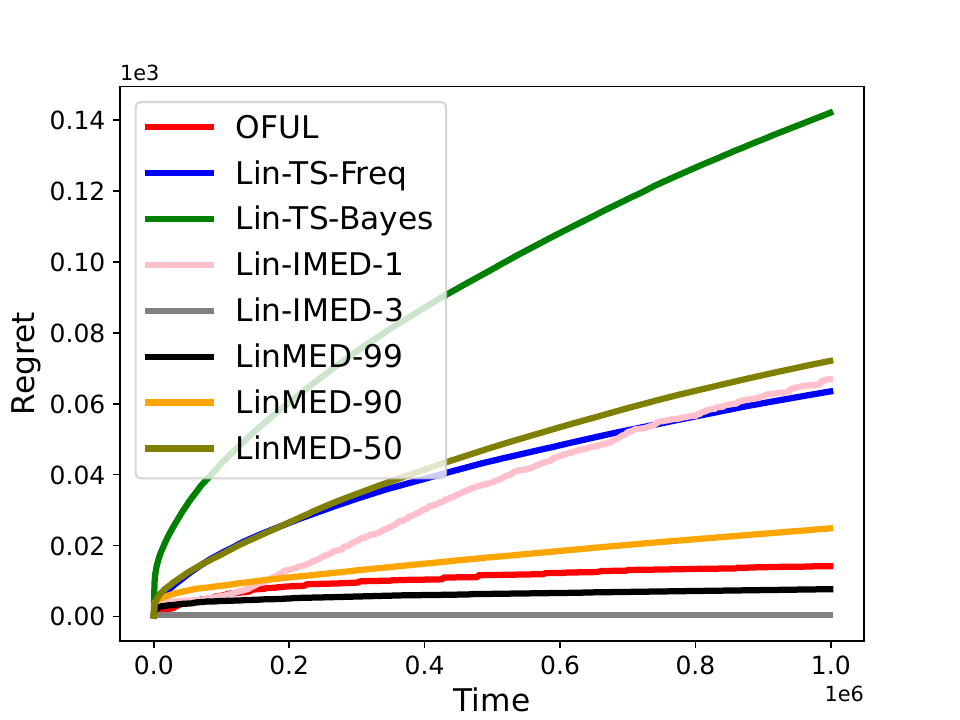}
					\caption{$\eps = 0.01$, $\sigma^2 = \sigma^2_* $}
				\end{subfigure}%
				\begin{subfigure}[b]{0.5\textwidth}
					\centering
					\includegraphics[width=\mygapin\textwidth]{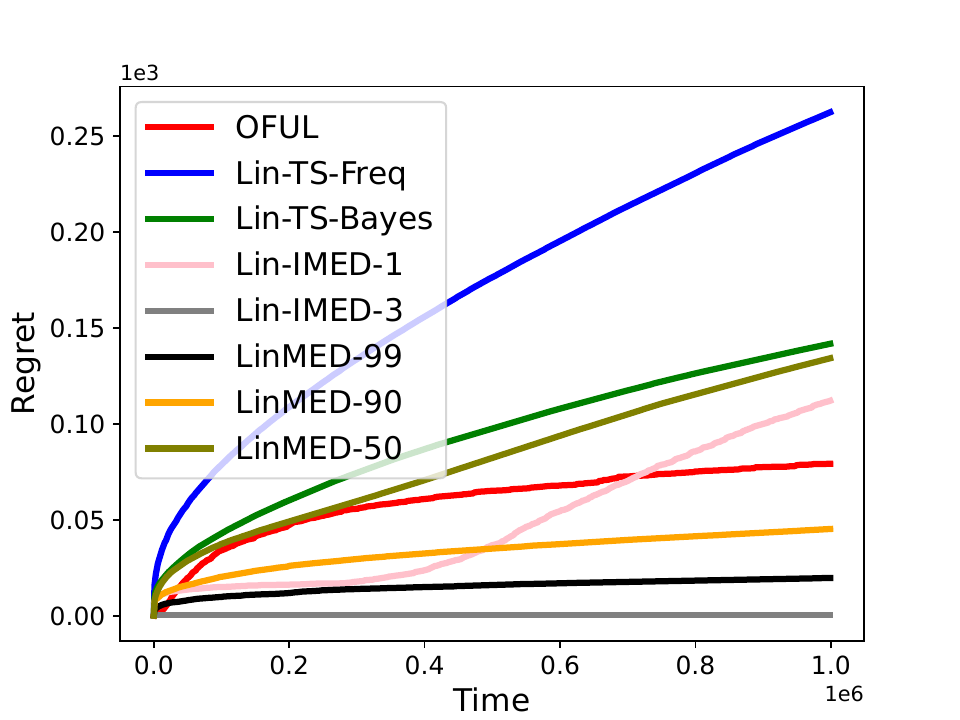}
					\caption{$\eps = 0.01$, $\sigma^2 = 2\cdot \sigma^2_* $}
				\end{subfigure}
				\caption{End of optimism experiments }
				\label{main-figure:end-of-opt-figure}
    \vspace{-1em}
			\end{figure*}
	  
		\section{EMPIRICAL STUDIES}\label{main-section:emp-studies-section}
		This section is dedicated to demonstrating the effectiveness of our algorithm in comparison to several well-known algorithms across various scenarios, each of which evaluates different aspects of algorithmic performance.
		Throughout our empirical studies, we fine-tune $(\alpha_{\text{emp}},\alpha_{\text{opt}})$ for LinMED algorithm using the following values: $(0.99,0.005)$,$(0.90,0.05)$, and $(0.5,0.25)$. We refer to the resulting variants as LinMED-99, LinMED-90, LinMED-50 respectively.
		
		SpannerIGW utilizes exploration parameters  $\gamma$ and $\eta$, which are dependent on the time horizon $n$ and remain fixed throughout all rounds. We modify these parameters to use $t$ in place of $n$ at each time step $t$, thereby deriving an anytime version of the algorithm, which we refer to as SpannerIGW-Anytime or SpannerIGW-AT. LinIMED has three variants, namely LinIMED-1, LinIMED-2, and LinIMED-3. 
        In our study, we use LinIMED-1 and LinIMED-3 only, as prior experiments in~\citet{bian24indexed} show that LinIMED-2 consistently performs between these two algorithms. Therefore, evaluating LinIMED-1 and LinIMED-3 sufficiently captures both ends of the performance spectrum. Additionally, LinIMED-3 has a parameter $C$, which we set to $30$, following~\citet{bian24indexed}. We also use a modified EXP2 algorithm (based on rewards instead of losses), presented in Algorithm \ref{app-algo:exp2-algorithm}, and refer to it simply as EXP2 for clarity.
		
		Furthermore, it is noteworthy that, throughout our experiments, we select either the frequentist (Lin-TS-Freq) \citep{agrawal14thompson} or Bayesian version (Lin-TS-Bayes) \citep{russo14learning} of LinTS, or both. However, whenever we choose only one version, it implies that the selected version significantly outperforms the omitted one. We use LinTS to refer to both Lin-TS-Freq and Lin-TS-Bayes.

		\textbf{Large gap instance.}
		Our algorithm achieves an instance-dependent regret bound of $O(\log^2n)$ with respect to $n$ omitting $\log\left(\log(n)\right)$ terms. This is much better than EXP2 and SpannerIGW, both of which have an instance-dependent lower bound in the order of $\Omega(\sqrt{n})$ with respect to $n$.  Our instance-dependent regret bound also shows a $\log(n)$ factor improvement over the original analysis of OFUL \citep{ay11improved}.
		The instance-dependent regret bounds for LinTS, LinIMED-1, and LinIMED-3 are not known to our knowledge.
		
		The experimental setup of this scenario is as follows: $\mathcal{A} = \{(1,0), (0,1)\}$ and $\theta^* = (1,0)$. 
		The noise follows a normal distribution $\mathcal{N}(0,\sigma_*^2)$ with $ \sigma^2 = \sigma_{*}^2 = 1$. The time horizon for each trial is $n=10,000$ and conduct $10$ such independent trials. We compare our algorithm against SpannerIGW \citep{zhu22contextual}, SpannerIGW-Anytime, LinIMED-1, LinIMED-3 \citep{bian24indexed}, OFUL \citep{ay11improved}, Lin-TS-Bayes (Bayesian version)\citep{russo14learning}, Lin-TS-Freq (Frequentest version) \citep{agrawal14thompson}, and EXP2 \citep{bubeck12towards}.
		
		Our simulations indicate that our algorithm outperforms SpannerIGW, SpannerIGW-Anytime, Lin-TS-Freq, EXP2, and LinIMED-1.
        Furthermore, our algorithm demonstrates performance that is sufficiently close to that of OFUL \citep{ay11improved}, LinIMED-3, and Lin-TS-Bayes. Figure \ref{main-subfigure:inst-dep-sims-subfigure-1} presents the primary plot of our results, while \ref{main-subfigure:inst-dep-sims-subfigure-2} displays the same data, with SpannerIGW, SpannerIGW-Anytime, and Lin-TS-Freq removed for a more precise comparison of the remaining algorithms. Furthermore, close visual inspection confirms the instance-dependent regret lower bound of $\Omega(\sqrt{n})$ that we proved in Section \ref{main-section:inst-dep-low-bound-spannerigw-exp2-section} for both EXP2 and SpannerIGW.
\looseness=-1

		\textbf{End of Optimism instance}
		The ``end of optimism instance'' \cite{lattimore17theend} is often cited as a pitfall for optimism-based algorithms such as OFUL. Inspired by the end of optimism-based simulations conducted by \citet{bian24indexed}, we perform similar experiments to evaluate the performance of our algorithm in comparison to OFUL, LinIMED-1, LinIMED-3, and LinTS.
		In this context, OFUL and LinTS are classified as optimistic algorithms, while LinIMED-1 and LinIMED-3 are minimum empirical divergence-based deterministic algorithms. We set the number of arms $K=3$ and dimension $d=2$ and $\mathcal{A} = \{a_1 = (1,0), a_2 = (0,1),a_3 = (1 - \eps, 2\eps)\}$ where $\eps \in \{0.005,0.01,0.02\}$ and $\theta^* = (1,0)$. The noise follows $\mathcal{N}(0, \sigma_*^2)$ with $\sigma_* = 0.1$.  The time horizon for each trial is $n=1000,000$ and conduct $20$ such independent trials. Furthermore, we conduct experiments for the cases where $\sigma^2 = \sigma_*^2 $ and $\sigma^2 = 2\cdot \sigma_*^2 $.
		
		Optimism-based algorithms typically identify the optimal arm ($a_1$) and the near-optimal arm ($a_3$) as the optimistic choices, frequently pulling these two arms. This behavior limits their ability to pull the highly sub-optimal arm ($a_2$), which provides a crucial piece of information for distinguishing between the optimal and near-optimal arms. As a result, optimistic algorithms struggle to differentiate effectively between these two arms, often incurring a small regret from repeatedly selecting the near-optimal arm ($a_3$) for an extended period.
  \looseness=-1
		
		In contrast, algorithms that do not follow optimistic principles explore adequately in the direction of the highly sub-optimal arm as well. Consequently, the trend of our algorithm reveals that it initially incurs significant regret by choosing the highly sub-optimal arm but ultimately converges on the optimal arm as a consistent choice. 
		
		From Figure \ref{main-figure:end-of-opt-figure}, it is evident that LinIMED-3 and LinMED ($\alpha_{\mathrm{emp}}=0.99$) perform very well and significantly out performs LinTS whereas LinMED ($\alpha_{\mathrm{emp}}=0.90$) and LinMED ($\alpha_{\mathrm{emp}}=0.50$) exhibit comparable performance. However, LinIMED-1 and OFUL start with good performance, but their effectiveness deteriorates over time, especially when $\eps$ is too small. This effect is amplified when $\sigma^2 = 2\cdot \sigma^2_*$. When the noise is over-specified, the performance of Lin-TS-Freq deteriorates significantly due to oversampling. We present detailed results in Appendix \ref{app-subsection:EOPT-eval-exp-subsection}.

	\section{CONCLUSION}\label{main-section:future-works-section}

        Our proposed algorithm LinMED posesses many intriguing properties and shows excellent empirical performance, which opens up exciting avenues for future research.
        First, it would be interesting to explore ways to generalize the noise model to exponential family (generalized linear models) or generalize the linear class to generic hypothesis class.
        Identifying fundamental limits of adapting to the unknown sub-Gaussian noise level would be interesting and important.
        Second, we believe the challenge of coping with the unknown noise level is an important problem that has received less attention in the literature. 
        Relatedly, \citet{jun24noiseadaptive} have shown that adapting to the unknown sub-Gaussian parameter $\sig_*^2$ is possible when it is overspecified.
        Finally, it would be intriguing to develop pure exploration or Bayesian optimization version of LinMED and explore the potential of the MED principle.
        
		
\subsubsection*{Acknowledgements}
Kapilan Balagopalan and Kwang-Sung Jun were supported in part by the National Science Foundation under grant CCF-2327013.
	
\bibliographystyle{abbrvnat_lastname_first_overleaf}
\bibliography{library-shared}


\clearpage
\section*{Checklist}

\begin{enumerate}

	\item For all models and algorithms presented, check if you include:
	\begin{enumerate}
		\item A clear description of the mathematical setting, assumptions, algorithm, and/or model. [Yes]
		\item An analysis of the properties and complexity (time, space, sample size) of any algorithm. [Yes]
		\item (Optional) Anonymized source code, with specification of all dependencies, including external libraries. [No]
	\end{enumerate}

	\item For any theoretical claim, check if you include:
	\begin{enumerate}
		\item Statements of the full set of assumptions of all theoretical results. [Yes]
		\item Complete proofs of all theoretical results. [Yes]
		\item Clear explanations of any assumptions. [Yes]     
	\end{enumerate}

	\item For all figures and tables that present empirical results, check if you include:
	\begin{enumerate}
		\item The code, data, and instructions needed to reproduce the main experimental results (either in the supplemental material or as a URL). [No]
		\item All the training details (e.g., data splits, hyperparameters, how they were chosen). [Not Applicable]
		\item A clear definition of the specific measure or statistics and error bars (e.g., with respect to the random seed after running experiments multiple times). [Not Applicable]
		\item A description of the computing infrastructure used. (e.g., type of GPUs, internal cluster, or cloud provider). [Not Applicable]
	\end{enumerate}
	
	\item If you are using existing assets (e.g., code, data, models) or curating/releasing new assets, check if you include:
	\begin{enumerate}
		\item Citations of the creator If your work uses existing assets. [Yes]
		\item The license information of the assets, if applicable. [Not Applicable]
		\item New assets either in the supplemental material or as a URL, if applicable. [Not Applicable]
		\item Information about consent from data providers/curators. [Not Applicable]
		\item Discussion of sensible content if applicable, e.g., personally identifiable information or offensive content. [Not Applicable]
	\end{enumerate}
	
	\item If you used crowdsourcing or conducted research with human subjects, check if you include:
	\begin{enumerate}
		\item The full text of instructions given to participants and screenshots. [Not Applicable]
		\item Descriptions of potential participant risks, with links to Institutional Review Board (IRB) approvals if applicable. [Not Applicable]
		\item The estimated hourly wage paid to participants and the total amount spent on participant compensation. [Not Applicable]
	\end{enumerate}
	
\end{enumerate}

\onecolumn
\appendix
\aistatstitle{Supplementary Materials}

\fancypagestyle{plain}{
\fancyfoot[C]{\thepage}}
\pagestyle{plain}
\setlength{\footskip}{20pt}	

\part{Appendix} 

\parttoc 

\graphicspath{{Figures/}} 
\clearpage
\section{NOTATIONS}
Let us define the relevant quantities: While some of these have already been introduced in the main body of the paper, we redefine them here for ease of reference.

\begin{itemize}
	\item $a^*_t := \argmax_{a \in \mathcal{A}_{t}} \langle \theta^*, a \rangle$
	\item $\eps_{\ell} := 2^{-2\ell}\cdot \eps$ where $\eps$ is a parameter to be determined later
	\item $\eps_{2,\ell} := 2^{-\ell}\cdot \eps_2$  where $\eps_2$ is a parameter to be determined later
	\item $\Delta_{a,t} := \langle \theta^*, a^*_t \rangle - \langle \theta^*, a \rangle$
	\item $\Delta_{t} := \Delta_{A_t,t}$
	\item $\Delta := \min_{t \in [n], a \in \cA_t: \Delta_{a,t} > 0 } \Delta_{a,t} $
	\item $\overline{ \Delta}_{a,t} := B \wedge \Delta_{a,t} $
	\item $\overline{ \Delta}_{t} := B \wedge \Delta_{A_t,t} $
	\item $\sqrt{\beta_{t-1}(\delta_{t-1})} := \sigma \sqrt{ 2 \log\left(\frac{\det\left( V_{t-1} \right)^{\frac{1}{2}} \det\left( \lambda I \right)^{-\frac{1}{2}} }{\delta_{t-1}}\right)} + \sqrt{\lambda}S$
	\item $\sqrt{\beta^{*}_{t-1}(\delta_{t-1})} := \sigma_{*} \sqrt{ 2 \log\left(\frac{\det\left( V_{t-1} \right)^{\frac{1}{2}} \det\left( \lambda I \right)^{-\frac{1}{2}} }{\delta_{t-1}}\right)} + \sqrt{\lambda}S_*$
	\item $H_{\mathrm{max}} := \max_{t \in [n]}\exp \left(\frac{\beta^{*}_{t-1}(\delta_{t-1})}{\beta_{t-1}(\delta_{t-1})}\right) $
	\item $V(p_t) := \sum_{a \in \mathcal{A}_t} p_t(a) aa^{\T}$
	\item $ \overline{V}(p_t) := \sum_{a \in \overline{ \mathcal{A}}_{(t)}} p_t(a) (\bar{a}_{(t)})(\bar{a}_{(t)})^{\T}$ where $\bar{a}_{(t)} = \sqrt{f_{t}(a)} \cd a $
\end{itemize}
Let $\mathcal{F}_t$ be the $\sigma$-algebra generated by $\left( A_1, Y_1, A_2, Y_2,......A_t,Y_t \right)$. We define the following shortcuts:
\begin{itemize}
\item 	$\PP_{t}(\mathcal{E}) := \PP(\mathcal{E} \mid \mathcal{F}_t)$ 
\item 	$\EE_{t}\sbr{\mathcal{E}} := \EE\sbr{\mathcal{E} \mid \mathcal{F}_t}$  
\item  	$\one_{\hat{a}_t}(a) := \one\cbr{a = \hat{a}_t}$
\end{itemize}

\section{ASSUMPTIONS}
We would like to remind you of the Assumptions \ref{main-assump:env-assumption},  \ref{main-assump:opt-lev-scr-assumption}, \ref{main-assump:opt-cardinality-assumption}

\section{PROOFS}\label{app-section:gmain-proof}

\subsection{Good event}
The following "good event" is derived from the work of \citet{ay11improved} and occurs with a probability of at least $1 - \delta_{t-1}$, as established in Lemma \ref{app-lemma:OFUL-conf-bound-lemma}.
\begin{align}\label{app-event:good-event-1}
    \mathcal{G}_1 = \vast \{ \lVert \theta^* - \hat{\theta}_{t-1}\rVert_{V_{t-1}}^2 \leq \beta^{*}_{t-1}(\delta_{t-1}), \spacex \forall t \geq 1\vast \}. 
\end{align}

\subsection{Conditioning events}
In the main body of the paper, we presented a high-level proof sketch. Here, we provide a more detailed analysis by breaking down and evaluating the regret case by case, according to the following events:
\begin{align*}
    \mathcal{U}_{t-1,\ell}(a) &= \cbr{ \lVert a \rVert^2_{V_{t-1}^{-1}} \geq \eps_{\ell}}\\
    \mathcal{U}_{t-1}(a) &= \cbr{ \lVert a \rVert^2_{V_{t-1}^{-1}} \geq 1}\\
    \mathcal{E}_{t}  &= \cbr{ \lvert \mathcal{B}_t \rvert >  0}\\
    \mathcal{V}_{t-1}(a) &= \cbr{\hat{\Delta}_{a,t} \geq \frac{\Delta_{a,t}}{1+c}}\\
    \mathcal{W}_{t-1,\ell}  &= \cbr{ \max_{a^{'} \in \mathcal{A}_{t}} \langle \hat{\theta}_{t-1}, a^{'} \rangle \geq \langle \theta^*, a^*_t \rangle - \eps_{2, \ell} }\\
    \mathcal{D}_{t,\ell}(a) &= \cbr{ B \cdot 2^{-\ell} < \overline{ \Delta}_{a,t}  \leq B  \cdot 2^{-\ell + 1}} \tag{$\overline{ \Delta}_{a,t} := B \wedge \Delta_{a,t} $}\\
    \overline{\mathcal{D}}_{t,L}(a) &=  \cbr{ \overline{ \Delta}_{a,t}  \leq B \cdot 2^{-L}}.
\end{align*}

\subsection{Proof of Lemma \ref{app-lemma:regretlemma}} \label{app-subsection:regret-lemma-subsection}
In this section, we present the fundamental lemma underlying our regret analysis, which ultimately leads to the theorems and corollaries concerning both the instance-dependent and minimax bounds.
\begin{lemma}[Regret Bound]\label{app-lemma:regretlemma}
	Under Assumptions \ref{main-assump:env-assumption}, \ref{main-assump:opt-lev-scr-assumption}, and \ref{main-assump:opt-cardinality-assumption}, with $\delta_t = \frac{1}{t+1}$, LinMED satisfies, $\forall n \geq 1$,
	\begin{align*}
		\Reg_n &\leq   6dB \left(3 + \frac{1}{\alpha_{\mathrm{emp}}^2}\right)\log \left(1 + \frac{2}{\lambda}\right) + 2B \log (n+1) +   \frac{192\beta_{n}(\delta)  \log (n) d}{2^{-L}B}  \left(1 +  \frac{1}{\alpha_{\mathrm{emp}}^2}\right) \log \left( 1 + \frac{32 \beta_{n}(\delta)  \log (n) }{\lambda 2^{-2L} B^2 }\right)\\
		&\spacex + \frac{192  \beta^{*}_{{n}}(\delta)d}{B \cdot 2^{-L}} \log\left( 1 + \frac{32\beta^{*}_{{n}}(\delta)}{\lambda B^2 \cdot 2^{-2L}}\right) + \frac{512 H_{\mathrm{max}}\cdot C_{\mathrm{opt}}\cdot d\log(d)}{\alpha_{\mathrm{opt}} B \cdot 2^{-L}} \left( \frac{\lambda (S_*)^2}{2} + \sigma_*^2 d \log \left( 1 + \frac{n}{d\lambda}\right) \right)\\
		&\spacex + \BII + \frac{ 4B}{\alpha_{\mathrm{emp}}} .
	\end{align*}
\end{lemma}
\begin{proof}

First, we decompose the proof based on the occurrence of the event $\mathcal{E}_t$. This decomposition is crucial because, if the event occurs, we select an arm arbitrarily from the set $\mathcal{B}_t$ with a probability of at-least one-half, as described in Algorithm \ref{main-algo:LinMED-algo}.  

\begin{align*}
    \Reg_n &= \EE\sbr{\sum_{t=1}^{n}\langle \theta^* , a^*_t \rangle - \langle \theta^* , A_t\rangle}\\
    &= \EE\sbr{\sum_{t=1}^{n} \Delta_{t} \langle \theta^* , a^*_t \rangle - \langle \theta^* , A_t\rangle} \tag{$\Delta_{t} := \Delta_{A_t,t}$}\\
    &\leq \EE\sbr{\sum_{t=1}^{n}\left(B \wedge \Delta_{t}\right) \langle \theta^* , a^*_t \rangle - \langle \theta^* , A_t\rangle}\\
    &= \EE\sbr{\sum_{t=1}^{n}\overline{ \Delta}_{t} \onec{ A_t \neq a^*_t}} \tag{ $\overline{ \Delta}_{t} := B \wedge \Delta_{A_t,t} $}\\
    &=  \EE\sbr{\sum_{t=1}^{n}\overline{ \Delta}_{t} \onec{ A_t \neq a^*_t} \left( \one\cbr{ \overline{\mathcal{E}}_t} + \one \cbr{ \mathcal{E}_t } \right)} \\
    &\leq  \underbrace{\EE\sbr{\sum_{t=1}^{n}\overline{ \Delta}_{t} \onec{ A_t \neq a^*_t}\one \cbr{ \overline{\mathcal{E}}_t } }}_{A_1}  + \underbrace{\EE\sbr{\sum_{t=1}^{n}\overline{ \Delta}_{t} \onec{ A_t \neq a^*_t} \one\cbr{ \mathcal{E}_t} } }_{A_2}.
 \end{align*} 
 
Consequently, if the selected arm $A_t$ is in $\mathcal{B}_t$, it directly implies that $ \lVert A_t \rVert_{V_{t-1}^{-1} }^2  > 1$, in accordance with the definition of the set $\mathcal{B}_t$. Moreover, the expected number of occurrences of the event $ \lVert A_t \rVert_{V_{t-1}^{-1} }^2  > 1$ can be managed using the Elliptical Potential Count (EPC), as demonstrated in Lemma \ref{app-lemma:epc-lemma}.

 \begin{align*}
 	A_2 &= \EE\sbr{\sum_{t=1}^{n}\overline{ \Delta}_{t} \onec{ A_t \neq a^*_t} \one\cbr{ \mathcal{E}_t }} \\
 	&= 2 \EE\sbr{\sum_{t=1}^{n}\overline{ \Delta}_{t} \onec{ A_t \neq a^*_t} \one\cbr{ \mathcal{E}_t } \frac{1}{2} } \\
 	&\leq 2 \EE\sbr{\sum_{t=1}^{n}\overline{ \Delta}_{t} \onec{ A_t \neq a^*_t} \one\cbr{ \mathcal{E}_t } \PP\left( A_t \in \mathcal{B}_t\right) } \\
 	&= 2 \EE\sbr{\sum_{t=1}^{n}\overline{ \Delta}_{t} \onec{ A_t \neq a^*_t} \one\cbr{ \mathcal{E}_t } \PP\left(  \lVert A_t \rVert_{V_{t-1}^{-1} }^2  > 1  \right) } \\
 	&= 2 \EE\sbr{\sum_{t=1}^{n}\overline{ \Delta}_{t} \onec{ A_t \neq a^*_t} \one\cbr{ \mathcal{E}_t }  \EE_{t-1} \sbr{ \one\cbr{ \lVert A_t \rVert_{V_{t-1}^{-1} }^2  > 1  } } } \\
 	&\leq  2\EE\sbr{\sum_{t=1}^{n}\overline{ \Delta}_{t} \onec{ A_t \neq a^*_t} \one\cbr{  \lVert A_t \rVert_{V_{t-1}^{-1} }^2  > 1   } } \\
 	&\leq 2B \cdot \EE\sbr{ \sum_{t=1}^{n} \one\cbr{  \lVert A_t \rVert_{V_{t-1}^{-1} }^2  > 1   } } \\
 	&\leq  2B \cdot 3d \log \left( 1 + \frac{2}{\lambda}\right) \tag{by lemma \ref{app-lemma:epc-lemma}} \\
 	&= 6Bd  \log \left( 1 + \frac{2}{\lambda}\right). 
 \end{align*}
This concludes the bounding of the term $A_2$.

Moving to $A_1$, we utilize the peeling technique on $\overline{ \Delta}_{t}$ to enhance the precision of our analysis. This involves decomposing $\overline{ \Delta}_{t}$ into piecewise ranges defined as $B \cdot 2^{-\ell} <\overline{ \Delta}_{t} \leq B  \cdot 2^{-\ell + 1}$, facilitating a more granular analysis of regret. From a broader perspective, this approach can be likened to approximating the area under a graph using rectangles; the smaller the area of the rectangle, the more precise and accurate the resulting analysis becomes. In our context, this technique significantly reduces the looseness of the analysis by a factor of $\frac{1}{\Delta}$.

 \begin{align*}
    A_1  &\leq \underbrace{ \EE\sbr{\sum_{t=1}^{n}\sum_{\ell = 1}^{L}\overline{ \Delta}_{t} \one\cbr{ \mathcal{D}_{t,\ell}(A_t)} \onec{ A_t \neq a^*_t} \one \cbr{ \overline{\mathcal{E}}_t } }}_{B_1} + \underbrace{\EE\sbr{\sum_{t=1}^{n} \overline{ \Delta}_{t} \one\cbr{\overline{\mathcal{D}}_{t,L}(a) }  \onec{ A_t \neq a^*_t}\one \cbr{ \overline{\mathcal{E}}_t }}}_{B_2}.
\end{align*}
The term $B_2$ can be bounded as follows: We maintain the analysis variable $L$ unchanged throughout the lemma, as it serves as the critical parameter in deriving both instance-dependent and minimax regret bounds from this lemma.
\begin{align*}
    B_2 &\leq \EE\sbr{\sum_{t=1}^{n}\overline{ \Delta}_{t} \one\cbr{\overline{\mathcal{D}}_{t,L}(a)} \one \cbr{ \overline{\mathcal{E}}_t } } \\
    &\leq \EE\sbr{\sum_{t=1}^{n}\overline{ \Delta}_{t} \one\cbr{\overline{\mathcal{D}}_{t,L}(a)}  } \\
    &= \EE\sbr{\sum_{t=1}^{n}\overline{ \Delta}_{t} \one \cbr{\overline{ \Delta}_{t} < B \cdot 2^{-L}} \onec{B\cdot 2^{-L} \leq \Delta}} \\
    &\spacex  + \EE\sbr{\sum_{t=1}^{n} \overline{ \Delta}_{t} \one \cbr{\overline{ \Delta}_{t} < B\cdot 2^{-L}} \onec{B\cdot 2^{-L} >\Delta}} \\
    &= 0 + n B 2^{-L} \cdot \onec{B\cdot2^{-L} >\Delta} \\
    &= \BII.
\end{align*}
This concludes the bounding of the term $B_2$. 

The first term $B_1$ can be further split into separate terms based on the condition $ \mathcal{U}_{t-1,\ell}$ as follows:

\begin{align*}
    B_1 &= \EE\sbr{\sum_{t=1}^{n}\sum_{\ell = 1}^{L}\overline{ \Delta}_{t} \one\cbr{ \mathcal{D}_{t,\ell}(A_t)} \onec{ A_t \neq a^*_t} \one \cbr{ \overline{\mathcal{E}}_t } }\\
    &=  \underbrace{\EE\sbr{\sum_{t=1}^{n}\sum_{\ell = 1}^{L}\overline{ \Delta}_{t} \one\cbr{ \mathcal{D}_{t,\ell}(A_t)} \onec{ A_t \neq a^*_t}  \one\cbr{\overline{\mathcal{U}}_{t-1,\ell}(A_t)}\one \cbr{ \overline{\mathcal{E}}_t } }}_{D_1} \\
    &\spacex + \underbrace{\EE\sbr{\sum_{t=1}^{n}\sum_{\ell = 1}^{L}\overline{ \Delta}_{t} \one\cbr{\mathcal{D}_{t,\ell}(A_t)} \onec{ A_t \neq a^*_t}  \one\cbr{\mathcal{U}_{t-1,\ell} (A_t)}\one \cbr{ \overline{\mathcal{E}}_t } }}_{D_2}.
\end{align*}

The term $D_2$ can be bounded using Elliptical Potential Count (Lemma \ref{app-lemma:epc-lemma}) as follows:
\begin{align*}
	D_2 &= \EE\sbr{\sum_{t=1}^{n}\sum_{\ell = 1}^{L}\overline{ \Delta}_{t} \one\cbr{\mathcal{D}_{t,\ell}(A_t)} \onec{ A_t \neq a^*_t}  \one\cbr{\mathcal{U}_{t-1,\ell} (A_t)} \one \cbr{ \overline{\mathcal{E}}_t } } \\
	&\leq \EE\sbr{\sum_{t=1}^{n}\sum_{\ell = 1}^{L}\overline{ \Delta}_{t} \one\cbr{\mathcal{D}_{t,\ell}(A_t)} \onec{ A_t \neq a^*_t}  \one\cbr{\mathcal{U}_{t-1,\ell} (A_t)}  } \\
	&= \EE\sbr{\sum_{t=1}^{n}\sum_{\ell = 1}^{L}\overline{ \Delta}_{t}  \one\cbr{\mathcal{D}_{t,\ell}(A_t)}  \onec{ A_t \neq a^*_t} \one\cbr{ \lVert A_t \rVert^2_{V_{t-1}^{-1}} \geq \eps_{\ell}}} \\
	&\leq\EE\sbr{\sum_{t=1}^{n}\sum_{\ell = 1}^{L}\overline{ \Delta}_{t}  \one\cbr{\mathcal{D}_{t,\ell}(A_t)} \one\cbr{ \lVert A_t \rVert^2_{V_{t-1}^{-1}} \geq \eps_{\ell}}} \\
	&= \EE\sbr{\sum_{t=1}^{n}\sum_{\ell = 1}^{L}\overline{ \Delta}_{t}  \one \cbr{ B \cdot 2^{-\ell} \leq\overline{ \Delta}_{t} \leq B \cdot 2^{-\ell + 1}} \one\cbr{ \lVert A_t \rVert^2_{V_{t-1}^{-1}} \geq \eps_{\ell}}} \\
	&\leq \EE\sbr{\sum_{t=1}^{n}\sum_{\ell = 1}^{L} B \cdot2^{-\ell + 1}\cdot \one\cbr{ \lVert A_t \rVert^2_{V_{t-1}^{-1}} \geq \eps_{\ell}}} \\
	&\leq \EE\sbr{\sum_{t=1}^{n}\sum_{\ell = 1}^{L} B \cdot2^{-\ell + 1}\cdot \one\cbr{ \lVert A_t \rVert^2_{V_{t-1}^{-1}} \geq 2^{-2\ell} \eps}} \\
	&\leq \sum_{\ell = 1}^{L} B \cdot2^{-\ell + 1}\cdot \EE\sbr{\sum_{t=1}^{n} \one\cbr{ \lVert A_t \rVert^2_{V_{t-1}^{-1}} \geq 2^{-2\ell} \eps}}.
\end{align*}
We can apply the EPC from Lemma \ref{app-lemma:epc-lemma} directly only when  $ 2^{-2\ell} \eps < 1$. Consequently, we must analyze it in two cases, as follows:

\begin{align*}
	D_2 &\leq \sum_{\ell = 1}^{L} B \cdot 2^{-\ell + 1} \cdot \EE\sbr{\sum_{t=1}^{n} \one\cbr{ \lVert A_t \rVert^2_{V_{t-1}^{-1}} \geq 2^{-2\ell} \eps}}\\
	&=  \sum_{\ell = 1}^{L} B\cdot 2^{-\ell + 1} \cdot\EE\sbr{\sum_{t=1}^{n} \one\cbr{ \lVert A_t \rVert^2_{V_{t-1}^{-1}} \geq 2^{-2\ell} \eps} \left(\one\cbr{2^{-2\ell}\eps \leq 1 } +\one\cbr{2^{-2\ell}\eps > 1 } \right) }\\
	&\leq  \sum_{\ell = 1}^{L} B\cdot 2^{-\ell + 1} \cdot \EE\sbr{\sum_{t=1}^{n} \one\cbr{ \lVert A_t \rVert^2_{V_{t-1}^{-1}} \geq 2^{-2\ell} \eps}\one\cbr{2^{-2\ell}\eps \leq 1 }}\\
	&\spacex + \sum_{\ell = 1}^{L} B\cdot 2^{-\ell + 1} \cdot \EE\sbr{\sum_{t=1}^{n} \one\cbr{ \lVert A_t \rVert^2_{V_{t-1}^{-1}} \geq 2^{-2\ell} \eps}\one\cbr{2^{-2\ell}\eps> 1 }}\\
	&\leq  \sum_{\ell = 1}^{L} B\cdot 2^{-\ell + 1}\cdot 3 \frac{d}{2^{-2\ell }\eps} \log \left( 1 + \frac{2}{\lambda 2^{-2\ell}\eps } \right)\tag{by lemma \ref{app-lemma:epc-lemma}}\\
	&\spacex +  \sum_{\ell = 1}^{L} B\cdot 2^{-\ell + 1}\EE\sbr{\sum_{t=1}^{n} \one\cbr{ \lVert A_t \rVert^2_{V_{t-1}^{-1}} \geq 1 }}\\
	&\leq \sum_{\ell = 1}^{L} B\cdot 2^{-\ell + 1} \cdot 3 \frac{d}{2^{-2\ell }\eps} \log \left( 1 + \frac{2}{\lambda 2^{-2\ell}\eps } \right) + \sum_{\ell = 1}^{L} B \cdot 2^{-\ell + 1} \cdot 3d \log \left( 1 + \frac{2}{\lambda} \right) \tag{by lemma \ref{app-lemma:epc-lemma}}\\
	&\leq \DII.
\end{align*}
This concludes the bounding of the term $D_2$.

Moving on to $D_1$, we can write $D_1$ into 3 terms based on the conditions $\mathcal{V}_{t-1}$ and $\mathcal{W}_{t-1,\ell}$ as follows:

\begin{align*}
    D_1 &= \EE\sbr{\sum_{t=1}^{n}\sum_{\ell = 1}^{L}\overline{ \Delta}_{t} \one\cbr{\mathcal{D}_{t,\ell}(A_t)} \onec{ A_t \neq a^*_t}  \one\cbr{\overline{\mathcal{U}}_{t-1,\ell}(A_t)} \one \cbr{ \overline{\mathcal{E}}_t } }\\
    &= \EE\sbr{\sum_{t=1}^{n}\sum_{\ell = 1}^{L}\overline{ \Delta}_{t} \one\cbr{\mathcal{D}_{t,\ell}(A_t)} \onec{ A_t \neq a^*_t}  \one\cbr{\overline{\mathcal{U}}_{t-1,\ell}(A_t)} \one\cbr{\mathcal{V}_{t-1}(A_t)}\one \cbr{ \overline{\mathcal{E}}_t }} \\
    &\spacex + \EE\sbr{\sum_{t=1}^{n}\sum_{\ell = 1}^{L}\overline{ \Delta}_{t} \one\cbr{\mathcal{D}_{t,\ell}(A_t)} \onec{ A_t \neq a^*_t}  \one\cbr{\overline{\mathcal{U}}_{t-1,\ell}(A_t)}\one\cbr{\overline{\mathcal{V}}_{t-1}(A_t)}  \one \cbr{ \overline{\mathcal{E}}_t }}\\
    &= \underbrace{\EE\sbr{\sum_{t=1}^{n}\sum_{\ell = 1}^{L}\overline{ \Delta}_{t} \one\cbr{\mathcal{D}_{t,\ell}(A_t)} \onec{ A_t \neq a^*_t}  \one\cbr{\overline{\mathcal{U}}_{t-1,\ell}(A_t)} \one\cbr{\mathcal{V}_{t-1}(A_t)}\one \cbr{ \overline{\mathcal{E}}_t }}}_{F_1} \\
    &\spacex + \underbrace{\EE\sbr{\sum_{t=1}^{n}\sum_{\ell = 1}^{L}\overline{ \Delta}_{t} \one\cbr{\mathcal{D}_{t,\ell}(A_t)} \onec{ A_t \neq a^*_t}  \one\cbr{\overline{\mathcal{U}}_{t-1,\ell}(A_t)} \one\cbr{\overline{\mathcal{V}}_{t-1}(A_t)} \one\cbr{\mathcal{W}_{t-1,\ell}}\one \cbr{ \overline{\mathcal{E}}_t }}}_{F_2}\\
    &\spacex + \underbrace{\EE\sbr{\sum_{t=1}^{n}\sum_{\ell = 1}^{L}\overline{ \Delta}_{t} \one\cbr{\mathcal{D}_{t,\ell}(A_t)} \onec{ A_t \neq a^*_t}  \one\cbr{\overline{\mathcal{U}}_{t-1,\ell}(A_t)} \one\cbr{\overline{\mathcal{V}}_{t-1}(A_t)}\one\cbr{\overline{\mathcal{W}}_{t-1,\ell}}\one \cbr{ \overline{\mathcal{E}}_t }}}_{F_3}.
\end{align*}

The term $F_1$ is conditioned on following events:
\begin{enumerate}
	\item  $\overline{\mathcal{U}}_{t-1,\ell}(a) = \cbr{ \lVert a \rVert^2_{V_{t-1}^{-1}} < \eps_{\ell}}$, which will be partly useful for upper bounding the denominator of $f_{t}(a)$. Intuitively, this condition indicates that arm $a$ has been sufficiently explored.
	\item $\mathcal{V}_{t-1}(a) = \cbr{\hat{\Delta}_{a,t} \geq \frac{\Delta_{a,t}}{1+c}}$,which will provide a lower bound for the numerator of $f_t(a)$. This condition suggests that the empirical gap is larger than the true gap, thereby ensuring that arm $a$ is appropriately distinguished as sub-optimal.
\end{enumerate}
Thus, we must be able to control the probability of selecting a sub-optimal arm as follows:

\begin{align*}
    F_1 &=\EE\sbr{\sum_{t=1}^{n}\sum_{\ell = 1}^{L}\overline{ \Delta}_{t} \one\cbr{\mathcal{D}_{t,\ell}(A_t)} \onec{ A_t \neq a^*_t}  \one\cbr{\overline{\mathcal{U}}_{t-1,\ell}(A_t)} \one\cbr{\mathcal{V}_{t-1}(A_t)}\one \cbr{ \overline{\mathcal{E}}_t }}\\
    &\leq \EE\sbr{\sum_{t=1}^{n}\sum_{\ell = 1}^{L}\overline{ \Delta}_{t} \one\cbr{\mathcal{D}_{t,\ell}(A_t)} \onec{ A_t \neq a^*_t}  \one\cbr{\overline{\mathcal{U}}_{t-1,\ell}(A_t)} \one\cbr{\mathcal{V}_{t-1}(A_t)}}\\
    &\leq \EE\sbr{\sum_{t=1}^{n}\sum_{\ell = 1}^{L} \sum_{a \in \mathcal{A}_{t}}\overline{ \Delta}_{a,t}  \one\cbr{\mathcal{D}_{t,\ell}(a)} \onec{ A_t = a }  \one\cbr{\overline{\mathcal{U}}_{t-1,\ell}(a)} \one\cbr{\mathcal{V}_{t-1}(a)}}\\
    &= \EE\sbr{\EE_{t-1}\sbr{\sum_{t=1}^{n}\sum_{\ell = 1}^{L} \sum_{a \in \mathcal{A}_{t}}\overline{ \Delta}_{a,t}  \one\cbr{\mathcal{D}_{t,\ell}(a)} \onec{ A_t = a }  \one\cbr{\overline{\mathcal{U}}_{t-1,\ell}(a)} \one\cbr{\mathcal{V}_{t-1}(a)}}}\tag{tower rule}\\
    &= \EE\sbr{\sum_{t=1}^{n}\sum_{\ell = 1}^{L} \sum_{a \in \mathcal{A}_{t}}\overline{ \Delta}_{a,t}  \one\cbr{\mathcal{D}_{t,\ell}(a)} \EE_{t-1}\sbr{ \onec{ A_t = a } }  \one\cbr{\overline{\mathcal{U}}_{t-1,\ell}(a)} \one\cbr{\mathcal{V}_{t-1}(a)}}\\
    &= \EE\sbr{\sum_{t=1}^{n}\sum_{\ell = 1}^{L} \sum_{a \in \mathcal{A}_{t}}\overline{ \Delta}_{a,t}  \one\cbr{\mathcal{D}_{t,\ell}(a)} p_{t}(a) \one\cbr{\overline{\mathcal{U}}_{t-1,\ell}(a)} \one\cbr{\mathcal{V}_{t-1}(a)}}\\
    &= \EE\sbr{\sum_{t=1}^{n}\sum_{\ell = 1}^{L} \sum_{a \in \mathcal{A}_{t}}\overline{ \Delta}_{a,t}  \one\cbr{\mathcal{D}_{t,\ell}(a)} q_{t}(a)\frac{f_{t}(a)}{\sum_{b \in \mathcal{A}_{t}} q_t(b) f_{t}(b)} \one\cbr{\overline{\mathcal{U}}_{t-1,\ell}(a)} \one\cbr{\mathcal{V}_{t-1}(a)}}\\
    &\leq \EE\sbr{\sum_{t=1}^{n}\sum_{\ell = 1}^{L} \sum_{a \in \mathcal{A}_{t}}\overline{ \Delta}_{a,t}  \one\cbr{\mathcal{D}_{t,\ell}(a)}  q_{t}(a)f_{t}(a)  \one\cbr{\overline{\mathcal{U}}_{t-1,\ell}(a)} \one\cbr{\mathcal{V}_{t-1}(a)}}\frac{1}{\alpha_{\text{emp}}} \tag{by lemma \ref{sketch-lemma:deno-lemma}}\\
    &= \underbrace{\EE\sbr{\sum_{t=1}^{n}\sum_{\ell = 1}^{L} \sum_{a \in \mathcal{A}_{t}}\overline{ \Delta}_{a,t}  \one\cbr{\mathcal{D}_{t,\ell}(a)}  q_{t}(a)f_{t}(a)  \one\cbr{\overline{\mathcal{U}}_{t-1,\ell}(a)}\one\cbr{\overline{\mathcal{U}}_{t-1,\ell}(\hat{a}_{t})} \one\cbr{\mathcal{V}_{t-1}(a)}}}_{F_{11}}\frac{1}{\alpha_{\text{emp}}} \\
    &\spacex +  \underbrace{\EE\sbr{\sum_{t=1}^{n}\sum_{\ell = 1}^{L} \sum_{a \in \mathcal{A}_{t}}\overline{ \Delta}_{a,t}  \one\cbr{\mathcal{D}_{t,\ell}(a)}  q_{t}(a)f_{t}(a)  \one\cbr{\overline{\mathcal{U}}_{t-1,\ell}(a)}\one\cbr{\mathcal{U}_{t-1,\ell}(\hat{a}_{t})} \one\cbr{\mathcal{V}_{t-1}(a)}}}_{F_{12}}\frac{1}{\alpha_{\text{emp}}}. 
\end{align*}

Since the denominator of $f_{t}(a)$ includes the Mahalanobis norm of the empirical best arm, it is essential to establish a bound on $\lVert \hat{a}_{t}\rVert_{V_{t-1}^{-1}}^2$. To achieve this, we further decompose the term $F_1$ into $F_{11}$ and $F_{12}$. The term $F_{11}$ comprises both the denominator and numerator as expected. Therefore, a modest application of algebra, alongside the triangle inequality, should yield the necessary bound. We will defer the tuning of $\eps$ until the conclusion to obtain a desirable bound.

\begin{align*}
    F_{11} &=\EE \vast[\sum_{t=1}^{n}\sum_{\ell = 1}^{L} \sum_{a \in \mathcal{A}_{t}, a \neq \hat{a}_{t}}\overline{ \Delta}_{a,t}  \one\cbr{\mathcal{D}_{t,\ell}(a)}  q_{t}(a)f_{t}(a)  \one\cbr{\overline{\mathcal{U}}_{t-1,\ell}(a)}\one\cbr{\overline{\mathcal{U}}_{t-1,\ell}(\hat{a}_{t})}\\
    &\spacex\spacex  \one\cbr{\mathcal{V}_{t-1}(a)}\vast]\\
    &\spacex + \EE \vast[\sum_{t=1}^{n}\sum_{\ell = 1}^{L} \overline{ \Delta}_{\hat{a}_{t} ,t} \one\cbr{\mathcal{D}_{t,\ell}(\hat{a}_{t})}  q_{t}(\hat{a}_{t})f_{t}(\hat{a}_{t})  \one\cbr{\overline{\mathcal{U}}_{t-1,\ell}(\hat{a}_{t})}\one\cbr{\overline{\mathcal{U}}_{t-1,\ell}(\hat{a}_{t})}\\
    &\spacex\spacex \one\cbr{\mathcal{V}_{t-1}(\hat{a}_{t})}\vast]\\
    &=\EE \vast[\sum_{t=1}^{n}\sum_{\ell = 1}^{L} \sum_{a \in \mathcal{A}_{t}, a \neq \hat{a}_{t}}\overline{ \Delta}_{a,t}  \one\cbr{\mathcal{D}_{t,\ell}(a)}  q_{t}(a)f_{t}(a)  \one\cbr{\overline{\mathcal{U}}_{t-1,\ell}(a)}\one\cbr{\overline{\mathcal{U}}_{t-1,\ell}(\hat{a}_{t})}\\
    &\spacex\spacex \one\cbr{\mathcal{V}_{t-1}(a)}\vast] \tag{$\one\cbr{\mathcal{V}_{t-1}(\hat{a}_{t})} =0$ }\\
    &\leq \EE\vast[ \sum_{t=1}^{n}\sum_{\ell = 1}^{L} \sum_{a \in \mathcal{A}_{t}}\overline{ \Delta}_{a,t}  \one\cbr{\mathcal{D}_{t,\ell}(a)}  q_{t}(a) \exp\left(- \frac{\hat{\Delta}^2_{a,t}}{2 \beta_{{t-1}}(\delta_{t-1})  \left(\lVert \hat{a}_{t}\rVert_{V_{t-1}^{-1}}^2 + \lVert a\rVert_{V_{t-1}^{-1}}^2 \right)}\right)  \one\cbr{\lVert a \rVert^2_{V_{t-1}^{-1}} < \eps_{\ell}} \tag{ from lemma \ref{sketch-lemma:denominator-expansion-lemma}}\\
    &\spacex\spacex \one\cbr{\lVert \hat{a}_{t} \rVert^2_{V_{t-1}^{-1}} < \eps_{\ell}} \one\cbr{\mathcal{V}_{t-1}(a)}\vast]\\
    &\leq \EE\sbr{\sum_{t=1}^{n}\sum_{\ell = 1}^{L} \sum_{a \in \mathcal{A}_{t}}\overline{ \Delta}_{a,t}  \one\cbr{\mathcal{D}_{t,\ell}(a)}  q_{t}(a) \exp\left(- \frac{\hat{\Delta}^2_{a,t}}{ 2\beta_{{t-1}}(\delta_{t-1}) \left( \eps_{\ell} + \eps_{\ell} \right)}\right)  \one\cbr{\mathcal{V}_{t-1}(a)}} \\
    &= \EE\sbr{\sum_{t=1}^{n}\sum_{\ell = 1}^{L} \sum_{a \in \mathcal{A}_{t}}\overline{ \Delta}_{a,t}  \one\cbr{\mathcal{D}_{t,\ell}(a)}  q_{t}(a) \exp\left(- \frac{\hat{\Delta}^2_{a,t}}{ 4\beta_{{t-1}}(\delta_{t-1})   \eps_{\ell}}\right)  \one\cbr{\hat{\Delta}_{a,t} \geq \frac{\Delta_{a,t}}{1+c}}}\\
    &\leq \EE\sbr{\sum_{t=1}^{n}\sum_{\ell = 1}^{L} \sum_{a \in \mathcal{A}_{t}}\overline{ \Delta}_{a,t}  \one\cbr{\mathcal{D}_{t,\ell}(a)}  q_{t}(a) \exp\left(- \frac{\left(\frac{\Delta_{a,t}}{1+c}\right)^2}{4 \beta_{{t-1}}(\delta_{t-1})   \eps_{\ell} }\right)  }.
\end{align*}
This is where the peeling technique proves beneficial. By applying the peeling technique to $\overline{ \Delta}_{a,t}$, we can effectively cancel out the denominator and numerator as follows:
\begin{align*}  
    F_{11} &\leq \EE\sbr{\sum_{t=1}^{n}\sum_{\ell = 1}^{L} \sum_{a \in \mathcal{A}_{t}}\overline{ \Delta}_{a,t}  \one\cbr{B\cdot 2^{-\ell} \leq\overline{ \Delta}_{a,t}  \leq B\cdot2^{-\ell + 1}}  q_{t}(a) \exp\left(- \frac{\left(\frac{\Delta_{a,t}}{1+c}\right)^2}{4 \beta_{{t-1}}(\delta_{t-1})   \eps_{\ell}}\right)  }\\
    &\leq \EE\sbr{\sum_{t=1}^{n}\sum_{\ell = 1}^{L} \sum_{a \in \mathcal{A}_{t}} B \cdot 2^{-\ell + 1} q_{t}(a) \exp\left(- \frac{\left(\frac{B \cdot 2^{-\ell}}{1+c}\right)^2}{ 4 \beta_{{t-1}}(\delta_{t-1})   \eps_{\ell}}\right)  } \\
    &= \EE\sbr{\sum_{t=1}^{n}\sum_{\ell = 1}^{L} \left( \sum_{a \in \mathcal{A}_{t}} q_{t}(a) \right)  B \cdot  2^{-\ell + 1} \exp\left(- \frac{\left(\frac{B \cdot 2^{-\ell}}{1+c}\right)^2}{ 4  \beta_{{t-1}}(\delta_{t-1})   \eps_{\ell}}\right)  }\\
    &= \EE\sbr{\sum_{t=1}^{n}\sum_{\ell = 1}^{L} B \cdot 2^{-\ell + 1} \exp\left(- \frac{\left(\frac{B \cdot 2^{-\ell}}{1+c}\right)^2}{4\beta_{{t-1}}(\delta_{t-1}) \eps_{\ell}}\right)  } \\
    &= \sum_{t=1}^{n}\sum_{\ell = 1}^{L} B \cdot 2^{-\ell + 1} \exp\left(- \frac{B^2}{ 4\eps \cdot \beta_{{t-1}}(\delta_{t-1})  (1+c)^2}\right) \\
    &= B \cdot  \sum_{t=1}^{n} \exp\left(- \frac{B^2}{ 4\eps \cdot \beta_{{t-1}}(\delta_{t-1}) (1+c)^2}\right) \sum_{\ell = 1}^{L} 2^{-\ell + 1} \\
    &= 2B \cdot \sum_{t=1}^{n} \exp\left(- \frac{B^2}{ 4\eps \cdot \beta_{{t-1}}(\delta_{t-1})  (1+c)^2}\right) \sum_{\ell = 1}^{L} 2^{-\ell} \\
    &\leq 2B \cdot \sum_{t=1}^{n} \exp\left(- \frac{B^2}{ 4\eps \cdot \beta_{{t-1}}(\delta_{t-1})  (1+c)^2}\right) \sum_{\ell = 1}^{\infty} 2^{-\ell} \\
    &= 4B \sum_{t=1}^{n} \exp\left(- \frac{B^2}{ 4\eps \cdot  \beta_{{t-1}}(\delta_{t-1})  (1+c)^2}\right). \tag{ geometric sum }
\end{align*}
Since $\beta_{{t-1}}(\delta_{t-1}) $ is an increasing function in $t$, we can upper bound $\beta_{{t-1}}(\delta_{t-1})$ with $\beta_{{n}}(\delta_n)$ which leads to,
\begin{align*}
    F_{11} &\leq 4B \sum_{t=1}^{n} \exp\left(- \frac{B^2}{ 4\eps \cdot \beta_{{n}}(\delta_n)  (1+c)^2}\right) \\
    &\leq 4B n\exp\left(- \frac{B^2}{ 4\eps \cdot \beta_{{n}}(\delta_n)  (1+c)^2}\right)  \\
    &\leq \FIFI. \tag{choose $c=1$} 
\end{align*}
This concludes the bounding of the term $F_{11}$.

The term $F_{12}$ presents a challenge. Unlike $F_{11}$, we cannot bound the probability directly because  $\lVert \hat{a}_{t}\rVert_{V_{t-1}^{-1}}^2$ is unbounded. However, we know that $\hat{a}_{t}$ is assigned a probability greater than the constant $\alpha_{\text{emp}}$  of being chosen at each round. Additionally, the elliptical potential count provides a bound on the number of times $\lVert A_t \rVert_{V_{t-1}^{-1}}^2$ can exceed $\eps_{\ell}$. Since the empirical best arm is chosen with significant probability, the elliptical potential count also indirectly limits the number of occurrences where$ \lVert \hat{a}_{t}\rVert_{V_{t-1}^{-1}}^2 \geq \eps_{\ell}$ (See Claim \ref{app-lemma:epc-q1-abstract-lemma}). Therefore,
\begin{align*}
    F_{12} &=\EE\sbr{\sum_{t=1}^{n}\sum_{\ell = 1}^{L} \sum_{a \in \mathcal{A}_{t}}\overline{ \Delta}_{a,t}  \one\cbr{\mathcal{D}_{t,\ell}(a)}  q_{t}(a)f_{t}(a)  \one\cbr{\overline{\mathcal{U}}_{t-1,\ell}(a)}\one\cbr{\mathcal{U}_{t-1,\ell}(\hat{a}_{t})} \one\cbr{\mathcal{V}_{t-1}(a)}} \\
    &\leq \EE\sbr{\sum_{t=1}^{n}\sum_{\ell = 1}^{L} \sum_{a \in \mathcal{A}_{t}}\overline{ \Delta}_{a,t}  \one\cbr{\mathcal{D}_{t,\ell}(a)}  q_{t}(a)f_{t}(a) \one\cbr{\mathcal{U}_{t-1,\ell}(\hat{a}_{t})} } \\
    &\leq \EE\sbr{\sum_{t=1}^{n}\sum_{\ell = 1}^{L} \sum_{a \in \mathcal{A}_{t}}\overline{ \Delta}_{a,t}  \one\cbr{\mathcal{D}_{t,\ell}(a)}  q_{t}(a) \one\cbr{\mathcal{U}_{t-1,\ell}(\hat{a}_{t})} } \tag{$f_{t}(a) \leq 1$}\\
    &= \EE\sbr{\sum_{t=1}^{n}\sum_{\ell = 1}^{L} \sum_{a \in \mathcal{A}_{t}}\overline{ \Delta}_{a,t}  \one\cbr{B \cdot 2^{-\ell} \leq\overline{ \Delta}_{a,t}\leq B \cdot 2^{-\ell + 1}}  q_{t}(a) \one\cbr{\mathcal{U}_{t-1,\ell}(\hat{a}_{t})} } \\
    &\leq B \EE\sbr{\sum_{t=1}^{n}\sum_{\ell = 1}^{L} \sum_{a \in \mathcal{A}_{t}} 2^{-\ell + 1} q_{t}(a) \one\cbr{\mathcal{U}_{t-1,\ell}(\hat{a}_{t})} } \\
    &= B\EE\sbr{\sum_{t=1}^{n}\sum_{\ell = 1}^{L} \left( \sum_{a \in \mathcal{A}_{t}}  q_{t}(a) \right) 2^{-\ell + 1} \one\cbr{\mathcal{U}_{t-1,\ell}(\hat{a}_{t})} }  \\
    &= B\EE\sbr{\sum_{t=1}^{n}\sum_{\ell = 1}^{L}  2^{-\ell + 1} \one\cbr{\mathcal{U}_{t-1,\ell}(\hat{a}_{t})} } \\
    &= B\EE\sbr{\sum_{t=1}^{n}\sum_{\ell = 1}^{L}  2^{-\ell + 1} \one\cbr{ \lVert \hat{a}_{t} \rVert^2_{V_{t-1}^{-1}} \geq \eps_{\ell}} } \\
    &= B \sum_{\ell = 1}^{L}  2^{-\ell + 1}  \EE\sbr{  \sum_{t=1}^{n} \one\cbr{ \lVert \hat{a}_{t} \rVert^2_{V_{t-1}^{-1}} \geq \eps_{\ell}} } \\
    &= \frac{B}{\alpha_{\text{emp}}} \sum_{\ell = 1}^{L}  2^{-\ell + 1}  \EE\sbr{  \sum_{t=1}^{n} \one\cbr{ \lVert A_t \rVert^2_{V_{t-1}^{-1}} \geq \eps_{\ell}} }. \tag{by Claim \ref{app-lemma:epc-q1-abstract-lemma}}
\end{align*}
From this point onward, the remaining results follow directly from $D_2$ after applying  Lemma~\ref{app-lemma:epc-lemma},
\begin{align*}
    F_{12} &\leq \FIFII.
\end{align*}
This concludes the bounding of the term $F_{12}$.

By combining the bounds obtained for $F_{11}$ and $F_{12}$, we can now establish a bound for $F_1$.
\begin{align*}
    F_1 &\leq \frac{1}{\alpha_{\text{emp}}}F_{11} + \frac{1}{\alpha_{\text{emp}}}F_{12}\\
    &\leq  \FI.
\end{align*}

This concludes the bounding of the term $F_{1}$.

We now proceed to analyze the term $F_2$. Specifically, we will further decompose $F_{2}$ into two components, $F_{21}$ and $F_{22}$, based on the occurrence of the favorable event $\mathcal{G}_1$ as follows:

\begin{align*}
    F2 &=\EE\sbr{\sum_{t=1}^{n}\sum_{\ell = 1}^{L}\overline{ \Delta}_{t} \one\cbr{\mathcal{D}_{t,\ell}(A_t)} \onec{ A_t \neq a^*_t}  \one\cbr{\overline{\mathcal{U}}_{t-1,\ell}(A_t)} \one\cbr{\overline{\mathcal{V}}_{t-1}(A_t)} \one\cbr{\mathcal{W}_{t-1,\ell}}\one \cbr{ \overline{\mathcal{E}}_t }}\\
    &\leq \EE\sbr{\sum_{t=1}^{n}\sum_{\ell = 1}^{L}\overline{ \Delta}_{t} \one\cbr{\mathcal{D}_{t,\ell}(A_t)} \onec{ A_t \neq a^*_t}  \one\cbr{\overline{\mathcal{U}}_{t-1,\ell}(A_t)} \one\cbr{\overline{\mathcal{V}}_{t-1}(A_t)} \one\cbr{\mathcal{W}_{t-1,\ell}}}\\
    &=\underbrace{\EE\sbr{\sum_{t=1}^{n}\sum_{\ell = 1}^{L}\overline{ \Delta}_{t} \one\cbr{\mathcal{D}_{t,\ell}(A_t)} \onec{ A_t \neq a^*_t}  \one\cbr{\mathcal{G}_1}\one\cbr{\overline{\mathcal{U}}_{t-1,\ell}(A_t)} \one\cbr{\overline{\mathcal{V}}_{t-1}(A_t)} \one\cbr{\mathcal{W}_{t-1,\ell}}}}_{F_{21}}\\
    &\spacex +  \underbrace{\EE\sbr{\sum_{t=1}^{n}\sum_{\ell = 1}^{L}\overline{ \Delta}_{t} \one\cbr{\mathcal{D}_{t,\ell}(A_t)} \onec{ A_t \neq a^*_t} \one\cbr{\overline{\mathcal{G}}_1} \one\cbr{\overline{\mathcal{U}}_{t-1,\ell}(A_t)} \one\cbr{\overline{\mathcal{V}}_{t-1}(A_t)} \one\cbr{\mathcal{W}_{t-1,\ell}}}}_{F_{22}}.
\end{align*}
In the term $F_{22}$, the favorable event $\mathcal{G}_1$ does not occur. Given the low probability of this occurrence, we can bound $F_{22}$ quite straightforwardly as follows:
\begin{align*}
    F_{22} &= \EE\sbr{\sum_{t=1}^{n}\sum_{\ell = 1}^{L}\overline{ \Delta}_{t} \one\cbr{\mathcal{D}_{t,\ell}(A_t)} \onec{ A_t \neq a^*_t} \one\cbr{\overline{\mathcal{G}}_1} \one\cbr{\overline{\mathcal{U}}_{t-1,\ell}(A_t)} \one\cbr{\overline{\mathcal{V}}_{t-1}(A_t)} \one\cbr{\mathcal{W}_{t-1,\ell}}}\\
    &\leq B\EE\sbr{\sum_{t=1}^{n} \one\cbr{\overline{\mathcal{G}}_1}  }\\
    &\leq B\sum_{t=1}^{n} \EE\sbr{\one\cbr{\overline{\mathcal{G}}_1}  }\\
    &\leq B\sum_{t=1}^{n} \PP\left(\overline{\mathcal{G}}_1  \right)\\
     &\leq B\sum_{t=1}^{n} \delta_t \tag{Lemma \ref{app-lemma:OFUL-conf-bound-lemma}}\\
     &=  B\sum_{t=1}^{n} \frac{1}{t+1} \tag{$\delta_t = \frac{1}{t+1}$} \\
    &\leq  \FIIFII.  
\end{align*}
This concludes the bounding of the term $F_{22}$.

However, the term $F_{21}$ is more complex. It encompasses the following primary events: 
It has the following main events:
\begin{enumerate}
	\item $\mathcal{W}_{t-1,\ell}  = \cbr{ \max_{a^{'} \in \mathcal{A}_{t}} \langle \hat{\theta}_{t-1}, a^{'} \rangle \geq \langle \theta^*, a^*_t \rangle - \eps_{2, \ell} }$. Intuitively, this event signifies that the estimated reward is sufficiently close to the maximum achievable reward. 
	\item $\mathcal{V}_{t-1}(a) = \cbr{\hat{\Delta}_{a,t} \geq \frac{\Delta_{a,t}}{1+c}}$.
	\item $\mathcal{G}_1$.
\end{enumerate} 
Through a series of algebraic manipulations and parameter tuning, we demonstrate that the occurrence of all three events is contingent upon the condition $\lVert A_t \rVert_{V_{t-1}^{-1}}^2 \geq  \frac{2^{-2\ell}}{ 16 \beta_{t-1}(\delta_{t-1})}$. This condition can be effectively managed using the elliptical potential count.

\begin{align*}
    F_{21} &= \EE\sbr{\sum_{t=1}^{n}\sum_{\ell = 1}^{L}\overline{ \Delta}_{t} \one\cbr{\mathcal{D}_{t,\ell}(A_t)} \onec{ A_t \neq a^*_t}  \one\cbr{\mathcal{G}_1}\one\cbr{\overline{\mathcal{U}}_{t-1,\ell}(A_t)} \one\cbr{\overline{\mathcal{V}}_{t-1}(A_t)} \one\cbr{\mathcal{W}_{t-1,\ell}}}\\
    &\leq \EE\sbr{\sum_{t=1}^{n}\sum_{\ell = 1}^{L}\overline{ \Delta}_{t} \one\cbr{\mathcal{D}_{t,\ell}(A_t)}  \one\cbr{\mathcal{G}_1} \one\cbr{\overline{\mathcal{V}}_{t-1}(A_t)} \one\cbr{\mathcal{W}_{t-1,\ell}}}\\
    &= \EE\sbr{\sum_{t=1}^{n}\sum_{\ell = 1}^{L}\overline{ \Delta}_{t} \one\cbr{\mathcal{D}_{t,\ell}(A_t)}   \one\cbr{\mathcal{G}_1}\one\cbr{\hat{\Delta}_{A_t,t} < \frac{\Delta_{A_t,t}}{1+c}} \one\cbr{\mathcal{W}_{t-1,\ell}} }.
\end{align*}
From this point onward, we adopt a proof style in which, if the occurrence of event $A$ implies the occurrence of event $B$ that is,
\begin{align*}
	A \implies B,
\end{align*}
then, we have
\begin{align*}
	\EE\sbr{\onec{A}}  \leq \EE\sbr{\onec{B}} ~.
\end{align*}
Similarly, if the occurrence of 2 events $A,B$ implies the occurrence of a third event $C$ that is,
\begin{align*}
	A , B \implies C,
\end{align*}
then, we have
\begin{align*}
	\EE\sbr{\onec{A}\onec{B}}  \leq \EE\sbr{\onec{C}}.
\end{align*} 
This approach to writing mathematical expressions facilitates a reduction in verbosity within the proof. Now, moving on to $F_{21}$,
\begin{align*}
   F_{21} &= \EE\bigg[\sum_{t=1}^{n}\sum_{\ell = 1}^{L}\overline{ \Delta}_{t} \one\cbr{\mathcal{D}_{t,\ell}(A_t)} \one\cbr{\mathcal{G}_1}\one\cbr{\max_{a^{'} \in \mathcal{A}_{t}} \langle \hat{\theta}_{t-1}, a^{'} \rangle - \langle \hat{\theta}_{t-1}, A_t \rangle < \frac{\Delta_{A_t,t}}{1+c} } \\
    &\spacex\spacex	\one\cbr{\max_{a^{'} \in \mathcal{A}_{t}} \langle \hat{\theta}_{t-1}, a^{'} \rangle \geq \langle \theta^*, a^*_t \rangle - \eps_{2, \ell} }\bigg]\\
    &\leq \EE\sbr{\sum_{t=1}^{n}\sum_{\ell = 1}^{L}\overline{ \Delta}_{t} \one\cbr{\mathcal{D}_{t,\ell}(A_t)} \one\cbr{\mathcal{G}_1}\one\cbr{\langle \theta^*, a^*_t \rangle - \eps_{2, \ell}  - \langle \hat{\theta}_{t-1}, A_t \rangle  < \frac{\Delta_{A_t,t}}{1+c} } }\\
    &= \EE\sbr{\sum_{t=1}^{n}\sum_{\ell = 1}^{L}\overline{ \Delta}_{t} \one\cbr{\mathcal{D}_{t,\ell}(A_t)} \one\cbr{\mathcal{G}_1}\one\cbr{\langle \theta^*, a^*_t \rangle - \langle \theta^*, A_t \rangle - \eps_{2, \ell} + \langle \theta^*, A_t \rangle  - \langle \hat{\theta}_{t-1}, A_t \rangle  < \frac{\Delta_{A_t,t}}{1+c} } }\\
    &= \EE\sbr{\sum_{t=1}^{n}\sum_{\ell = 1}^{L}\overline{ \Delta}_{t} \one\cbr{\mathcal{D}_{t,\ell}(A_t)} \one\cbr{\mathcal{G}_1}\one\cbr{\Delta_{A_t,t} - \eps_{2, \ell} + \langle \theta^*, A_t \rangle  - \langle \hat{\theta}_{t-1}, A_t \rangle  < \frac{\Delta_{A_t,t}}{1+c} } }\\
    &= \EE\sbr{\sum_{t=1}^{n}\sum_{\ell = 1}^{L}\overline{ \Delta}_{t} \one\cbr{\mathcal{D}_{t,\ell}(A_t)} \one\cbr{\mathcal{G}_1}\one\cbr{ \frac{c\Delta_{A_t,t}}{1+c} - \eps_{2, \ell}   <  \langle \hat{\theta}_{t-1} - \theta^* , A_t \rangle  } }\\
    &= \EE\sbr{\sum_{t=1}^{n}\sum_{\ell = 1}^{L}\overline{ \Delta}_{t} \one\cbr{\mathcal{D}_{t,\ell}(A_t)} \one\cbr{\mathcal{G}_1}\one\cbr{ \frac{\Delta_{A_t,t}}{2} - \eps_{2} \cdot 2^{-\ell}   <  \langle \hat{\theta}_{t-1} - \theta^* , A_t \rangle  } } \tag{choose $c=1$}\\
    &= \EE\sbr{\sum_{t=1}^{n}\sum_{\ell = 1}^{L}\overline{ \Delta}_{t} \one\cbr{B \cdot 2^{-\ell} \leq\overline{ \Delta}_{t} \leq B \cdot 2^{-\ell + 1}} \one\cbr{\mathcal{G}_1}\one\cbr{ \frac{\Delta_{A_t,t}}{2} - \eps_{2} \cdot 2^{-\ell}   <  \langle \hat{\theta}_{t-1} - \theta^* , A_t \rangle  } } \\
    &\leq \EE\sbr{\sum_{t=1}^{n}\sum_{\ell = 1}^{L} B \cdot 2^{-\ell + 1} \one\cbr{B \cdot 2^{-\ell} \leq\overline{ \Delta}_{t} \leq B \cdot 2^{-\ell + 1}} \one\cbr{\mathcal{G}_1}\one\cbr{ \frac{\Delta_{A_t,t}}{2} - \eps_{2} \cdot 2^{-\ell}   <  \langle \hat{\theta}_{t-1} - \theta^* , A_t \rangle  } } \\
    &\leq \EE\sbr{\sum_{t=1}^{n}\sum_{\ell = 1}^{L} B \cdot2^{-\ell + 1}\one\cbr{\mathcal{G}_1}\one\cbr{ \frac{B \cdot 2^{-\ell}}{2} - \eps_{2} \cdot 2^{-\ell}   <  \langle \hat{\theta}_{t-1} - \theta^* , A_t \rangle  } } \\
    &= \EE\sbr{\sum_{t=1}^{n}\sum_{\ell = 1}^{L} B \cdot2^{-\ell + 1}\one\cbr{\mathcal{G}_1}\one\cbr{ \frac{B \cdot 2^{-\ell}}{4}  <  \langle \hat{\theta}_{t-1} - \theta^* , A_t \rangle  } } \tag{choose $\eps_{2} = \frac{B}{4}$}\\
    &\leq \EE\sbr{\sum_{t=1}^{n}\sum_{\ell = 1}^{L} B \cdot2^{-\ell + 1}\one\cbr{\mathcal{G}_1}\one\cbr{ \frac{B^2 \cdot 2^{-2\ell}}{16}  < \lVert \hat{\theta}_{t-1} - \theta^* \rVert_{V_{t-1}}^2\lVert A_t \rVert_{V_{t-1}^{-1}}^2  } } \tag{Cauchy-Schwartz}\\
    &= \EE\sbr{\sum_{t=1}^{n}\sum_{\ell = 1}^{L} B \cdot2^{-\ell + 1}\one\cbr{\lVert \theta^* - \hat{\theta}_{t-1}\rVert_{V_{t-1}}^2 \leq \beta^{*}_{t-1}(\delta_{t-1})}\one\cbr{ \frac{B^2 \cdot 2^{-2\ell}}{16}  < \lVert \hat{\theta}_{t-1} - \theta^* \rVert_{V_{t-1}}^2\lVert A_t \rVert_{V_{t-1}^{-1}}^2  } } \\
    &\leq \EE\sbr{\sum_{t=1}^{n}\sum_{\ell = 1}^{L} B \cdot2^{-\ell + 1}\one\cbr{\lVert A_t \rVert_{V_{t-1}^{-1}}^2 \geq  \frac{B^2 \cdot2^{-2\ell}}{ 16  \beta^{*}_{t-1}(\delta_{t-1})}  } } \\
    &= \sum_{\ell = 1}^{L} B \cdot 2^{-\ell + 1} \EE\sbr{\sum_{t=1}^{n} \one\cbr{\lVert A_t \rVert_{V_{t-1}^{-1}}^2 \geq  \frac{B^2 \cdot 2^{-2\ell}}{ 16 \beta^{*}_{t-1}(\delta_{t-1})}  } }.
\end{align*}

We can apply the EPC from Lemma \ref{app-lemma:epc-lemma} directly only when  $\frac{B^2 \cdot2^{-2\ell}}{ 16 \beta^{*}_{t-1}(\delta_{t-1})} \leq 1$. Consequently, we must analyze it in two cases as follows:

\begin{align*}
	F_{21} &\leq \sum_{\ell = 1}^{L} B \cdot 2^{-\ell + 1} \EE\sbr{\sum_{t=1}^{n} \one\cbr{\lVert A_t \rVert_{V_{t-1}^{-1}}^2 \geq  \frac{B^2 \cdot2^{-2\ell}}{ 16 \beta^{*}_{t-1}(\delta_{t-1})}  } } \\
	&= \sum_{\ell = 1}^{L}B \cdot 2^{-\ell + 1} \EE\sbr{\sum_{t=1}^{n} \one\cbr{\lVert A_t \rVert_{V_{t-1}^{-1}}^2 \geq  \frac{B^2 \cdot 2^{-2\ell}}{ 16 \beta^{*}_{t-1}(\delta_{t-1})}  } \left( \one\cbr{\frac{B^2 \cdot 2^{-2\ell}}{ 16 \beta^{*}_{t-1}(\delta_{t-1})}  \leq 1 } + \one\cbr{\frac{B^2 \cdot 2^{-2\ell}}{ 16 \beta^{*}_{t-1}(\delta_{t-1})}  >  1 } \right)} \\
	&= \sum_{\ell = 1}^{L} B \cdot 2^{-\ell + 1} \EE\sbr{\sum_{t=1}^{n} \one\cbr{\lVert A_t \rVert_{V_{t-1}^{-1}}^2 \geq  \frac{B^2 \cdot 2^{-2\ell}}{ 16 \beta^{*}_{t-1}(\delta_{t-1})}  }  \one\cbr{\frac{B^2 \cdot 2^{-2\ell}}{ 16 \beta^{*}_{t-1}(\delta_{t-1})}  \leq 1 }} \\
	&\spacex + \sum_{\ell = 1}^{L} B \cdot 2^{-\ell + 1} \EE\sbr{\sum_{t=1}^{n} \one\cbr{\lVert A_t \rVert_{V_{t-1}^{-1}}^2 \geq  \frac{B^2 \cdot 2^{-2\ell}}{ 16 \beta^{*}_{t-1}(\delta_{t-1})}  }  \one\cbr{\frac{B^2 \cdot 2^{-2\ell}}{ 16 \beta^{*}_{t-1}(\delta_{t-1})}  > 1 }} \\
	&\leq  \sum_{\ell = 1}^{L} B \cdot 2^{-\ell + 1} \cdot 3 \frac{d}{\frac{B^2 \cdot 2^{-2\ell}}{ 16 \beta^{*}_{t-1}(\delta_{t-1})} } \log \left( 1 + \frac{2}{\lambda \frac{B^2 \cdot 2^{-2\ell}}{ 16 \beta^{*}_{t-1}(\delta_{t-1})} }\right) \tag{by lemma \ref{app-lemma:epc-lemma}}\\
	&\spacex + \sum_{\ell = 1}^{L} B \cdot 2^{-\ell + 1} \EE\sbr{\sum_{t=1}^{n} \one\cbr{\lVert A_t \rVert_{V_{t-1}^{-1}}^2 \geq 1  } } \\
	&\leq \frac{192  \beta^{*}_{{t-1}}(\delta_{t-1})d}{B \cdot 2^{-L}} \log\left( 1 + \frac{32\beta^{*}_{{t-1}}(\delta_{t-1})}{\lambda B^2 \cdot 2^{-2L}}\right) +  \sum_{\ell = 1}^{L} B \cdot 2^{-\ell + 1} \cdot 3d \log \left( 1 + \frac{2}{\lambda}\right) \tag{by lemma \ref{app-lemma:epc-lemma}}\\
	&\leq \FIIFI.
\end{align*}
This concludes the bounding of the term $F_{21}$.

By combining the bounds obtained for $F_{21}$ and $F_{22}$, we can now establish a bound for $F_2$.

\begin{align*}
    F_2 &= F_{21} + F_{22}\\
    &\leq  \FII.
\end{align*}

This concludes the bounding of the term $F_{2}$.

We now proceed to analyze the term $F_3$. Similar to $F_2$, we will further decompose $F_{3}$ into two components, $F_{31}$ and $F_{32}$, based on the occurrence of the favorable event $\mathcal{G}_1$ as follows:

\begin{align*}
    F_3 &= \EE\sbr{\sum_{t=1}^{n}\sum_{\ell = 1}^{L}\overline{ \Delta}_{t} \one\cbr{\mathcal{D}_{t,\ell}(A_t)} \onec{ A_t \neq a^*_t}  \one\cbr{\overline{\mathcal{U}}_{t-1,\ell}(A_t)} \one\cbr{\overline{\mathcal{V}}_{t-1}(A_t)}\one\cbr{\overline{\mathcal{W}}_{t-1,\ell}}\one \cbr{ \overline{\mathcal{E}}_t }}\\
    &=\underbrace{\EE\sbr{\sum_{t=1}^{n}\sum_{\ell = 1}^{L}\overline{ \Delta}_{t} \one\cbr{\mathcal{D}_{t,\ell}(A_t)} \onec{ A_t \neq a^*_t}\one\cbr{\mathcal{G}_1}  \one\cbr{\overline{\mathcal{U}}_{t-1,\ell}(A_t)} \one\cbr{\overline{\mathcal{V}}_{t-1}(A_t)}\one\cbr{\overline{\mathcal{W}}_{t-1,\ell}}\one \cbr{ \overline{\mathcal{E}}_t }}}_{F_{31}}\\
    &\spacex + \underbrace{\EE\sbr{\sum_{t=1}^{n}\sum_{\ell = 1}^{L}\overline{ \Delta}_{t} \one\cbr{\mathcal{D}_{t,\ell}(A_t)} \onec{ A_t \neq a^*_t}\one\cbr{\overline{\mathcal{G}}_1}  \one\cbr{\overline{\mathcal{U}}_{t-1,\ell}(A_t)} \one\cbr{\overline{\mathcal{V}}_{t-1}(A_t)}\one\cbr{\overline{\mathcal{W}}_{t-1,\ell}}\one \cbr{ \overline{\mathcal{E}}_t }}}_{F_{32}}.
\end{align*}

Analogous to $F_{22}$, we can bound $F_{32}$ with relative ease; therefore, we shall omit the details.

\begin{align*}
    F_{32} &\leq \FIIIFII. \tag{$\delta_t = \frac{1}{t+1}$}
\end{align*}

The analysis of the term  $F_{31}$ is the most complex and intricate among the terms. Initially, we exclude certain terms that are not pertinent to the specific analysis approach we will employ for $F_{31}$

\begin{align*}
    F_{31} &= \EE\sbr{\sum_{t=1}^{n}\sum_{\ell = 1}^{L}\overline{ \Delta}_{t} \one\cbr{\mathcal{D}_{t,\ell}(A_t)} \onec{ A_t \neq a^*_t}\one\cbr{\mathcal{G}_1}  \one\cbr{\overline{\mathcal{U}}_{t-1,\ell}(A_t)} \one\cbr{\overline{\mathcal{V}}_{t-1}(A_t)}\one\cbr{\overline{\mathcal{W}}_{t-1,\ell}}\one \cbr{ \overline{\mathcal{E}}_t } } \\
    &\leq \EE\sbr{\sum_{t=1}^{n}\sum_{\ell = 1}^{L}\overline{ \Delta}_{t} \one\cbr{\mathcal{D}_{t,\ell}(A_t)} \one\cbr{\mathcal{G}_1} \one\cbr{\overline{\mathcal{W}}_{t-1,\ell}}\one \cbr{ \overline{\mathcal{E}}_t }}.
\end{align*}

In the derived simplified version above, it is evident that below 3 events are occurring, each serving a significant purpose in the analysis.
\begin{enumerate}
	\item $ \mathcal{G}_1$
	\item $\overline{\mathcal{W}}_{t-1,\ell}$
	\item $\overline{\mathcal{E}}_t $
\end{enumerate}
Our primary objective is to bound the probability of the event $\overline{\mathcal{W}}_{t-1,\ell}$ occurring in conjunction with the other two events.
\begin{align*}
 F_{31}	&\leq \EE\sbr{\sum_{t=1}^{n}\sum_{\ell = 1}^{L}\overline{ \Delta}_{t} \one\cbr{2^{-\ell} \leq\overline{ \Delta}_{t} \leq 2^{-\ell + 1}} \one\cbr{\mathcal{G}_1} \one\cbr{\overline{\mathcal{W}}_{t-1,\ell}}\one \cbr{ \overline{\mathcal{E}}_t } }\\
    &\leq \EE\sbr{\sum_{t=1}^{n}\sum_{\ell = 1}^{L} B \cdot 2^{-\ell + 1} \one\cbr{\mathcal{G}_1} \one\cbr{\overline{\mathcal{W}}_{t-1,\ell}}\one \cbr{ \overline{\mathcal{E}}_t }}\\
    &= \EE\sbr{\sum_{t=1}^{n}\sum_{\ell = 1}^{L} B\cdot 2^{-\ell + 1} \one\cbr{\mathcal{G}_1} \one\cbr{ \max_{a^{'} \in \mathcal{A}_{t}} \langle \hat{\theta}_{t-1}, a^{'} \rangle < \langle \theta^*, a^*_t \rangle - \eps_{2, \ell}}\one \cbr{ \overline{\mathcal{E}}_t }}\\
    &\leq \EE\sbr{\sum_{t=1}^{n}\sum_{\ell = 1}^{L} B \cdot2^{-\ell + 1} \one\cbr{\mathcal{G}_1} \one\cbr{ \langle \hat{\theta}_{t-1}, a^*_t \rangle < \langle \theta^*, a^*_t \rangle - \eps_{2, \ell}} \one \cbr{ \overline{\mathcal{E}}_t } }\\
    &= \EE\sbr{\sum_{t=1}^{n}\sum_{\ell = 1}^{L} B \cdot2^{-\ell + 1} \one\cbr{\mathcal{G}_1} \one\cbr{ \eps_{2, \ell} \leq \langle \theta^*-\hat{\theta}_{t-1}, a^*_t \rangle  } \one \cbr{ \overline{\mathcal{E}}_t } }\\
    &\leq \EE\sbr{\sum_{t=1}^{n}\sum_{\ell = 1}^{L} B \cdot2^{-\ell + 1} \one\cbr{\mathcal{G}_1} \one\cbr{ \eps_{2, \ell} \leq \lVert \theta^*-\hat{\theta}_{t-1} \rVert_{V(p_t)} \lVert a^{*} \rVert_{V(p_t)^{-1}}  } \one \cbr{ \overline{\mathcal{E}}_t } } \tag{Cauchy-Schwartz}.
\end{align*} 
In the aforementioned derivation, it can be anticipated that the terms $\lVert \theta^*-\hat{\theta}_{t-1} \rVert_{V(p_t)}$ and $ \lVert a^{*} \rVert_{V(p_t)^{-1}} $ cannot assume significantly large values due to the following reasons:
\begin{enumerate}
	\item $\lVert \theta^*-\hat{\theta}_{t-1} \rVert_{V(p_t)}$ - The online learning paradigm necessitates that $\hat{\theta}_{t-1}$ remains sufficiently close to $\theta^*$.
	\item $ \lVert a^{*} \rVert_{V(p_t)^{-1}} $ - The $\mathrm{ApproxDesign}()$ ensures that exploration is adequately conducted in all relevant directions.
\end{enumerate}

To facilitate easy understanding of the proof, we first bound $ \lVert a^{*} \rVert_{V(p_t)^{-1}} $ by leveraging guarantees from the $\mathrm{ApproxDesign}()$, which is established in detail in Lemma \ref{sketch-lemma:leverage-score-lemma}.
\begin{align*}
    F_{31} &\leq \EE\sbr{\sum_{t=1}^{n}\sum_{\ell = 1}^{L} B \cdot 2^{-\ell + 1} \one\cbr{\mathcal{G}_1} \one\cbr{ \eps_{2, \ell} \leq \lVert \theta^*-\hat{\theta}_{t-1} \rVert_{V(p_t)} \cdot \sqrt{\frac{2}{\alpha_{\mathrm{opt}}} \exp \left(\frac{ \lVert \hat{\theta}_{t-1} - \theta^* \rVert_{V_{t-1}}^2}{\beta_{t-1}(\delta_{t-1}) }\right)\cdot C_{\mathrm{opt}}\cdot d\log(d)}}\one \cbr{ \overline{\mathcal{E}}_t } }. \tag{by Lemma \ref{sketch-lemma:leverage-score-lemma}}
\end{align*}
Next, we utilize the fact that $\mathcal{G}_1 = \vast \{ \lVert \theta^* - \hat{\theta}_{t-1}\rVert_{V_{t-1}}^2 \leq \beta^{*}_{t-1}(\delta_{t-1}), \spacex \forall t \geq 1\vast \}$ to simplify the analysis further as follows:
\begin{align*}
    F_{31}  &\leq \EE\sbr{\sum_{t=1}^{n}\sum_{\ell = 1}^{L} B \cdot 2^{-\ell + 1} \one\cbr{ \eps_{2, \ell} \leq \lVert \theta^*-\hat{\theta}_{t-1} \rVert_{V(p_t)} \cdot \sqrt{\frac{2}{\alpha_{\mathrm{opt}}} \exp \left(\frac{\beta^{*}_{t-1}(\delta_{t-1})}{\beta_{t-1}(\delta_{t-1})}\right)\cdot C_{\mathrm{opt}}\cdot d\log(d) }}  \one \cbr{ \overline{\mathcal{E}}_t }} \\
    &=  \sum_{\ell = 1}^{L} B \cdot 2^{-\ell + 1} \EE\sbr{\sum_{t=1}^{n}\one\cbr{ \eps_{2, \ell} \leq \lVert \theta^*-\hat{\theta}_{t-1} \rVert_{V(p_t)} \cdot \sqrt{\frac{2}{\alpha_{\mathrm{opt}}} \exp \left(\frac{\beta^{*}_{t-1}(\delta_{t-1})}{\beta_{t-1}(\delta_{t-1})}\right)\cdot C_{\mathrm{opt}}\cdot d\log(d)} }\one \cbr{ \overline{\mathcal{E}}_t } } \\
    &=  \sum_{\ell = 1}^{L} B \cdot 2^{-\ell + 1} \sum_{t=1}^{n}\EE\sbr{ \one\cbr{ \eps_{2, \ell} \leq \lVert \theta^*-\hat{\theta}_{t-1} \rVert_{V(p_t)} \cdot \sqrt{\frac{2}{\alpha_{\mathrm{opt}}} \exp \left(\frac{\beta^{*}_{t-1}(\delta_{t-1})}{\beta_{t-1}(\delta_{t-1})}\right) \cdot C_{\mathrm{opt}}\cdot d\log(d)} } \one \cbr{ \overline{\mathcal{E}}_t } } \\
    &=  \sum_{\ell = 1}^{L} B \cdot 2^{-\ell + 1} \sum_{t=1}^{n}\EE\sbr{ \one\cbr{\lVert \theta^*-\hat{\theta}_{t-1} \rVert_{V(p_t)}^2  \geq \frac{\eps_{2, \ell}^2 }{ \frac{2}{\alpha_{\mathrm{opt}}} \exp \left(\frac{\beta^{*}_{t-1}(\delta_{t-1})}{\beta_{t-1}(\delta_{t-1})}\right) \cdot C_{\mathrm{opt}}\cdot d\log(d)}  } \one \cbr{ \overline{\mathcal{E}}_t } }.
\end{align*}
By the definition of $V(p_t)$, we have $\lVert \theta^*-\hat{\theta}_{t-1} \rVert_{V(p_t)}^2 =\EE_{A_t \sim p_t} \sbr{ \left( \left( \theta^* - \hat{\theta}_{t-1}\right)A_t^{\T}\right)^2} $.
\begin{align*}
     F_{31}  &\leq  \sum_{\ell = 1}^{L} B \cdot 2^{-\ell + 1} \sum_{t=1}^{n}\EE\sbr{ \one\cbr{ \EE_{A_t \sim p_t} \sbr{ \left( \left( \theta^* - \hat{\theta}_{t-1}\right)A_t^{\T}\right)^2} \geq \frac{\eps_{2, \ell}^2 }{ \frac{2}{\alpha_{\mathrm{opt}}} \exp \left(\frac{\beta^{*}_{t-1}(\delta_{t-1})}{\beta_{t-1}(\delta_{t-1})}\right) \cdot C_{\mathrm{opt}}\cdot d\log(d)}  } \one \cbr{ \overline{\mathcal{E}}_t } } \tag{Claim \ref{app-claim:confident-interval-equi-expectation-lemma}}\\
    &=  \sum_{\ell = 1}^{L} B \cdot 2^{-\ell + 1} \sum_{t=1}^{n}\EE \vast[ \one\cbr{ \sum_{a \in \mathcal{A}_{t}} p_t(a)  \sbr{ \left( \left( \theta^* - \hat{\theta}_{t-1}\right)a^{\T}\right)^2} \geq \frac{\eps_{2, \ell}^2 }{ \frac{2}{\alpha_{\mathrm{opt}}} \exp \left(\frac{\beta^{*}_{t-1}(\delta_{t-1})}{\beta_{t-1}(\delta_{t-1})}\right) \cdot C_{\mathrm{opt}}\cdot d\log(d)}  } \\
    &\spacex \one \cbr{ \forall a \in \mathcal{A}_t , \lVert a\rVert_{V_{t-1}^{-1}}^2 \leq 1  } \vast].
\end{align*}
The event  $\overline{\mathcal{E}}_t $ implies that, for all the arms in $\mathcal{A}_t$ their leverage score is bounded above by $1$. Consequently, we can confidently incorporate the index function $\one\cbr{\lVert a\rVert_{V_{t-1}^{-1}}^2 \leq 1  }$ within the summation ($\sum_{a \in \mathcal{A}_{t}} $) without affecting the analysis.
\begin{align*}
    F_{31}  &\leq  \sum_{\ell = 1}^{L} B \cdot 2^{-\ell + 1} \sum_{t=1}^{n}\EE \sbr{\one\cbr{ \sum_{a \in \mathcal{A}_{t}} p_t(a)  \sbr{ \left( \left( \theta^* - \hat{\theta}_{t-1}\right)a^{\T}\right)^2} \one\cbr{\lVert a\rVert_{V_{t-1}^{-1}}^2 \leq 1  }\geq \frac{\eps_{2, \ell}^2 }{ \frac{2}{\alpha_{\mathrm{opt}}} \exp \left(\frac{\beta^{*}_{t-1}(\delta_{t-1})}{\beta_{t-1}(\delta_{t-1})}\right) \cdot C_{\mathrm{opt}}\cdot d\log(d)}  } }.
\end{align*}
Furthermore, it is evident that
\begin{align*}
	 \sum_{a \in \mathcal{A}_{t}} p_t(a)  \sbr{ \left( \left( \theta^* - \hat{\theta}_{t-1}\right)a^{\T}\right)^2} \one\cbr{\lVert a\rVert_{V_{t-1}^{-1}}^2 \leq 1  } = \EE_{A_t \sim p_t}  \sbr{ \left( \left( \theta^* - \hat{\theta}_{t-1}\right)A_t^{\T}\right)^2 \one\cbr{\lVert A_t \rVert_{V_{t-1}^{-1}}^2 \leq 1  } }.
\end{align*}
Hence,
\begin{align*}
    F_{31}  &\leq \sum_{\ell = 1}^{L} B \cdot 2^{-\ell + 1} \sum_{t=1}^{n}\EE \sbr{\one\cbr{ \EE_{A_t \sim p_t}  \sbr{ \left( \left( \theta^* - \hat{\theta}_{t-1}\right)A_t^{\T}\right)^2 \one\cbr{\lVert A_t \rVert_{V_{t-1}^{-1}}^2 \leq 1  } }\geq \frac{\eps_{2, \ell}^2 }{ \frac{2}{\alpha_{\mathrm{opt}}} \exp \left(\frac{\beta^{*}_{t-1}(\delta_{t-1})}{\beta_{t-1}(\delta_{t-1})}\right) \cdot C_{\mathrm{opt}}\cdot d\log(d)}  } } \\
    &=  \sum_{\ell = 1}^{L} B \cdot 2^{-\ell + 1} \sum_{t=1}^{n}\EE\sbr{  \one\cbr{ \EE_{t-1} \sbr{ \left( \left( \theta^* - \hat{\theta}_{t-1}\right)A_t^{\T}\right)^2 \one\cbr{\lVert A_t \rVert_{V_{t-1}^{-1}}^2 \leq 1  }  } \geq \frac{\eps_{2, \ell}^2 }{ \frac{2}{\alpha_{\mathrm{opt}}} \exp \left(\frac{\beta^{*}_{t-1}(\delta_{t-1})}{\beta_{t-1}(\delta_{t-1})}\right)  \cdot C_{\mathrm{opt}}\cdot d\log(d)}  } } \\
    &=  \sum_{\ell = 1}^{L} B \cdot  2^{-\ell + 1} \sum_{t=1}^{n}\PP \left( \EE_{t-1} \sbr{ \left( \left( \theta^* - \hat{\theta}_{t-1}\right)A_t^{\T}\right)^2 \one\cbr{\lVert A_t \rVert_{V_{t-1}^{-1}}^2 \leq 1  }  } \geq \frac{\eps_{2, \ell}^2 }{ \frac{2}{\alpha_{\mathrm{opt}}} \exp \left(\frac{\beta^{*}_{t-1}(\delta_{t-1})}{\beta_{t-1}(\delta_{t-1})}\right) \cdot C_{\mathrm{opt}}\cdot d\log(d)}  \right).
\end{align*}
Applying Markov's inequality to the expression above, we can derive the following bound:
\begin{align*}   
    F_{31}  &\leq  \sum_{\ell = 1}^{L} B \cdot  2^{-\ell + 1} \sum_{t=1}^{n} \frac{ \EE \sbr{\EE_{t-1} \sbr{\left( \left( \theta^* - \hat{\theta}_{t-1}\right)A_t^{\T}\right)^2 \one\cbr{\lVert A_t \rVert_{V_{t-1}^{-1}}^2 \leq 1  } }}}{ \frac{\eps_{2, \ell}^2 }{ \frac{2}{\alpha_{\mathrm{opt}}}\exp \left(\frac{\beta^{*}_{t-1}(\delta_{t-1})}{\beta_{t-1}(\delta_{t-1})}\right) \cdot C_{\mathrm{opt}}\cdot d\log(d)}  } \tag{Markov's Inequality}\\
    &\leq \sum_{\ell = 1}^{L} B \cdot  2^{-\ell + 1} \frac{ \frac{2}{\alpha_{\mathrm{opt}}} H_{\mathrm{max}}\cdot C_{\mathrm{opt}}\cdot d\log(d)}{\eps_{2, \ell}^2} \EE \sbr{  \sum_{t=1}^{n}  \left( \left( \theta^* - \hat{\theta}_{t-1}\right)A_t^{\T}\right)^2  \one\cbr{\lVert A_t \rVert_{V_{t-1}^{-1}} \leq 1 }}.
 \end{align*}
This section marks a pivotal moment in our analysis, as we leverage the Online Learning Equality established in Lemma \ref{sketch-lemma:regret-equality-modified} to derive Lemma \ref{sketch-lemma:online-learning-modified-lemma}. This derived lemma subsequently enables us to bound the expression $ \EE \sbr{  \sum_{t=1}^{n}  \left( \left( \theta^* - \hat{\theta}_{t-1}\right)A_t^{\T}\right)^2  \one\cbr{\lVert A_t \rVert_{V_{t-1}^{-1}} \leq 1 }}$ as follows:

 \begin{align*}  
    F_{31}  &\leq  \sum_{\ell = 1}^{L} B \cdot  2^{-\ell + 1} \frac{ \frac{2}{\alpha_{\mathrm{opt}}} H_{\mathrm{max}}\cdot C_{\mathrm{opt}}\cdot d\log(d)}{\eps_{2, \ell}^2} \left( 2\lambda \lVert \theta^*\rVert_2^2 + 4\sigma_*^2 d \log \left( 1 + \frac{n}{d\lambda}\right) \right)\tag{by Lemma \ref{sketch-lemma:online-learning-modified-lemma}}\\
    &\leq \sum_{\ell = 1}^{L}  B \cdot 2^{-\ell + 1} \frac{ \frac{2}{\alpha_{\mathrm{opt}}} H_{\mathrm{max}}\cdot C_{\mathrm{opt}}\cdot d\log(d)}{\eps_{2, \ell}^2} \left( 2 \lambda (S_*)^2 + 4\sigma_*^2 d \log \left( 1 + \frac{n}{d\lambda}\right) \right)\\
    &= \FIIIFI. \tag{$\eps_{2} = \frac{B}{4}$}.
\end{align*}

This concludes the bounding of the term $F_{31}$.

By combining the bounds obtained for $F_{31}$ and $F_{32}$, we can now establish a bound for $F_3$.
\begin{align*}
	F_3 &= F_{31} + F_{32}\\
	&\leq \FIII. 
\end{align*}

In summary, by consolidating all the bounds derived throughout our analysis, we arrive at the final regret bound articulated as follows:
\begin{align*}
	\Reg_n &\leq \AII \\
	&\spacex +  \FI \\
	&\spacex  + \FII \\
	&\spacex + \FIII \\
	&\spacex  +  \DII + \BII \\
	&\leq  6dB \left(4 + \frac{1}{\alpha_{\mathrm{emp}}^2}\right)\log \left(1 + \frac{2}{\lambda}\right) + 2B \log (n+1) +   \frac{12dB}{2^{-L}\eps}  \left(1 +  \frac{1}{\alpha_{\mathrm{emp}}^2}\right) \log \left( 1 + \frac{2}{\lambda 2^{-2L}\eps } \right)\\
	&\spacex + \frac{192  \beta^{*}_{{n}}(\delta)d}{B \cdot 2^{-L}} \log\left( 1 + \frac{32\beta^{*}_{{n}}(\delta)}{\lambda B^2 \cdot 2^{-2L}}\right) + \frac{512 H_{\mathrm{max}}\cdot C_{\mathrm{opt}}\cdot d\log(d)}{\alpha_{\mathrm{opt}} B \cdot 2^{-L}} \left( \frac{\lambda (S_*)^2}{2} + \sigma_*^2 d \log \left( 1 + \frac{n}{d\lambda}\right) \right)\\
	&\spacex + \BII + \frac{1}{\alpha_{\mathrm{emp}}}\cdot 4B n\exp\left(- \frac{B^2}{ 16\eps \cdot  \beta_{{n}}(\delta_n)  }\right).
\end{align*}
Furthermore, by tuning $\eps = \frac{B^2}{16  \beta_{n}(\delta)  \log n}$ we can derive the final bound as follows:
\begin{align*}
	\Reg_n &\leq   6dB \left(3 + \frac{1}{\alpha_{\mathrm{emp}}^2}\right)\log \left(1 + \frac{2}{\lambda}\right) + 2B \log (n+1) +   \frac{192\beta_{n}(\delta)  \log (n) d}{2^{-L}B}  \left(1 +  \frac{1}{\alpha_{\mathrm{emp}}^2}\right) \log \left( 1 + \frac{32 \beta_{n}(\delta)  \log (n) }{\lambda 2^{-2L} B^2 }\right)\\
	&\spacex + \frac{192  \beta^{*}_{{n}}(\delta)d}{B \cdot 2^{-L}} \log\left( 1 + \frac{32\beta^{*}_{{n}}(\delta)}{\lambda B^2 \cdot 2^{-2L}}\right) + \frac{512 H_{\mathrm{max}}\cdot C_{\mathrm{opt}}\cdot d\log(d)}{\alpha_{\mathrm{opt}} B \cdot 2^{-L}} \left( \frac{\lambda (S_*)^2}{2} + \sigma_*^2 d \log \left( 1 + \frac{n}{d\lambda}\right) \right)\\
	&\spacex + \BII + \frac{ 4B}{\alpha_{\mathrm{emp}}} .
\end{align*}

This bound encapsulates the cumulative effects of the various factors we have considered, providing a comprehensive measure of the regret associated with our algorithm. The implications of this bound are significant, as they delineate the performance guarantees of our approach under the specified conditions, ultimately contributing to a deeper understanding of the theoretical foundations underpinning our work.

Proof concludes.
\end{proof}
\clearpage
\subsection{Proof for augmenting the arm set by eliminating highly sub-optimal arms (version 1)}  \label{app-subsection:regret-lemma-elim-subsection}
We eliminate all the arms for which $f_{t}(a)  <  \frac{1}{e}$, where $f_{t}(a)$ is the quantity defined in Equation \eqref{eq:main-f_tm1}. The title of this section may be slightly misleading. In fact, Version 1 eliminates arms that exhibit one or both of the following characteristics: 
	\begin{enumerate} 
	\item A large estimated sub-optimality gap, $\hat\Delta_{a,t}^2$ and/or
	\item Arms that have already been sufficiently explored, as indicated by the bound  $ \lVert \hat{a}_{t} - a \rVert^2_{V_{t-1}^{-1}} \leq 2\left(\lVert \hat{a}_{t}\rVert^2_{V_{t-1}^{-1}} + \lVert a \rVert^2_{V_{t-1}^{-1}}  \right)$
\end{enumerate}
This approach is intuitively appealing, as it ensures that we avoid incurring regret by assigning probability to highly sub-optimal arms. Additionally, there is no need to allocate probability to directions that have already been sufficiently explored.

Before proceeding with the proof, it is important to note that this version of the augmenting arm set is ineffective when $\sigma_*^2$ and $S_*$ are under-specified. We require $\sigma_*^2 \leq \sigma^2$ and $S_* \leq S$.

The majority of the proof for Version 0 applies to this version as well; however, we derive a different bound for  $\lVert a^{*}_t \rVert_{V(p_t)^{-1}}^2$ compared to the one presented in Lemma \ref{sketch-lemma:leverage-score-lemma}
\begin{align*}
	\overline{\mathcal{A}}_{(t)} = \{a \in \mathcal{A}_t : f_{t}(a) \geq \frac{1}{e}\}.
\end{align*}

\textbf{Step 1:} We prove that with high probability $\forall t$, $a_t^{*} \in \overline{\mathcal{A}}_{(t)} $
\begin{align*}
	f_{t}(a_t^{*}) &= \exp \del[4]{- \frac{\hat\Delta_{a_t^*,t}^2}{\beta_{t-1}(\delta_{t-1})  \lVert \hat{a}_{t} - a_t^*\rVert^2_{V_{t-1}^{-1}} }}. \\
\end{align*}
Furthermore,
\begin{align*}
	\hat\Delta_{a_t^{*},t}  &= \langle \hat{a}_{t}, \hat{\theta}_{t-1} \rangle - \langle a^{*}_{t}, \hat{\theta}_{t-1} \rangle\\
	&= \langle \hat{a}_{t}, \hat{\theta}_{t-1} \rangle -  \langle \hat{a}_{t},\theta^* \rangle + \langle \hat{a}_{t},\theta^* \rangle  - \langle a^{*}_{t}, \hat{\theta}_{t-1} \rangle\\
	&\leq \langle \hat{a}_{t}, \hat{\theta}_{t-1} \rangle -  \langle \hat{a}_{t},\theta^* \rangle + \langle  a^{*}_{t},\theta^* \rangle  - \langle a^{*}_{t}, \hat{\theta}_{t-1} \rangle\\
	&= \langle \hat{a}_{t}, \hat{\theta}_{t-1} - \theta^* \rangle - \langle a^{*}_{t}, \hat{\theta}_{t-1} - \theta^* \rangle\\
	&\leq \lVert \hat{a}_{t} -a^{*}_{t} \rVert_{V_{t-1}^{-1}}  \cdot  \lVert \hat{\theta}_{t-1} - \theta^* \rVert_{V_{t-1}} \tag{Cauchy-Schwartz}  \\
	\hat\Delta_{a_t^{*},t}^2 &\leq  \lVert \hat{a}_{t} -a^{*}_{t} \rVert_{V_{t-1}^{-1}}^2  \cdot  \lVert \hat{\theta}_{t-1} - \theta^* \rVert_{V_{t-1}}^2.
\end{align*}
Combining the two displays above, we obtain the following result:

\begin{align*}
	f_{t}(a_t^{*}) &\geq  \exp \del[4]{- \frac{ \lVert \hat{a}_{t} -a^{*}_{t} \rVert_{V_{t-1}^{-1}}^2  \cdot  \lVert \hat{\theta}_{t-1} - \theta^* \rVert_{V_{t-1}}^2}{\beta_{t-1}(\delta_{t-1})  \lVert \hat{a}_{t} - a_t^*\rVert^2_{V_{t-1}^{-1}} }} \\
	&= \exp \del[4]{- \frac{ \lVert \hat{\theta}_{t-1} - \theta^* \rVert_{V_{t-1}}^2}{\beta_{t-1}(\delta_{t-1}) }}.
\end{align*}
Moreover, with a probability of at least $1-\delta_t$,
\begin{align*}
	\forall t, \spacex \lVert \hat{\theta}_{t-1} - \theta^* \rVert_{V_{t-1}}^2 \leq \beta_{t-1}(\delta_{t-1}).
\end{align*}
Hence we can conclude that, with a probability of at least $1-\delta_t$,
\begin{align*}
		f_{t}(a_t^{*}) &\geq \exp(-1) = \frac{1}{e} \implies a_t^{*} \in \overline{\mathcal{A}}_{(t)}.
\end{align*}
This also implies that we can apply the guarantees $\mathrm{ApproxDesign}()$ on $a_t^*$

\textbf{Step 2:}  Bounding $\lVert a^{*}_t \rVert_{V(p_t)^{-1}}^2$

This proof is analogous to the one presented in Lemma \ref{sketch-lemma:leverage-score-lemma}.
\begin{align*}
	V(p_t) &= \sum_{a \in \overline{\mathcal{A}}_{(t)}} p_t(a) aa^{\T}\\
	&\succeq \frac{1}{2}\sum_{a \in \overline{\mathcal{A}}_{(t)}} p^{'}_t(a) aa^{\T}\\
	&= \frac{1}{2} \sum_{a \in \overline{\mathcal{A}}_{(t)}}\frac{ q_t(a) f_{t}(a) }{\sum_{b \in \mathcal{A}_t} q_t(b)f_{t}(b)}aa^{\T}\\
	&\succeq  \frac{1}{2} \sum_{a \in \overline{\mathcal{A}}_{(t)}}  q_t(a) f_{t}(a) aa^{\T} \tag{by Lemma \ref{sketch-lemma:deno-lemma}}\\
	&\succeq \frac{1}{2} \sum_{a \in \overline{\mathcal{A}}_{(t)}}  \alpha_{\mathrm{opt}} \cdot q_t^{\mathrm{opt}}(a) f_{t}(a) aa^{\T}\\
	&\succeq \frac{1}{2e} \sum_{a \in \overline{\mathcal{A}}_{(t)}}  \alpha_{\mathrm{opt}} \cdot q_t^{\mathrm{opt}}(a) aa^{\T} \tag{$\forall a \in  \overline{\mathcal{A}}_{(t)}$, $f_{t}(a) \geq \frac{1}{e}$}\\
	&= \frac{\alpha_{\mathrm{opt}}}{2e} \sum_{a \in \overline{\mathcal{A}}_{(t)}} q_t^{\mathrm{opt}}(a)  aa^T\\
	&= \frac{\alpha_{\mathrm{opt}}}{2e} V(q_t^{\mathrm{opt}}).
\end{align*}

In conclusion, we obtain the following bound:
\begin{align*}
	\lVert a^{*}_t \rVert_{V(p_t)^{-1}}^2 \leq \frac{2e}{\alpha_{\mathrm{opt}}} \cdot  C_{\mathrm{opt}}\cdot d\log(d).
\end{align*}
With the exception of the aforementioned two steps, the remainder of the proof closely follows that of Version 0; therefore, we omit the details.
\clearpage

\subsection{Proof of Theorem \ref{main-thm:inst-dep-reg-bound-theorem}}\label{app-subsection:inst-dep-reg-bound-proof-subsection}
\newtheorem*{app-thm:inst-dep-reg-bound-theorem}{Theorem \ref{main-thm:inst-dep-reg-bound-theorem}}
\begin{app-thm:inst-dep-reg-bound-theorem}[Instance-dependent bound]\label{app-thm:inst-dep-reg-bound-theorem}
	Under Assumptions \ref{main-assump:env-assumption}, \ref{main-assump:opt-lev-scr-assumption}, and \ref{main-assump:opt-cardinality-assumption}, with $\delta_t = \frac{1}{t+1}$, LinMED satisfies, $\forall n \geq 1$,
\begin{align*}
	\EE\Reg_n &= O\bigg( \frac{1}{\Delta} d \log (n) \bigg(\left(\sigma^2 d\log(n) + \lambda S^2 \right)\log \left(\log n  \right) + \left(\sigma_*^2 d\log(n) + \lambda S_*^2 \right) H_{\mathrm{max} } \bigg) \bigg).
\end{align*}
\end{app-thm:inst-dep-reg-bound-theorem}
\begin{proof}
	
	From Lemma \ref{app-lemma:regretlemma} we have
\begin{align*}
	\Reg_n &\leq  \REGLEMMAFOUR.
\end{align*}
	Where $2^{-L}$ is an analysis variable we introduced, such that 
	\begin{align*}
		\mathcal{D}_{t,L}(a) &=  \cbr{ \Delta_{a,t} \leq B \cdot 2^{-L}}.
	\end{align*} 
	By choosing $L$ such that $ \frac{\Delta}{2}\leq B \cdot 2^{-L} \leq \Delta  $, we can show that,
	\begin{align*}
		\onec{ B \cdot 2^{-L} >\Delta} &= \onec{\Delta >\Delta} = 0.
	\end{align*}
	Hence,
	\begin{align*}
		\Reg_n &\leq 6dB \left(3 + \frac{1}{\alpha_{\mathrm{emp}}^2}\right)\log \left(1 + \frac{2}{\lambda}\right) + 2B \log (n+1) +   \frac{384\beta_{n}(\delta)  \log (n) d}{\Delta}  \left(1 +  \frac{1}{\alpha_{\mathrm{emp}}^2}\right) \log \left( 1 + \frac{128 \beta_{n}(\delta)  \log (n) }{\lambda \Delta^2 }\right)\\
		&\spacex + \frac{384  \beta^{*}_{{n}}(\delta)d}{\Delta} \log\left( 1 + \frac{128\beta^{*}_{{n}}(\delta)}{\lambda \Delta^2 }\right) + \frac{1024 H_{\mathrm{max}}\cdot C_{\mathrm{opt}}\cdot d \log(d)}{\alpha_{\mathrm{opt}} \Delta} \left( \frac{\lambda (S_*)^2}{2} + \sigma_*^2 d \log \left( 1 + \frac{n}{d\lambda}\right) \right) + \frac{ 4B}{\alpha_{\mathrm{emp}}} \\
		&= O\left( \frac{1}{\Delta} d \log (n) \left(\left(\sigma^2 d\log(n) + \lambda S^2 \right)\log \left(\log n  \right) +   \left(\sigma_*^2 d\log(n) + \lambda S_*^2 \right) H_{\mathrm{max} } \right) \right).
	\end{align*}
	Proof concludes.
\end{proof}
\clearpage
\subsection{Proof of Theorem \ref{main-thm:mini-max-reg-bound-theorem}} \label{app-subsection:mini-max-reg-bound-proof-subsection}
\newtheorem*{app-thm:mini-max-reg-bound-theorem}{Theorem \ref{main-thm:mini-max-reg-bound-theorem}}

\begin{app-thm:mini-max-reg-bound-theorem}[Minimax bound]\label{app-thm:mini-max-reg-bound-theorem}
	Under Assumptions \ref{main-assump:env-assumption}, \ref{main-assump:opt-lev-scr-assumption}, and \ref{main-assump:opt-cardinality-assumption}, with $\delta_t = \frac{1}{t+1}$, LinMED satisfies, $\forall n \geq 1$, 
\begin{align*}
	\EE\Reg_n &= O\bigg( \sqrt{n} \bigg( \log^{\frac{1}{2}}(n) \left(d \sigma \log(n) + \frac{\lambda S}{\sigma} \right) + \frac{ H_{\mathrm{max}}  }{  \sigma \log^{\frac{3}{2}}(n) } \left( d \sigma_*^2 \log(n) + \lambda S_*^2 \right) \bigg)   \bigg).
\end{align*}
\end{app-thm:mini-max-reg-bound-theorem}
\begin{proof}
	From Lemma \ref{app-lemma:regretlemma} we have
\begin{align*}
	\Reg_n &\leq  \REGLEMMAFOUR.
\end{align*}
	Where $2^{-L}$ is an analysis variable we introduced, such that 
	\begin{align*}
		\mathcal{D}_{t,L}(a) &=  \cbr{ \Delta_{a,t} \leq B \cdot 2^{-L}}.
	\end{align*} 
	Case 1  : $n \leq 4 \sigma^2 \left(\frac{d}{B}\right)^2 \log^3(n)$
	
	This is a trivial case. We can show that
	\begin{align*}
		\Reg_n &\leq n\\
		&= \sqrt{n} \cdot \sqrt{n}\\
		& \leq 2 \sigma \log^\frac{3}{2}(n)\frac{d}{B} \sqrt{n}.
	\end{align*}
	Case 2  : $n > 4\sigma^2 \left(\frac{d}{B}\right)^2\log^3(n) $
	
	We can set $ 2^{-(L+1)} \leq \frac{\sigma d\log^\frac{3}{2}(n)}{B\sqrt{n}} \leq 2^{-L} <  \frac{1}{2}$  $\spacex$ ($L \geq 1$ will be assured).
	
	Also,
	\begin{align*}
		\onec{B \cdot 2^{-L} >\Delta} \leq 1.
	\end{align*}
	Hence the regret bound is
	\begin{align*}
		\Reg_n &\leq 6dB \left( 3 + \frac{1}{\alpha_{\mathrm{emp}}^2}\right)\log \left(1 + \frac{2}{\lambda}\right) + 2B \log (n+1) +   \frac{192\beta_{n}(\delta) \sqrt{n}}{\sigma  \log^{\frac{1}{2}}(n)}  \left(1 +  \frac{1}{\alpha_{\mathrm{emp}}^2}\right) \log \left( 1 + \frac{32 \beta_{n}(\delta) n  }{\sigma^2 \lambda d^2 \log^2 (n)}\right)\\
		&\spacex + \frac{192  \beta^{*}_{{n}}(\delta)\sqrt{n}}{ \sigma \log^{\frac{3}{2} } (n)} \log\left( 1 + \frac{32\beta^{*}_{{n}}(\delta) n}{\lambda \sigma^2 d^2 \log^{3}(n) }\right) + \frac{512 \sqrt{n} H_{\mathrm{max}}\cdot C_{\mathrm{opt} } \log(d) }{ \sigma \log^{\frac{3}{2}}(n)\alpha_{\mathrm{opt}} } \left( \frac{\lambda (S_*)^2}{2} + \sigma_*^2 d \log \left( 1 + \frac{n}{d\lambda}\right) \right)\\
		&\spacex + 2\sigma  d \sqrt{n} \log^{\frac{3}{2}} (n)+ \frac{ 4B}{\alpha_{\mathrm{emp}}} \\
		&= O\left( \sqrt{n} \left( \log^{\frac{1}{2}}(n) \left(d \sigma \log(n) + \frac{\lambda S^2}{\sigma} \right) + \frac{ H_{\mathrm{max}}  }{  \sigma \log^{\frac{3}{2}}(n) } \left( d \sigma_*^2 \log(n) + \lambda S_*^2 \right) \right)   \right).
	\end{align*}
	Proof concludes.
\end{proof}
\clearpage
\subsection{Proof of Corollary \ref{main-cor:inst-dep-reg-bound-corollary}} \label{app-subsection:inst-dep-reg-bound-proof-cor-subsection}
\newtheorem*{app-cor:inst-dep-reg-bound-theorem}{Corollary \ref{main-cor:inst-dep-reg-bound-corollary}}
\begin{app-cor:inst-dep-reg-bound-theorem}[Instance-dependent bound]\label{cor-thm:instance-dep-reg-bound-theorem}
	Under Assumptions \ref{main-assump:env-assumption}, \ref{main-assump:opt-lev-scr-assumption}, and \ref{main-assump:opt-cardinality-assumption}, assuming $\sigma^2 \geq \sigma_*^2$, $S \geq S_*$ with $\lambda = \frac{\sigma^2}{S^2}$ and $\delta_t = \frac{1}{t+1}$,  LinMED satisfies, $\forall n \geq 1$,
\begin{align*}
	&\EE\Reg_n = O\left( \sigma^2 \frac{d^2}{\Delta}  \log^2 (n) \log \left(\log n\right)\right) .
\end{align*}
\end{app-cor:inst-dep-reg-bound-theorem}
\begin{proof}
	The proof of this corollary follows directly from Theorem \ref{main-thm:inst-dep-reg-bound-theorem} and is straightforward. Given that $\sigma^2 \geq \sigma_*^2$ and $S \geq S_*$, it follows that $H_{\mathrm{max}} \leq \exp(1)$, resulting in the dominance of the first term over the second. Further substitution of $\lambda = \frac{\sigma^2}{S^2}$ yields the final result.
\end{proof}

\subsection{Proof of Corollary \ref{main-cor:mini-max-reg-bound-corollary}} \label{app-subsection:mini-max-reg-bound-proof-cor2-subsection}
\newtheorem*{app-cor:mini-max-reg-bound-theorem}{Corollary \ref{main-cor:mini-max-reg-bound-corollary}}
\begin{app-cor:mini-max-reg-bound-theorem}[Minimax bound]\label{cor-thm:mini-max-reg-bound-theorem}
	Under Assumptions \ref{main-assump:env-assumption}, \ref{main-assump:opt-lev-scr-assumption}, and \ref{main-assump:opt-cardinality-assumption}, assuming $\sigma^2 \geq \sigma_*^2$, $S \geq S_*$ and with $\lambda = \frac{\sigma^2}{S^2}$ and $\delta_t = \frac{1}{t+1}$, LinMED satisfies, $\forall n \geq 1$,
	\begin{align*}
		&\EE\Reg_n = O\left( \sig d\sqrt{n}\log^{\frac{3}{2}}(n)  \right).
	\end{align*}
\end{app-cor:mini-max-reg-bound-theorem}
\begin{proof}
	The proof of this corollary follows directly from Theorem \ref{main-thm:mini-max-reg-bound-theorem} and is straightforward. Given that $\sigma^2 \geq \sigma_*^2$ and $S \geq S_*$, it follows that $H_{\mathrm{max}} \leq \exp(1)$, resulting in the dominance of the first term over the second. Further substitution of $\lambda = \frac{\sigma^2}{S^2}$ yields the final result.
\end{proof}

\subsection{Proof of Corollary \ref{main-cor:mini-max-reg-bound-corollary-under-est}} \label{app-subsection:mini-max-reg-bound-proof-cor2-underspec-subsection}
\newtheorem*{app-cor:mini-max-underspec-reg-bound-theorem}{Corollary \ref{main-cor:mini-max-reg-bound-corollary-under-est}}
\begin{app-cor:mini-max-underspec-reg-bound-theorem}[Minimax bound]\label{cor-thm:mini-max-reg-bound-underspec-theorem}
		Under Assumptions \ref{main-assump:env-assumption}, \ref{main-assump:opt-lev-scr-assumption}, and \ref{main-assump:opt-cardinality-assumption}, assuming $\sigma^2 < \sigma_*^2$, $S \geq S_*$ and with $\lambda = \frac{\sigma^2}{S^2}$ and $\delta_t = \frac{1}{t+1}$, $\forall n \geq 1$, LinMED satisfies
	\begin{align*}
		\EE\Reg_n =
		O \del[4]{\frac{\sigma d  \sqrt{n} }{ \log^{\frac{1}{2}}(n) }\del[3]{ \log^{2}(n) + \frac{\sigma_{*}^2}{\sigma^2} \exp\del[2]{\frac{\sigma_*^2}{\sigma^2}} } }.
	\end{align*}
\end{app-cor:mini-max-underspec-reg-bound-theorem}
\begin{proof}
	We can bound $\beta_{t}^{*}(\delta_t)$ as follows : 
	\begin{align*}
		\beta_{t}^{*}(\delta_t) &= \left(\sigma_* \sqrt{\log \left(\frac{\det V_{t}}{\det V_0} \right) + 2\log\frac{1}{\delta_t}} + \sqrt{\lambda}S_*\right)^2 \\
		&\leq \left(\sigma_* \sqrt{\log \left(\frac{\det V_{t}}{\det V_0} \right) + 2\log\frac{1}{\delta_t}} + \sqrt{\lambda}S\right)^2 \\
		&\leq \left(\sigma_* \sqrt{\log \left(\frac{\det V_{t}}{\det V_0} \right) + 2\log\frac{1}{\delta_t}} + \sigma \right)^2 \\
		&= \sigma_*^2 \left( \sqrt{\log \left(\frac{\det V_{t}}{\det V_0} \right) + 2\log\frac{1}{\delta_t}} + \frac{\sigma }{ \sigma_* }\right)^2 \\
		&\leq \sigma_*^2 \left( \sqrt{\log \left(\frac{\det V_{t}}{\det V_0} \right) + 2\log\frac{1}{\delta_t}} +1 \right)^2 .
	\end{align*}
	Similarly,
	\begin{align*}
		\beta_{t}(\delta_t) &= \sigma^2 \left( \sqrt{\log \left(\frac{\det V_{t}}{\det V_0} \right) + 2\log\frac{1}{\delta_t}} + 1 \right)^2 .
	\end{align*}
	Hence, we can conclude that,
	\begin{align*}
	 H_{\mathrm{max}} \leq \exp\left(\frac{\sigma_{*}^2}{\sigma^2}\right).
	\end{align*}
	By bounding $H_{\mathrm{max}}$ with the aforementioned quantity and $S_*$ with $S$, and subsequently substituting $\lambda = \frac{\sigma^2}{S^2}$ into Theorem \ref{main-thm:mini-max-reg-bound-theorem}, the final results are obtained.
\end{proof} 
\clearpage

\section{LEMMATA} \label{app-section:necessary-lemmas-section}
\begin{lemma}\label{sketch-lemma:denominator-expansion-lemma}
	Let $f_t(a)$ be the quantity defined in Equation \eqref{eq:main-f_tm1} where $\hat{a}_{t} \neq a $. Then $\forall t > 1$,
	\begin{align*}
		f_{t}(a)   \leq \exp \del[4]{- \frac{\hat\Delta_{a,t}^2}{ 2  \beta_{t-1}(\delta_{t-1})  \left(\lVert \hat{a}_{t}\rVert^2_{V_{t-1}^{-1}} + \lVert  a \rVert^2_{V_{t-1}^{-1}} \right)}}.
	\end{align*}
\end{lemma}
\begin{proof}
	\begin{align*}
		f_{t}(a) &= \exp \del[4]{- \frac{\hat\Delta_{a,t}^2}{ \beta_{t-1}(\delta_{t-1})  \lVert \hat{a}_{t} - a \rVert^2_{V_{t-1}^{-1}} }}\\
		&\leq \exp \del[4]{- \frac{\hat\Delta_{a,t}^2}{   \beta_{t-1}(\delta_{t-1})  \left(\lVert \hat{a}_{t}\rVert_{V_{t-1}^{-1}} + \lVert  a \rVert_{V_{t-1}^{-1}} \right)^2}} \tag{triangle inequality}\\
		&\leq \exp \del[4]{- \frac{\hat\Delta_{a,t}^2}{ 2  \beta_{t-1}(\delta_{t-1})  \left(\lVert \hat{a}_{t}\rVert^2_{V_{t-1}^{-1}} + \lVert  a \rVert^2_{V_{t-1}^{-1}} \right)}}. \tag{AM-GM}
	\end{align*}
\end{proof}

\begin{lemma}\label{sketch-lemma:deno-lemma}
	Let $f_t(a)$ be the quantity defined in Equation \eqref{eq:main-f_tm1} where $\hat{a}_{t} \neq a $ and $q_t(a)$ be the quantity defined in Equation \eqref{main-eq:q_t}.
	Then $\forall t > 1$,
	\begin{align*}
		1 \geq	\sum_{b \in \mathcal{A}_t} q_t(b) f_{t}(b) &\geq \alpha_{\text{emp}}.
	\end{align*}
\end{lemma}
\begin{proof}
	\begin{align*}
		\sum_{b \in \mathcal{A}_t} q_t(b) f_{t}(b) &\geq q_t(\hat{a}_{t}) f_{t}(\hat{a}_{t}).
	\end{align*}
	
	Furthermore, $f_{t}(\hat{a}_{t}) = \exp(0) = 1 $ because $\hat{\Delta}_{\hat{a}_{t},t} = 0$, which leads to 
	\begin{align*}
		\sum_{b \in \mathcal{A}_t } q_t(b) f_{t}(b) &\geq  q_t(\hat{a}_{t})\\
		&\geq \alpha_{\text{emp}}\onec{\hat{a}_{t} = \hat{a}_{t} }\\
		&= \alpha_{\text{emp}}. 
	\end{align*}
	Hence,
	\begin{align*}
		\sum_{b \in \mathcal{A}_t} q_t(b) f_{t}(b) &\geq \alpha_{\text{emp}}.
	\end{align*}
	Also, note that
	\begin{align*}
		\sum_{b \in \mathcal{A}_t} q_t(b) f_{t}(b) \leq \sum_{b \in \mathcal{A}_t} q_t(b) 
		\leq 1 \tag{Because, $f_{t}(b) \leq 1$, $\forall b$}.
	\end{align*}
\end{proof}
\begin{lemma}\label{sketch-lemma:emp-best-arm-prob-lemma}
In the context of the LinMED algorithm, $\forall t > 1$, the probability of choosing the empirical best arm by LinMED algorithm satisfies	$ \PP\left(A_t = \hat{a}_{t}\right) \geq \alpha_{\mathrm{emp}}$.
\end{lemma}
\begin{proof}
	\begin{align*}
		\PP\left(A_t = \hat{a}_{t}\right) &= p_t(\hat{a}_{t})\\
		&= q_t(\hat{a}_{t}) \cdot  \frac{f_{t}(\hat{a}_{t})}{\sum_{b \in \mathcal{A}} q_t(b) f_{t}(b)} \\
		&\geq q_t(\hat{a}_{t}) \cdot f_{t}(\hat{a}_{t}) \tag{by lemma \ref{sketch-lemma:deno-lemma}}\\
		&= q_t(\hat{a}_{t}) \\
		&\geq \alpha_{\mathrm{emp}} \cdot \onec{\hat{a}_{t} = \hat{a}_{t} }\\
		&= \alpha_{\mathrm{emp}}.
	\end{align*}
\end{proof}
\begin{lemma}\label{sketch-lemma:leverage-score-lemma}
	In the context of the LinMED algorithm, we have
	\begin{align*}
		\lVert a^{*}_t \rVert_{V(p_t)^{-1}}^2 \leq \frac{2}{\alpha_{\mathrm{opt}}} \exp \left(\frac{ \lVert \hat{\theta}_{t-1} - \theta^* \rVert_{V_{t-1}}^2}{ \beta_{t-1}(\delta_{t-1}) }\right)\cdot C_{\mathrm{opt}}\cdot d\log(d).
	\end{align*}
	where $a^{*}_t$ is true best arm at time $t$.
\end{lemma}
\begin{proof}
	\begin{align*}
		V(p_t) &= \sum_{a \in \mathcal{A}_t} p_t(a) aa^{\T}\\
		 &\succeq \frac{1}{2}\sum_{a \in \mathcal{A}_t} p^{'}_t(a) aa^{\T}\\
		&= \frac{1}{2} \sum_{a \in \mathcal{A}_t}\frac{ q_t(a) f_{t}(a) }{\sum_{b \in \mathcal{A}_t} q_t(b)f_{t}(b)}aa^{\T}\\
		&\succeq  \frac{1}{2} \sum_{a \in \mathcal{A}_t}  q_t(a) f_{t}(a) aa^{\T} \tag{by Lemma \ref{sketch-lemma:deno-lemma}}\\
		&\succeq \frac{1}{2} \sum_{a \in \mathcal{A}_t}  \alpha_{\mathrm{opt}} \cdot q_t^{\mathrm{opt}}(a) f_{t}(a) aa^{\T}\\
		&= \frac{\alpha_{\mathrm{opt}}}{2} \sum_{a \in \mathcal{A}_t} q_t^{\mathrm{opt}}(a)  \left( \sqrt{ f_{t}(a)}a \right)\left( \sqrt{ f_{t}(a)}a \right)^{\T}\\
		&=  \frac{\alpha_{\mathrm{opt}}}{2} \sum_{a \in \mathcal{A}_t} q_t^{\mathrm{opt}}(a) (\bar{a}_{(t)})(\bar{a}_{(t)})^{\T}\\
		&= \frac{\alpha_{\mathrm{opt}}}{2} \overline{V}(q_t^{\mathrm{opt}}).
	\end{align*}
	Here, note that both $V(p_t)$ and $\overline{V}(q_t^{\mathrm{opt}})$ are invertible.
	\begin{align*}
		\lVert a^{*}_t \rVert_{V(p_t)^{-1}}^2 &\leq \frac{2}{\alpha_{\mathrm{opt}}} \lVert a^{*}_t \rVert_{\overline{V}(q_t^{\mathrm{opt}})^{-1}}^2\\
		&\leq \frac{2}{\alpha_{\mathrm{opt}}}\frac{1}{f_{t}(a_t^{*})}  \lVert \overline{a^{*}}_t \rVert_{\overline{V}(q_t^{\mathrm{opt}})^{-1}}^2\\
		&\leq  \frac{2}{\alpha_{\mathrm{opt}}} \frac{1}{f_{t}(a_t^{*})}C_{\mathrm{opt}} \cdot d \log(d).  \tag{Assumption \ref{main-assump:opt-lev-scr-assumption}}
	\end{align*}
	\begin{align*}
		\frac{1}{f_{t}(a_t^{*})} &= \exp \del[4]{\frac{\hat\Delta_{a_t^{*},t}^2}{ \beta_{t-1}(\delta_{t-1})  \lVert \hat{a}_{t} - a_t^{*} \rVert^2_{V_{t-1}^{-1}} }}.
	\end{align*}
	\begin{align*}
		\hat\Delta_{a_t^{*},t}  &= \langle \hat{a}_{t}, \hat{\theta}_{t-1} \rangle - \langle a^{*}_{t}, \hat{\theta}_{t-1} \rangle\\
		&= \langle \hat{a}_{t}, \hat{\theta}_{t-1} \rangle -  \langle \hat{a}_{t},\theta^* \rangle + \langle \hat{a}_{t},\theta^* \rangle  - \langle a^{*}_{t}, \hat{\theta}_{t-1} \rangle\\
		&\leq \langle \hat{a}_{t}, \hat{\theta}_{t-1} \rangle -  \langle \hat{a}_{t},\theta^* \rangle + \langle  a^{*}_{t},\theta^* \rangle  - \langle a^{*}_{t}, \hat{\theta}_{t-1} \rangle\\
		&= \langle \hat{a}_{t}, \hat{\theta}_{t-1} - \theta^* \rangle - \langle a^{*}_{t}, \hat{\theta}_{t-1} - \theta^* \rangle\\
		&\leq \lVert \hat{a}_{t} -a^{*}_{t} \rVert_{V_{t-1}^{-1}}  \cdot  \lVert \hat{\theta}_{t-1} - \theta^* \rVert_{V_{t-1}} \tag{Cauchy-Schwartz}  \\
		\hat\Delta_{a_t^{*},t}^2 &\leq  \lVert \hat{a}_{t} -a^{*}_{t} \rVert_{V_{t-1}^{-1}}^2  \cdot  \lVert \hat{\theta}_{t-1} - \theta^* \rVert_{V_{t-1}}^2.
	\end{align*}
	\begin{align*}
		\lVert a^{*}_t \rVert_{V(p_t)^{-1}}^2 \leq \frac{2}{\alpha_{\mathrm{opt}}} \exp \left(\frac{\lVert \hat{a}_{t} -a^{*}_{t} \rVert_{V_{t-1}^{-1}}^2  \cdot  \lVert \hat{\theta}_{t-1} - \theta^* \rVert_{V_{t-1}}^2}{ \beta_{t-1}(\delta_{t-1})  \lVert \hat{a}_{t} - a_t^{*} \rVert^2_{V_{t-1}^{-1}}}\right)\cdot C_{\mathrm{opt}}\cdot d\log(d).
	\end{align*}
\end{proof}

\begin{claim} \label{app-lemma:epc-q1-abstract-lemma}
In the context of the LinMED algorithm, we have
 \begin{align*}
 	 \mathcal{U}_{t-1,\ell}(a) = \cbr{ \lVert a \rVert^2_{V_{t-1}^{-1}} \geq \eps_{\ell}}.
 \end{align*}
 Then,
	\begin{align*}
		&\EE \sbr{ \sum_{t=1}^{n} \one \cbr{\overline{\mathcal{U}}_{t,\ell}(\hat{a}_{t})}   } \leq  \frac{1}{\alpha_{\mathrm{emp}}} \EE \sbr{ \sum_{t=1}^{n} \one \cbr{\overline{\mathcal{U}}_{t-1,\ell}( A_t)}   }.
	\end{align*}
\end{claim}
\begin{proof}
	\begin{align*}
		&\EE \sbr{ \sum_{t=1}^{n} \one \cbr{\overline{\mathcal{U}}_{t-1,\ell}(\hat{a}_{t})}   } =  \frac{1}{\alpha_{\mathrm{emp}}}\EE \sbr{ \sum_{t=1}^{n} \alpha_{\mathrm{emp}} \cdot  \one \cbr{\overline{\mathcal{U}}_{t-1,\ell}(\hat{a}_{t})}   } \\
		&\leq  \frac{1}{\alpha_{\mathrm{emp}}} \EE \sbr{ \sum_{t=1}^{n}\PP\left( A_t = \hat{a}_{t}\right)  \cdot  \one \cbr{\overline{\mathcal{U}}_{t-1,\ell}(\hat{a}_{t})}   }  \tag{by lemma \ref{sketch-lemma:emp-best-arm-prob-lemma}}\\
		&=  \frac{1}{\alpha_{\mathrm{emp}}} \EE \sbr{ \sum_{t=1}^{n} \EE_{t-1} \sbr{ \one \cbr{ A_t = \hat{a}_{t}}} \cdot  \one \cbr{\overline{\mathcal{U}}_{t-1,\ell}(\hat{a}_{t})}   }  \\
		&=  \frac{1}{\alpha_{\mathrm{emp}}} \EE \sbr{ \sum_{t=1}^{n}  \one \cbr{ A_t = \hat{a}_{t}} \cdot  \one \cbr{\overline{\mathcal{U}}_{t-1,\ell}(\hat{a}_{t})}   } \tag{tower rule}\\
		&=  \frac{1}{\alpha_{\mathrm{emp}}} \EE \sbr{ \sum_{t=1}^{n} \one \cbr{\overline{\mathcal{U}}_{t-1,\ell}( A_t)}   }.
	\end{align*}
\end{proof}
\begin{lemma}[from Lemma 7.1 A Modern Introduction to Online Learning]\label{sketch-lemma:regret-equality}
	Let $\Theta \subseteq \mathbb{R}^d$ be closed and non-empty. Denote by $F_t(\theta) = \psi_t(\theta) + \sum_{s=1}^{t-1} \ell_{s}(\theta)$. Assume that $\argmin_{\theta \in \Theta} F_t(\theta)$ is not empty and set $\hat{\theta}_{t-1} \in \argmin_{\theta \in \Theta} F_t(\theta) $. Then for any $\theta^* \in \mathbb{R}^d$, we have
	\begin{align*}
		\sum_{t=1}^{n} \left( \ell_t(\hat{\theta}_{t-1}) - \ell_t(\theta^*) \right) &= \psi_{n+1}(\theta^*) - \min_{\theta \in \Theta } \psi_{1}(\theta) + \sum_{t=1}^{n} \sbr{ F_t(\hat{\theta}_{t-1}) -  F_{t+1}(\hat{\theta}_{t}) + \ell_t(\hat{\theta}_{t-1})} \\
		\spacex &+ F_{n+1}(\hat{\theta}_n) - F_{n+1}(\theta^*).
	\end{align*}
\end{lemma}

\begin{lemma}\label{sketch-lemma:regret-equality-modified}
	Let $\Theta \subseteq \mathbb{R}^d$ be closed and non-empty, $\ell_t(\theta) = \frac{1}{2} \left( A_t^{\T}\theta - y_t \right)^2$  and $F_t(\theta) = \lambda \lVert \theta \lVert_2^2 + \sum_{s=1}^{t-1} \ell_{s}(\theta)$.  Assume that $\argmin_{\theta \in \Theta} F_t(\theta)$ is not empty and set $\hat{\theta}_{t-1} \in \argmin_{\theta \in \Theta} F_t(\theta) $.  Then for any $\theta^* \in \mathbb{R}^d$, we have
	\begin{align*}
		\sum_{t=1}^{n} \left( \ell_t(\hat{\theta}_{t-1}) - \ell_t(\theta^*) \right) &= \frac{\lambda}{2} \lVert \theta^* \rVert_2^2 + \sum_{t=1}^{n} \ell_t(\hat{\theta}_{t-1}) \lVert A_t \rVert^2_{V_{t}^{-1}} - \frac{1}{2} \lVert \hat{\theta}_n - \theta^* \rVert^2_{V_n}.
	\end{align*}
\end{lemma}
\begin{proof}
	By Lemma \ref{sketch-lemma:regret-equality}, we have
	\begin{align*}
		\sum_{t=1}^{n} \left( \ell_t(\hat{\theta}_{t-1}) - \ell_t(\theta^*) \right) &= \psi_{n+1}(\theta^*) - \min_{\theta \in \Theta } \psi_{1}(\theta) + \sum_{t=1}^{n} \sbr{ F_t(\hat{\theta}_{t-1}) -  F_{t+1}(\hat{\theta}_{t}) + \ell_t(\hat{\theta}_{t-1})} \\
		\spacex &+ F_{n+1}(\hat{\theta}_n) - F_{n+1}(\theta^*).
	\end{align*}
	\begin{align*}
		\nabla \ell_t (\theta) &=  \left( A_t^{\T}\theta - y_t \right)\cdot A_t = \sqrt{2\ell_t (\theta) } \cdot A_t\\
		\nabla^2 \ell_t (\theta)&=   A_t  A_t^{\T}.
	\end{align*}
	\begin{align*}
		\nabla F_t(\theta) &=  \lambda \theta +  \sum_{s=1}^{t-1} \left( A_s^{\T}\theta - y_s \right)\cdot A_s.\\
		\nabla^2 F_t(\theta) &=  \lambda  +    \sum_{s=1}^{t-1} A_s  A_s^{\T} = V_{t-1}.
	\end{align*}
	Using the Taylor's theorem for a quadratic polynomial,
	\begin{align*}
		F_{n+1}(\theta^*) &= F_{n+1}(\hat{\theta}_n) + \left(\nabla F_{n+1}(\hat{\theta}_n)\right)^{\T}\left(\theta^* -\hat{\theta}_{n}\right)   +  \frac{1}{2}\left(\theta^* -\hat{\theta}_{n}\right)^{T}  \nabla^2 F_{n+1}(\hat{\theta}_n) \left(\theta^* -\hat{\theta}_{n}\right) \\
		&= F_{n+1}(\hat{\theta}_n) + 0  + \frac{1}{2}\left(\theta^* -\hat{\theta}_{n}\right)^{T} V_n \left(\theta^* -\hat{\theta}_{n}\right) \tag{second term is $0$ by the optimality condition}\\
		F_{n+1}(\hat{\theta}_n) - F_{n+1}(\theta^*) &= -\frac{1}{2}\lVert \hat{\theta}_n - \theta^* \rVert^2_{V_n}.
	\end{align*}
	Then,
	\begin{align*}
		\sum_{t=1}^{n} \sbr{ F_t(\hat{\theta}_{t-1}) -  F_{t+1}(\hat{\theta}_{t}) + \ell_t(\hat{\theta}_{t-1})} &= \sum_{t=1}^{n} \sbr{ F_{t+1}(\hat{\theta}_{t-1}) -  F_{t+1}(\hat{\theta}_{t}) } \\
		&= \sum_{t=1}^{n} \frac{1}{2} \lVert \hat{\theta}_{t} - \hat{\theta}_{t-1} \rVert^2_{V_{t}}.
	\end{align*}
	Similarly,
	\begin{align*}
		\sum_{t=1}^{n} \sbr{ F_t(\hat{\theta}_{t-1}) -  F_{t+1}(\hat{\theta}_{t}) + \ell_t(\hat{\theta}_{t-1})} &= \sum_{t=1}^{n} \sbr{ F_{t}(\hat{\theta}_{t-1}) -  F_{t}(\hat{\theta}_{t}) + \ell_t(\hat{\theta}_{t-1}) - \ell_t(\hat{\theta}_{t}) } \\
		&= \sum_{t=1}^{n} \sbr{ -\left(F_{t}(\hat{\theta}_{t})-   F_{t}(\hat{\theta}_{t-1})  \right)   + \ell_t(\hat{\theta}_{t-1}) - \ell_t(\hat{\theta}_{t}) } \\
		&= \sum_{t=1}^{n} - \frac{1}{2} \lVert \hat{\theta}_{t} - \hat{\theta}_{t-1} \rVert^2_{V_{t-1} } - \left(  \ell_t(\hat{\theta}_{t})  - \ell_t(\hat{\theta}_{t-1}) \right)  \\
		&= \sum_{t=1}^{n} - \frac{1}{2} \lVert \hat{\theta}_{t} - \hat{\theta}_{t-1} \rVert^2_{V_{t-1} } - \Bigg( \left( \hat{\theta}_{t} - \hat{\theta}_{t-1}\right)^{\T}\nabla \ell_t(\hat{\theta}_{t-1}) \\
		&\spacex + \frac{1}{2}\left( \hat{\theta}_{t} - \hat{\theta}_{t-1}\right)^{\T} \nabla^2 \ell_t(\hat{\theta}_{t-1})  \left( \hat{\theta}_{t} - \hat{\theta}_{t-1}\right) \Bigg)  \\
		&= \sum_{t=1}^{n} - \frac{1}{2} \lVert \hat{\theta}_{t} - \hat{\theta}_{t-1} \rVert^2_{V_{t-1} } -  \Bigg( \left( \hat{\theta}_{t} - \hat{\theta}_{t-1}\right)^{\T}\nabla\ell_t(\hat{\theta}_{t-1})\\
		&\spacex  + \frac{1}{2}\left( \hat{\theta}_{t} - \hat{\theta}_{t-1}\right)^{\T}A_tA_t^{\T}  \left( \hat{\theta}_{t} - \hat{\theta}_{t-1}\right)  \Bigg) \\
		&= \sum_{t=1}^{n} - \frac{1}{2} \lVert \hat{\theta}_{t} - \hat{\theta}_{t-1} \rVert^2_{V_{t} } -  \left( \hat{\theta}_{t} - \hat{\theta}_{t-1}\right)^{\T}\nabla \ell_t(\hat{\theta}_{t-1}).
	\end{align*}
	Hence,
	\begin{align*}
		\lVert \hat{\theta}_{t} - \hat{\theta}_{t-1} \rVert^2_{V_{t} } &=  \left( \hat{\theta}_{t-1} - \hat{\theta}_{t}\right)^{\T}\nabla\ell_t(\hat{\theta}_{t-1})\\
	\iff	\left( \hat{\theta}_{t-1} - \hat{\theta}_{t}\right)^{\T} V_t \left( \hat{\theta}_{t-1} - \hat{\theta}_{t}\right) &= \left( \hat{\theta}_{t-1} - \hat{\theta}_{t}\right)^{\T}\nabla\ell_t(\hat{\theta}_{t-1})\\
	\implies	\hat{\theta}_{t-1} - \hat{\theta}_{t} &= V_t^{-1} \nabla \ell_t(\hat{\theta}_{t-1})\\
		&=  \sqrt{2 \ell_t(\hat{\theta}_{t-1})} V_t^{-1}A_t\\
		\frac{1}{2} \lVert \hat{\theta}_{t} - \hat{\theta}_{t-1} \rVert^2_{V_{t}} &= \ell_t(\hat{\theta}_{t-1})  \lVert V_t^{-1}A_t \rVert^2_{V_{t}}\\
		&=  \ell_t(\hat{\theta}_{t-1})  \lVert A_t \rVert^2_{V_{t}^{-1}}.
	\end{align*}
	
	Finally, we have
	\begin{align*}
		\sum_{t=1}^{n} \sbr{ F_t(\hat{\theta}_{t-1}) -  F_{t+1}(\hat{\theta}_{t}) + \ell_t(\hat{\theta}_{t-1})} &=  \ell_t(\hat{\theta}_{t-1})  \lVert A_t \rVert^2_{V_{t}^{-1}}.
	\end{align*}
	
	Also, trivially we have
	\begin{align*}
		\min_{\theta \in \Theta } \psi_{1}(\theta)  &= 0.
	\end{align*}
	Putting everything together, we have
	\begin{align*}
		\sum_{t=1}^{n} \left( \ell_t(\hat{\theta}_{t-1}) - \ell_t(\theta^*) \right) &= \frac{\lambda}{2} \lVert \theta^* \rVert_2^2 + \sum_{t=1}^{n} \ell_t(\hat{\theta}_{t-1}) \lVert A_t \rVert^2_{V_{t}^{-1}} - \frac{1}{2} \lVert \hat{\theta}_n - \theta^* \rVert^2_{V_n}.
	\end{align*}
\end{proof}

\begin{lemma}\label{sketch-lemma:online-learning-lemma}
	In the context of the LinMED algorithm, we have
	\begin{align*}
		\EE \sbr{ \sum_{t=1}^{n} \left( A_t^{\T}(\hat{\theta}_{t-1} - \theta^*) \right)^2 \left(1 - \lVert A_t \rVert_{V_{t}^{-1}}^2 \right) } \leq \lambda \lVert \theta^* \rVert_2^2 + 2\sigma_*^2 d \log \left( 1 + \frac{n}{d \lambda}\right). 
	\end{align*}
	where $\sigma_*^2$ is sub-gaussian parameter of the noise.
\end{lemma}
\begin{proof}
	Let $r_t = A_t^{\T}(\hat{\theta}_{t-1} - \theta^*)$, $D_t = \lVert A_t \rVert^2_{V_{t}^{-1}}$.
	\begin{align*}
		\ell_t(\hat{\theta}_{t-1}) - \ell_t(\theta^*)  &= \frac{1}{2} \left( \left(A_t^T\hat{\theta}_{t-1} - y_t\right)^2 - \left(A_t^T\theta^* - y_t\right)^2 \right)\\
		&= \frac{1}{2} \left( \left(A_t^T\left(\hat{\theta}_{t-1} - \theta^*\right) - \eta_t\right)^2 - \eta_t^2 \right) \tag{because $y_t = A_t^{\T} \theta^* + \eta_t$}\\
		&=  \frac{1}{2} \left( A_t^T\left(\hat{\theta}_{t-1} - \theta^*\right) \right)^2 - A_t^T\left(\hat{\theta}_{t-1} - \theta^*\right) \cdot \eta_t \\
		&=  \frac{1}{2}r_t^2 - \eta_t r_t .
	\end{align*}
	Hence,
	\begin{align*}
		\sum_{t=1}^{n} \left( \ell_t(\hat{\theta}_{t-1}) - \ell_t(\theta^*) \right) &= \frac{1}{2}\sum_{t=1}^{n}r_t^2  - \sum_{t=1}^{n} \eta_t r_t.
	\end{align*}
	By Lemma \ref{sketch-lemma:regret-equality-modified}, we have
	\begin{align*}
		\sum_{t=1}^{n} \left( \ell_t(\hat{\theta}_{t-1}) - \ell_t(\theta^*) \right) &= \frac{\lambda}{2} \lVert \theta^* \rVert_2^2 + \sum_{t=1}^{n} \ell_t(\hat{\theta}_{t-1}) \lVert A_t \rVert_{V_{t}^{-1}} - \frac{1}{2} \lVert \hat{\theta}_n - \theta^* \rVert^2_{V_n}\\
		&\leq \frac{\lambda}{2} \lVert \theta^* \rVert_2^2 + \sum_{t=1}^{n} \ell_t(\hat{\theta}_{t-1}) \lVert A_t \rVert^2_{V_{t}^{-1}}.
	\end{align*}
	Hence,
	\begin{align*}
		\frac{1}{2}\sum_{t=1}^{n}r_t^2  - \sum_{t=1}^{n} \eta_t r_t  &\leq \frac{\lambda}{2} \lVert \theta^* \rVert_2^2 + \sum_{t=1}^{n} \ell_t(\hat{\theta}_{t-1}) \lVert A_t \rVert^2_{V_{t}^{-1}} \\
		\frac{1}{2}\sum_{t=1}^{n}r_t^2  - \sum_{t=1}^{n} \eta_t r_t  &\leq  \frac{\lambda}{2} \lVert \theta^* \rVert_2^2 + \sum_{t=1}^{n} \ell_t(\hat{\theta}_{t-1})D_t.
	\end{align*}
	Also note that $\ell_t(\hat{\theta}_{t-1})$ can be expanded as follows,
	\begin{align*}
		\ell_t(\hat{\theta}_{t-1}) &= \frac{1}{2} \left( A_t^{\T} \hat{\theta}_{t-1} -y_t \right)^2\\
		&=  \frac{1}{2} \left( A_t^{\T} \hat{\theta}_{t-1} - A_t^{\T} \theta^* - \eta_t \right)^2\\
		&= \frac{1}{2} \left( A_t^{\T} \left(\hat{\theta}_{t-1} - \theta^* \right) - \eta_t \right)^2\\
		&= \frac{1}{2} \left( r_t - \eta_t \right)^2\\
		&= \frac{1}{2}  \left( r_t^2 - 2r_t \eta_t + \eta_t^2 \right).
	\end{align*}
	Hence,
	\begin{align*}
		\frac{1}{2}\sum_{t=1}^{n}r_t^2  - \sum_{t=1}^{n} \eta_t r_t  &\leq  \frac{\lambda}{2} \lVert \theta^* \rVert_2^2 + \sum_{t=1}^{n} \frac{1}{2}  \left( r_t^2 - 2r_t \eta_t + \eta_t^2 \right) D_t\\
		\frac{1}{2}\sum_{t=1}^{n}r_t^2 \left(1 - D_t \right) &\leq   \frac{\lambda}{2} \lVert \theta^* \rVert_2^2 +  \frac{1}{2}  \sum_{t=1}^{n} \eta_t^2 D_t + \frac{1}{2}  \sum_{t=1}^{n}r_t \eta_t  \left(1 - D_t \right) .
	\end{align*}
	We can take expectation both sides,
	\begin{align*}
		\EE \sbr{\sum_{t=1}^{n}r_t^2 \left(1 - D_t \right) } &\leq \lambda  \lVert \theta^* \rVert_2^2 + \EE \sbr{  \sum_{t=1}^{n} \eta_t^2 D_t  } + \EE \sbr{ \sum_{t=1}^{n}r_t \eta_t  \left(1 - D_t \right) }\\
		&=  \lambda  \lVert \theta^* \rVert_2^2  +  \EE \sbr{ \sum_{t=1}^{n} \EE_{t-1}\sbr{\eta_t^2} D_t  }  + \EE \sbr{  \sum_{t=1}^{n}r_t\EE_{t-1}\sbr{\eta_t} \left(1 - D_t \right) }\\
		&= \lambda  \lVert \theta^* \rVert_2^2  + \sigma_*^2 \EE \sbr{ \sum_{t=1}^{n}  D_t  } + 0\\
		&= \lambda  \lVert \theta^* \rVert_2^2  + \sigma_*^2 \EE \sbr{ \sum_{t=1}^{n}  \lVert A_t \rVert^2_{V_{t}^{-1}} } \\
		&\leq \lambda \lVert \theta^* \rVert_2^2  + 2 d \sigma_*^2 \log \left( 1 + \frac{n}{d\lambda}\right). \tag{Lemma \ref{app-lemma:epl-lemma}} \\ 
	\end{align*}
	Thus proved.
\end{proof}
\begin{lemma}\label{sketch-lemma:matrix-inversion-lemma}
	In the context of the LinMED algorithm, we have
	\begin{align*}
		\lVert A_t \rVert_{V_{t-1}^{-1} }^2 \leq 1 \iff	\lVert A_t \rVert_{V_{t}^{-1} }^2 \leq \frac{1}{2}.
	\end{align*}
	
\end{lemma}
\begin{proof}
	Sherman-Morrison formula says,
	\begin{align*}
		\left(A + uv^{\T} \right)^{-1} &= A^{-1} - \frac{A^{-1} uv^{\T} A^{-1}}{1 + v^{\T}A^{-1}u}.
	\end{align*}
	replace $A = V_{t-1}$ and $u=v=A_t$,
	\begin{align*}
		\left(V_{t-1} + A_tA_t^{\T} \right)^{-1} &= V_{t-1}^{-1} -\frac{V_{t-1}^{-1} A_tA_t^{\T} V_{t-1}^{-1}}{1 + A_t^{\T}V_{t-1}^{-1}A_t}\\ 
		V_t^{-1} &= V_{t-1}^{-1} -\frac{V_{t-1}^{-1} A_tA_t^{\T} V_{t-1}^{-1}}{1 + \lVert A_t \rVert_{V_{t-1}^{-1} }^2}\\ 
		A_t^{\T}V_t^{-1}A_t &= A_t^{\T}V_{t-1}^{-1}A_t - \frac{A_t^{\T}V_{t-1}^{-1} A_tA_t^{\T} V_{t-1}^{-1}A_t}{1 + \lVert A_t \rVert_{V_{t-1}^{-1} }^2}\\ 
		\lVert A_t \rVert_{V_{t}^{-1} }^2 &= \lVert A_t \rVert_{V_{t-1}^{-1} }^2 - \frac{\lVert A_t \rVert_{V_{t-1}^{-1} }^4}{1 + \lVert A_t \rVert_{V_{t-1}^{-1} }^2}\\
		&= \lVert A_t \rVert_{V_{t-1}^{-1} }^2 \left( 1 -\frac{\lVert A_t \rVert_{V_{t-1}^{-1} }^2}{1 + \lVert A_t \rVert_{V_{t-1}^{-1} }^2} \right)\\
		&= \lVert A_t \rVert_{V_{t-1}^{-1} }^2 \left( \frac{1}{1 + \lVert A_t \rVert_{V_{t-1}^{-1} }^2} \right)\\
		&= 1 - \frac{1}{1 + \lVert A_t \rVert_{V_{t-1}^{-1} }^2}.
	\end{align*}
	If $\lVert A_t \rVert_{V_{t-1}^{-1} }^2 \leq 1$,
	\begin{align*}
		\lVert A_t \rVert_{V_{t}^{-1} }^2&\leq 1 - \frac{1}{2} \tag{ $\lVert A_t \rVert_{V_{t-1}^{-1} }^2 \leq 1$}\\
		&=\frac{1}{2}.
	\end{align*}
	Hence,
	\begin{align*}
		\lVert A_t \rVert_{V_{t-1}^{-1} }^2 \leq 1 \implies	\lVert A_t \rVert_{V_{t}^{-1} }^2 \leq \frac{1}{2}.
	\end{align*}
	If $\lVert A_t \rVert_{V_{t}^{-1} }^2 \leq \frac{1}{2}$
	\begin{align*}
		1 - \frac{1}{1 + \lVert A_t \rVert_{V_{t-1}^{-1} }^2} &\leq \frac{1}{2}\\
		\lVert A_t \rVert_{V_{t-1}^{-1} }^2 &\leq 1.
	\end{align*}
	Hence,
	\begin{align*}
		\lVert A_t \rVert_{V_{t}^{-1} }^2 \leq \frac{1}{2} \implies	\lVert A_t \rVert_{V_{t-1}^{-1} }^2 \leq 1.
	\end{align*}
	Thus proved.
\end{proof}
\begin{lemma}\label{sketch-lemma:online-learning-modified-lemma}
	In the context of the LinMED algorithm, we have
	\begin{align*}
		\EE \sbr{ \sum_{t=1}^{n} \left( A_t^{\T}(\hat{\theta}_{t-1}- \theta^*) \right)^2 \one\cbr{\lVert A_t \rVert_{V_{t-1}^{-1} }^2 \leq 1 } } \leq 2 \lambda \lVert \theta^* \rVert_2^2 + 4\sigma_*^2 d \log \left( 1 + \frac{n}{d \lambda}\right). 
	\end{align*}
	where $\sigma_*^2$ is sub-gaussian parameter of the noise.
\end{lemma}
\begin{proof}
	\begin{align*}
		\lVert A_t \rVert_{V_{t-1}^{-1} }^2 \leq 1 \iff	\lVert A_t \rVert_{V_{t}^{-1} }^2 \leq \frac{1}{2}. \tag{by Lemma \ref{sketch-lemma:matrix-inversion-lemma}}
	\end{align*}
	Hence,
	\begin{align*}
		\one\cbr{\lVert A_t \rVert_{V_{t-1}^{-1} }^2 \leq 1 } = \one\cbr{ \lVert A_t \rVert_{V_{t}^{-1} }^2 \leq \frac{1}{2}}.
	\end{align*}
	\begin{align*}
		&\EE \sbr{ \sum_{t=1}^{n} \left( A_t^{\T}(\hat{\theta}_{t-1} - \theta^*) \right)^2 \left(1 - \lVert A_t \rVert_{V_{t}^{-1}}^2 \right) } \\
		&\spacex\geq \EE \sbr{ \sum_{t=1}^{n} \left( A_t^{\T}(\hat{\theta}_{t-1} - \theta^*) \right)^2 \left(1 - \lVert A_t \rVert_{V_{t}^{-1}}^2 \right)\one\cbr{\lVert A_t \rVert_{V_{t-1}^{-1} }^2 \leq 1 }\one\cbr{\lVert A_t \rVert_{V_{t-1}^{-1} }^2 \leq 1 } }\\
		&\spacex=\EE \sbr{ \sum_{t=1}^{n} \left( A_t^{\T}(\hat{\theta}_{t-1} - \theta^*) \right)^2 \left(1 - \lVert A_t \rVert_{V_{t}^{-1}}^2 \right)\one\cbr{ \lVert A_t \rVert_{V_{t}^{-1} }^2 \leq \frac{1}{2}}\one\cbr{\lVert A_t \rVert_{V_{t-1}^{-1} }^2 \leq 1 } }\\
		&\spacex\geq \frac{1}{2}\EE \sbr{ \sum_{t=1}^{n} \left( A_t^{\T}(\hat{\theta}_{t-1} - \theta^*) \right)^2 \one\cbr{\lVert A_t \rVert_{V_{t-1}^{-1} }^2 \leq 1 } }.
	\end{align*}
	Hence,
	\begin{align*}
		\EE \sbr{ \sum_{t=1}^{n} \left( A_t^{\T}(\hat{\theta}_{t-1} - \theta^*) \right)^2 \one\cbr{\lVert A_t \rVert_{V_{t-1}^{-1} }^2 \leq 1 } } &\leq 	2\EE \sbr{ \sum_{t=1}^{n} \left( A_t^{\T}(\hat{\theta}_{t-1} - \theta^*) \right)^2 \left(1 - \lVert A_t \rVert_{V_{t}^{-1}}^2 \right) }\\
		&\leq 2\lambda\lVert \theta^* \rVert_2^2 + 4\sigma_*^2 d \log \left( 1 + \frac{n}{d \lambda}\right). \tag{by Lemma \ref{sketch-lemma:online-learning-lemma}}
	\end{align*}
\end{proof}
\begin{claim}\label{app-claim:confident-interval-equi-expectation-lemma}
	In the context of the LinMED algorithm, we have
		 \begin{align*}
		 	 \sum_{t=1}^{n} \lVert \theta^* - \hat{\theta}_{t-1} \rVert_{V(p_t)}^2  &=  \EE_{A_t \sim p_t} \sbr{ \left( \left( \theta^* - \hat{\theta}_{t-1}\right) A_t^{\T}\right)^2 }. \\
		 \end{align*}
\end{claim}

\begin{proof}
	\begin{align*}
		 \lVert \theta^* - \hat{\theta}_{t-1} \rVert_{V(p_t)}^2 &=  \lVert \theta^* - \hat{\theta}_{t-1} \rVert_{V(p_t)}^2 \\
		&=  \left( \theta^* - \hat{\theta}_{t-1}\right)^{\T} V(p_t) \left( \theta^* - \hat{\theta}_{t-1}\right)   \\
		&=   \left( \theta^* - \hat{\theta}_{t-1}\right)^{\T} \sum_{a \in \mathcal{A}_t} p_t(a)aa^{\T}\left( \theta^* - \hat{\theta}_{t-1}\right)    \\
		&=  \sum_{a \in \mathcal{A}_t}  p_t(a) \left( \left( \theta^* - \hat{\theta}_{t-1}\right) a^{\T}\right) \left( \left( \theta^* - \hat{\theta}_{t-1}\right) a^{\T} \right)^{\T} \one\cbr{\mathcal{U}_{t-1}(A_t)}  \\
		&= \sum_{a \in \mathcal{A}_t}  p_t(a) \left( \left( \theta^* - \hat{\theta}_{t-1}\right) a^{\T}\right)^2  \\
		&=  \EE_{A_t \sim p_t} \sbr{ \left( \left( \theta^* - \hat{\theta}_{t-1}\right) A_t^{\T}\right)^2 }.
	\end{align*}
\end{proof}

\begin{lemma}[Elliptical potential count lemma adapted from Lemma C.2 of \citet{jun24noiseadaptive} ]\label{app-lemma:epc-lemma}
	Let $x_1,x_2,...,x_t \in \mathbb{R}^d$ be a sequence of vectors with $\lVert x_s\rVert_2 \leq 1, \forall s \in [t]$. Let $V_t = \lambda I + \sum_{s=1}^{t}x_sx_s^{\T}$ for some $\lambda > 0$. Let $J = \{ s \in [t] : \lVert x_s\rVert^2_{V_{s-1}^{-1}} \geq L^2 \}$ for some $L^2 \leq 1$. Then,
	\begin{align*}
		 \lvert J \rvert \leq 3 \frac{d}{L^2}\ln \left( 1 + \frac{2}{L^2 \lambda}\right).
	\end{align*}
\end{lemma}

\begin{lemma}[Elliptical potential lemma adapted from Proposition 2 of \citet{abeille17linear} ]\label{app-lemma:epl-lemma}
	Let $x_1,x_2,...,x_t \in \mathbb{R}^d$ be a sequence of vectors with $\lVert x_s\rVert_2 \leq 1, \forall s \in [t]$. Let $V_t = \lambda I + \sum_{s=1}^{t}x_sx_s^{\T}$ for some $\lambda > 0$. Then,
	\begin{align*}
	 	\sum_{s=1}^{t} \lVert x_s \rVert^2_{V_{s}^{-1}} \leq 2d \log (1 + \frac{t}{d \lambda}).
	\end{align*}
\end{lemma}

\begin{lemma}[OFUL confidence bound lemma adapted from Theorem 2 of \citet{ay11improved}]\label{app-lemma:OFUL-conf-bound-lemma} 
	Assume  $\forall s\in [t], \spacex \lVert a_s\rVert \leq 1$, and $\lVert \theta^*\rVert_2 \leq S$, for some fixed S > 0. We also assume $\Delta_a := \max_{a^{'} \in \mathcal{A}_{t}} \langle a^{'}, \theta^*\rangle - \langle a, \theta^*\rangle \leq 1, \spacex \forall a \in \mathcal{A}$
	\begin{align*}
	\forall t \geq 1, \spacex 	\PP \left( \lVert \hat{\theta}_{t-1} - \theta^* \rVert_{V_{t-1}} \leq \sqrt{\beta_{t-1}(\delta_{t-1})} \right) \geq 1 - \delta.
	\end{align*}
\end{lemma}

\clearpage
\section{LOWER BOUND ARGUMENTS}\label{app-section:low-bound-args-section}
In this section, we establish an instance-dependent regret lower bound of order $\Omega(\Delta \sqrt{n})$ for the modified version of EXP2 (outlined below) as well as for SpannerIGW \citep{zhu22contextual}.To demonstrate this, we consider the following instance:

Let the arm set $\mathcal{A} \subset \mathbb{R}^2$ be 

\begin{align*}
	\mathcal{A} := \{e_1,e_2\} \tag{where as $e_1 = (1,0)^\T, e_2 = (0,1)^\T$}.
\end{align*}
Let $\theta^* = (1,0)^\T$ and dimension of the arm set and $\theta^*$ be $d=2$
\subsection{Lower bound for EXP2 algorithm (modified version)}\label{app-subsection:sexp2-lower-bound-subsection}
The original EXP2 algorithm \citep{bubeck12towards} was developed for the bounded loss model. However, we introduce a slightly modified version of this algorithm to accommodate the unbounded reward setting. This modified algorithm is essential for establishing the lower bound for regret in such environments.

\begin{algorithm} 
	\textbf{Input:} Finite Arm set $\mathcal{A} \in \mathbb{R}^d$, learning rate $\eta$,  
	exploration distribution $ \pi $, exploration parameter $\gamma$
	Optimal design fraction: $\alpha_{\opt}$, 
	\begin{algorithmic}[1] 	
		\FOR{$t=1,2,\ldots n $}
		\STATE Compute sampling distribution
		\begin{align*}
			P_t(a) = \gamma \pi(a) + (1-\gamma) \frac{\exp\left( \eta \langle a, \sum_{s=1}^{t-1} \hat{\theta}_{s}\rangle\right)}{\sum_{a^{'} \in \mathcal{A}_{t}}\exp\left( \eta \langle a^{'}, \sum_{s=1}^{t-1} \hat{\theta}_{s}\rangle\right)  }. 
		\end{align*}
		\STATE Sample action $$A_t \sim p_t.$$
		\STATE Observe the reward, $$ Y_t = \langle \theta^*, A_t \rangle + \eta_t. $$ 
		\STATE Update $$ \hat{\theta}_t = Q_t^{-1} A_tY_t . $$ where $Q_t = \sum_{a \in \mathcal{A}_{t}} P_t(a) a a^{\T}.$
		\ENDFOR
	\end{algorithmic}
	\caption{EXP2 (Reward Version) }
	\label{app-algo:exp2-algorithm}
\end{algorithm}

\newtheorem*{app-thm:exp2-inst-dep-bound-theorem}{Theorem \ref{main-thm:exp2-inst-dep-bound-theorem}}

\begin{app-thm:exp2-inst-dep-bound-theorem}\label{app-thm:exp2-inst-dep-bound-theorem}
	There exists a linear bandit problem for which the EXP2 algorithm satisfies
	\begin{align*}
		\EE\Reg_n \geq \Omega(\Delta \sqrt{n}).
	\end{align*}
\end{app-thm:exp2-inst-dep-bound-theorem}
\begin{proof}
The EXP2 algorithm samples an arm according to the following probability expression:
\begin{align*}
	P_t(a) &= \gamma \pi(a) + (1-\gamma) \frac{\exp \left( \eta \sum_{s=1}^{t-1} \langle \hat{\theta}_{s} ,a \rangle \right)  }{\sum_{a^{'} \in \mathcal{A} } \exp \left( \eta \sum_{s=1}^{t-1} \langle \hat{\theta}_{s} ,a^{'} \rangle \right) }.
\end{align*}
Moreover, from Lemma \ref{app-lemma:lower-bound-lemma}, we know that, $\pi(e_1) = \pi(e_2) = \frac{1}{2}$ forms a valid G-optimal design~\citep{kiefer60theequivalence}.

Hence,
\begin{align*}
	R_n &=\EE\sbr{ \sum_{t=1}^{n} \langle \theta^*, a^*_t \rangle - \langle \theta^*, A_t \rangle}\\
	&= \EE \sbr{\sum_{t=1}^{n}\overline{ \Delta}_{t} \one \cbr{A_t \neq a^*_t}}\\
	&= \EE \sbr{\sum_{t=1}^{n} \Delta_{e_2} \one \cbr{A_t = e_2 }}\\
	&= \sum_{t=1}^{n} \Delta_{e_2} P_t(e_2)\\
	&\geq  \sum_{t=1}^{n} \Delta_{e_2} \gamma \pi(e_2)\\
	&= \gamma \frac{n}{2} \Delta.  \tag{$\Delta= \min_{a \in \mathcal{A} , a \neq a^*_t } \Delta_a = \Delta_{e_2}$ }
\end{align*}
Furthermore, to achieve an optimal minimax bound, we must appropriately tune the parameter $\gamma$ as follows~\citep{bubeck12towards}: 
\begin{align*}
	\gamma = \sqrt{\frac{ g(\pi)^2 \log K }{ (2 g(\pi) + d) n}}. \tag{$K = \text{ Number of arms } = 2$}
\end{align*}
Hence,
\begin{align*}
	R_n \geq  \Delta \cdot  \sqrt{\frac{n g(\pi)^2 \log K }{ (2 g(\pi) + d)}}  &= \Omega(\Delta \sqrt{n}).
\end{align*}

This concludes the proof.
\end{proof}
\subsection{SpannerIGW} \label{app-subsection:spannerigw-lower-bound-subsection}
\newtheorem*{app-thm:spannerigw-inst-dep-bound-theorem}{Theorem \ref{main-thm:spannerigw-inst-dep-bound-theorem}}
\begin{app-thm:spannerigw-inst-dep-bound-theorem}\label{app-thm:spannerigw-inst-dep-bound-theorem}
	There exists a linear bandit problem for which the SpannerIGW algorithm satisfies
	\begin{align*}
		\EE\Reg_n \geq \Omega(\Delta \sqrt{n}).
	\end{align*}
\end{app-thm:spannerigw-inst-dep-bound-theorem}
\begin{proof}
Let $\overline{\mathcal{A}}_{(t)}$ be the transformed arm set at time $t$, SpannerIGW~\citep{zhu22contextual} calculates an approximate design ($q_t^{\opt}$), similar to G-optimal design. From Lemma \ref{app-lemma:lower-bound-lemma},
\begin{align*}
	q_t^{\opt} (e_2) = \frac{1}{2}.
\end{align*}
\begin{align*}
	q_t(a) &= \frac{1}{2} q_t^{\opt}(a) + \frac{1}{2} \mathbb{I}_{\hat{a}_t}(a)\\
	q_t(e_2) &\geq \frac{1}{4}.
\end{align*}

Let $n \ge 8$.
Recall that $\Delta = 1$.
Let us further assume that there is no noise in the reward: $Y_t = \la A_t,\th^*\ra + \eta_t$ where $\eta_t=0$ with probability 1.
This noise $\eta_t$ can be viewed as having $1$-sub-Gaussian noise, and we assume that the algorithm only knows that the noise is 1-sub-Gaussian and does not know that the actual noise is deterministically 0.

For the regression oracle, let us use the online ridge regression $\hth_t = (\tau I + \sum_{s=1}^t A_s A_s^\T)^{-1}\sum_{s=1}^t A_s Y_s \in \RR^2$ with regularizer $\tau = 1/10$ and the initial parameter $\hth_0 = 0 \in \RR^2$.
This will enjoy a regret bound of $\Reg_\Sq(t) \le c \log(t)$  for some $c$.
Let $N_1$ be the number of times arm $e_1$ has been pulled up to (and including) time step $t-1$.
Then, if $N_1 \ge 1$, then the prediction at time $t$ will be $\hf_{t}(e_1) = \fr{N}{N + \tau}$ and 
\begin{align*}
  0.9 \le  \fr{1}{1 + \tau} \le \hf_t(e_1) \le 1~.
\end{align*}
Furthermore, the prediction for $e_2$ is always $\hf_t(e_2) = 0$ at all time.

Note that once both arms have been pulled at least once up to (and including) time step $t-1$, then the probability of pulling the arm $e_2$ is
\begin{align*}
  \fr{q_t^{\mathrm{opt}} + \fr12 \II_{\ha_t}}{\lam + \eta(\hf_t(\ha_{t}) - \hf_t(e_2))} 
  \ge \fr{\fr12\cd\fr12 + 0}{\lam + \eta\cd (1- 0)} 
  \ge  \fr{\fr14}{1 + \fr{\gam}{C_{\opt} d} } 
  =    \fr{\fr14}{1 + \fr{\gam}{d} } 
  &=    \fr{\fr14}{1 + \fr1d \cd \sqrt{\fr{2n}{c \ln(n) + 32 \log(2/\dt)}} }. 
\end{align*}
where we set $\delta = \fr14$.

Let $J$ be the first time step at the end of which we have pulled both $e_1$ and $e_2$ at least once; i.e.,
\begin{align*}
  J := \min\cbr{t\in \NN_+: \sum_{s=1}^t \onec{A_t = e_1} \ge 1, \sum_{s=1}^t \onec{A_t = e_2} \ge 1}~.
\end{align*}

Then,
\begin{align*}
  \EE[\Reg_n] 
  &\ge \EE[\Reg_n \onec{J \le 4}]
\\&\ge \EE[\onec{J \le 4} \sum_{t=1}^n \onec{A_t = e_2}]
\\&\ge \EE[\onec{J \le 4} \sum_{t=5}^n \onec{A_t = e_2}]
\\&= \EE[\onec{J \le 4} \sum_{t=5}^n \EE[\onec{A_t = e_2} \mid A_1,\ldots,A_{t-1}]]
\\&\ge \EE[\onec{J \le 4} ] \cd  (n-4)\cd \fr{\fr14}{1 + \fr1d \cd \sqrt{\fr{2n}{c \ln(n) + 32 \log(2/\dt)}} } 
\\&\ge \PP(J \le 4) \cd  \Omega(\sqrt{n \log(n)}).
\end{align*}
It remains to show that $\PP(J \le 4)$ is lowerbounded by an absolute constant.
Note that
\begin{align*}
  \PP(J \le 4)
  = 1 - \PP(J \ge 5).
\end{align*}
and
\begin{align*}
  \PP(J \ge 5)
  \le \PP(\forall t\in[4], A_t = e_1) + \PP(\forall t\in[4], A_t = e_2).
\end{align*}
For the first term,
\begin{align*}
  \PP(\forall t\in[4], A_t = e_1) 
  &= \PP(A_1 = e_1) \prod_{t=2}^4 \PP(A_t = e_1 \mid A_{1:t-1} = e_1)
\\&\le \del{\fr14 + \fr12} \cd 1^3
\\&\le \fr34.
\end{align*}
For the second term,
\begin{align*}
  \PP(\forall t\in[4], A_t = e_2) 
   \le (\fr{\fr14}{\lam})^4
  ~\le (\fr{\fr14}{\fr12})^4 
  ~= \fr{1}{16}.
\end{align*}
Thus,
\begin{align*}
  \PP(J \le 4) \ge 1 - \fr{13}{16}  = \fr{3}{16} ~.
\end{align*}
This implies that
\begin{align*}
	 \EE[\Reg_n]  \geq \Omega(\sqrt{n \log(n)}\Delta) ~.
\end{align*}
This concludes the proof.
\end{proof}

\subsection{Lemmata}
The objective of the following lemma \ref{app-lemma:lower-bound-lemma} is to demonstrate that, regardless of the scaling applied to the arm set in the previous section (orthogonal basis of $\mathbb{R}^2$), a probability distribution that assigns equal probability to each arm will still satisfy Assumption \ref{main-assump:opt-lev-scr-assumption} and Assumption \ref{main-assump:opt-cardinality-assumption}. Moreover, such a distribution is a G-optimal design~\citep{kiefer60theequivalence}. This lemma plays a crucial role in establishing the lower bound proof.

\begin{lemma}\label{app-lemma:lower-bound-lemma}
	For an arm set $\mathcal{A} := \{p\cdot e_1, q\cdot e_2\}$ where $ p,q > 0$  and $e_1 = (1,0)^\T, e_2 = (0,1)^\T$. Let $\pi$ be a probability distribution that assigns probability to each arms as follows:
	\begin{align*}
		\pi(a_1) = \pi(a_2) = \frac{1}{2}. \tag{where $\forall i, \spacex a_i$ denotes the $i$-th arm in $\mathcal{A}$}
	\end{align*}
	Then, $\pi$ is G-optimal design.
\end{lemma}
\begin{proof}
	\begin{align*}
		\mathcal{A} := \{p\cdot e_1, q\cdot e_2\}. \tag{where $ p,q > 0$}
	\end{align*}
	Here, dimension $d=2$ and $\lvert \mathcal{A} \rvert = 2$.
	Let be $\pi$ be probability distribution with,
	
	\begin{align*}
		\pi(a_1) = \pi(a_2) = \frac{1}{2}.
	\end{align*}
	\begin{align*}
		\sum_{i=1}^{2} \pi(a_i) a_i a_i^T &= \frac{1}{2} \left( \begin{bmatrix}
			p^2 & 0\\
			0 & 0
		\end{bmatrix} + \begin{bmatrix}
			0 & 0\\
			0 & q^2
		\end{bmatrix}\right)\\
		V(\pi) &:=  \begin{bmatrix}
			\frac{p^2}{2} & 0\\
			0 & \frac{q^2}{2}
		\end{bmatrix}\\
		V^{-1}(\pi) &= \frac{4}{p^2\cdot q^2}   \begin{bmatrix}
			\frac{q^2}{2} & 0\\
			0 & \frac{p^2}{2}
		\end{bmatrix}\\
		&= \frac{1}{p^2\cdot q^2}   \begin{bmatrix}
			2p^2 & 0\\
			0 & 2q^2
		\end{bmatrix}.
	\end{align*}
	\begin{align*}
		\lVert a_1 \rVert_{V^{-1}(\pi)}^2 &= p^2 e_1^T V^{-1}(\pi) e_1\\
		&= 2.
	\end{align*}
	\begin{align*}
		\lVert a_2 \rVert_{V^{-1}(\pi)}^2 &= q^2 e_2^T V^{-1}(\pi) e_2\\
		&= 2.
	\end{align*}
	Let 
	\begin{align*}
		g(\pi) &:= \max_{a \in \mathcal{A} } \lVert a  \rVert_{V^{-1}(\pi)}^2 \\
		&= 2  = d .
	\end{align*}
	However, A G-optimal design~\citep{kiefer60theequivalence} $\pi^*$ should satisfy the following : 
	\begin{align*}
		g(\pi^*) = d.
	\end{align*}
	Hence we can conclude,
	\begin{align*}
		\pi = \pi^*.
	\end{align*}
	Hence, $\pi$ is G-optimal design.
\end{proof}
\clearpage
\section{ALGORITHM FOR APPROXIMATE OPTIMAL EXPERIMENTAL DESIGN}\label{app-section:approx-design-section}
In this section, we present a procedure to obtain a $\mathrm{ApproxDesign}()$ which can satisfy the assumptions \ref{main-assump:opt-lev-scr-assumption} and \ref{main-assump:opt-cardinality-assumption}.

 This procedure consist of two steps. The first step involves implementing the computationally efficient version of the BH sampling algorithm, as presented by \citet{gales22norm} (in appendix Section C.1, Algorithm 5), which refines the original algorithm by \citet{betke93approximating}. We present this in Algorithm~\ref{app-algo:comp-efficient-BH}. This implementation outputs $\mathcal{A}_0 = \{a_1,...,a_{\lvert \mathcal{A}_0 \rvert}\} \subseteq \mathcal{A}$ such that, $\lvert \mathcal{A}_0\rvert \leq 2 d$ and the determinant of the matrix $\overline{V}_{\lvert \mathcal{A}_0 \rvert}$ is sufficiently large, where $\overline{V}_{k} := \sum_{s=1}^{k} a_sa_s^{\T}$.

\kj{in the current form, it is not super clear if Algorithm 4 is the same one as the one shown in Gales et al.}

\begin{algorithm} 
	\textbf{Input:} Original arm set $\mathcal{A} \subset \mathbb{R}^d$ with $\lvert \mathcal{A} \rvert = K$
	\begin{algorithmic}[1] 	
		\IF{$K \leq 2d$}
		\STATE $\mathcal{A}_0 \leftarrow \mathcal{A}$
		\RETURN $\mathcal{A}_0$
		\ENDIF
		\STATE $\Psi \leftarrow \{0\}$, $\mathcal{A}_0 \leftarrow \emptyset $, $i\leftarrow0$, $v_0 \leftarrow (0,\cdots,0)^\T \in \mathbb{R}^{d}$
		\WHILE{ $\mathbb{R}^{d}\setminus \Psi  \neq \emptyset$}
		\STATE $i \leftarrow i +1$
		\IF{$i$ = 1}
		\STATE Set $b_i = e_i$ where $e_i$ is the $i$-th index vector.
		\ELSE 
		\STATE Set $v_{i-1}^{\perp} = v_{i-1} - \sum_{j=0}^{i-2} \frac{\langle v_j^{\perp}, v_{i-1} \rangle}{\langle  v_j^{\perp}, v_j^{\perp} \rangle} v_j^{\perp}$
		\STATE Set $b_i = e_i - \sum_{j=0}^{i-1} \frac{\langle v_j^{\perp}, e_i\rangle}{\langle  v_j^{\perp}, v_j^{\perp} \rangle} v_j^{\perp}$
		\ENDIF
		\STATE $p \leftarrow \argmax_{a \in \mathcal{A}}\langle b_i,a\rangle $
		\STATE $q \leftarrow \argmin_{a \in \mathcal{A}}\langle b_i,a\rangle $
		\STATE $ \mathcal{A}_0 \leftarrow \mathcal{A}_0 \cup \{p\} \cup \{q\}$ \kj{[D ] not a standard way; use $\larrow$}
		\STATE $v_i \leftarrow p -q$
		\STATE $\Psi \leftarrow \mathrm{Span}(\Psi, {v_i})$ \kj{[ D] not a standard way; use $\larrow$}
		\ENDWHILE
		\RETURN $\mathcal{A}_0$
	\end{algorithmic}
	\caption{ Computationally efficient BH algorithm }
	\label{app-algo:comp-efficient-BH}
\end{algorithm}

The second step of the procedure is detailed in Algorithm \ref{app-algo:optimal-design-algorithm}. It takes the output $\mathcal{A}_0$ from the computationally efficient BH sampling algorithm, refines the probability distribution for arms in $\mathcal{A}$.

\begin{algorithm} 
	\textbf{Input:} Original arm set $\mathcal{A} \subset \mathbb{R}^d$
	\begin{algorithmic}[1] 	
		\STATE $\mathcal{A}_0$ $\leftarrow$    Computationally efficient BH algorithm($\mathcal{A}$)
		\STATE $ k \leftarrow \lvert \mathcal{A}_0 \rvert + 1$ 
		\WHILE {$ \max_{a \in \mathcal{A}} \lVert a \rVert_{\overline{V}_{k-1}^{-1}}^2 > 1 $}
		\STATE $a_{k} = \argmax_{a \in \mathcal{A}}  \lVert a \rVert_{\overline{V}_{k-1}^{-1}}^2 $  
		\STATE $k \leftarrow k + 1$
		\ENDWHILE
		\STATE $\tau = k-1$
		\RETURN $\pi$ such that $\forall a \in \mathcal{A}, \spacex \pi(a) = \frac{{C}^{\mathcal{A}}_{\tau}(a)}{\sum_{b \in \mathcal{A}} {C}^{\mathcal{A}}_{\tau}(b)  }$ where $\mathcal{C}^{\mathcal{A}}_k(a) := \sum_{s=1}^{k}  \one \cbr{a_s = a}, \spacex  \forall a \in \mathcal{A}.$
	\end{algorithmic}
	\caption{ ApproxDesign}
	\label{app-algo:optimal-design-algorithm}
\end{algorithm}
Furthermore, Theorem 5 of \citet{gales22norm} shows that,
\begin{align} \label{app-eq:tau-sample-equation}
	\tau = \mathcal{O}(d\log d).
\end{align}

\begin{claim}\label{app-claim:optimality-claim}
	Let $\pi$ be the design from the $\mathrm{ApproxDesign}()$ Algorithm \ref{app-algo:optimal-design-algorithm}, then,
	\begin{align*}
		\lVert a \rVert_{V^{-1}(\pi)}^2 &\leq C_\opt d \log(d), \forall a \in \mathcal{A}.
	\end{align*}
\end{claim}
\begin{proof}
	Note that, when the algorithm stops,
	\begin{align}\label{app-eq:algo-stop-condition-equation}
		\max_{a \in \mathcal{A}} \lVert a \rVert_{\overline{V}_{{\tau}}^{-1}}^2 &\leq 1 .
	\end{align}
	Furthermore, we have
	\begin{align*}
		\overline{V}_{\tau} &=  \sum_{s=1}^{{\tau}} a_sa_s^{\T}\\
		&= \sum_{a \in \mathcal{A}} {C}^{\mathcal{A}}_{\tau}(a) aa^{\T}\\
		&= \left(\sum_{b \in \mathcal{A}} {C}^{\mathcal{A}}_{\tau}(b) \right)\sum_{a \in \mathcal{A}} \frac{{C}^{\mathcal{A}}_{\tau}(a)}{\sum_{b \in \mathcal{A}} {C}^{\mathcal{A}}_{\tau}(b)} aa^{\T}\\
		&= \left(\sum_{b \in \mathcal{A}} {C}^{\mathcal{A}}_{\tau}(b) \right)\sum_{a \in \mathcal{A}} \pi(a) aa^{\T}\\
		&= \left(\sum_{b \in \mathcal{A}} {C}^{\mathcal{A}}_{\tau}(b) \right)V(\pi)\\
		&= \tau V(\pi)\\
		&= \mathcal{O}(d\log d) V(\pi).\tag{from \eqref{app-eq:tau-sample-equation}}
	\end{align*}
Then,	\begin{align*}
	 \lVert a \rVert_{\overline{V}_{{\tau}}^{-1}}^2 &= 	\frac{1}{\mathcal{O}(d\log d)} \lVert a \rVert_{V^{-1}(\pi)}^2 \\
		\lVert a \rVert_{V^{-1}(\pi)}^2 &\leq  \mathcal{O}(d\log d) \cdot \max_{a \in \mathcal{A}} \lVert a \rVert_{\overline{V}_{{\tau}}^{-1}}^2 \\
		&\leq \mathcal{O}(d\log d) \cdot 1 \tag{from \eqref{app-eq:algo-stop-condition-equation}}\\
		&\leq  C_\opt d \log(d).
	\end{align*}
\end{proof}
\begin{claim}\label{app-claim:optimal-cardinality-claim}
		 \begin{align*}
			\lvert \supp(\pi) \rvert &= \tilde{\mathcal{O}}(d) ~.
		\end{align*}
\end{claim}
\begin{proof}
	\begin{align*}
		\lvert \supp(\pi) \rvert &\leq \tau \\
		&=  \mathcal{O}(d\log d) \tag{from \eqref{app-eq:tau-sample-equation}}\\
		&= \tilde{O}(d).
	\end{align*}
\end{proof}

Hence, it is proved that the Algorithm \ref{app-algo:optimal-design-algorithm} along with  computationally efficient BH sampling Algorithm \ref{app-algo:comp-efficient-BH} outputs a design that satisfies Assumptions \ref{main-assump:opt-lev-scr-assumption} and \ref{main-assump:opt-cardinality-assumption}.

\clearpage
\section{EMPIRICAL STUDIES}
\subsection{End of optimism experiments} \label{app-subsection:EOPT-eval-exp-subsection}
Since this experiment has already been covered in the main body of the paper, we present it concisely here.

\textbf{Experimental setup: } We set the number of arms $K=3$, dimension $d=2$ and $\mathcal{A} = \{a_1 = (1,0)^\T, a_2 = (0,1)^\T,a_3 = (1 - \eps, 2\eps)^\T\}$ where $\eps \in \{0.005,0.01,0.02\}$ and $\theta^* = (1,0)^\T$. The noise follows $\mathcal{N}(0, \sigma_*^2)$.  The time horizon for each trial is $n=1,000,000$ and conduct $20$ such independent trials. Furthermore, we conduct experiments for the cases i) $\sigma^2 = \sigma_*^2 $ where $\sigma_*^2 = 0.1$, ii) $\sigma^2 = 2\cdot \sigma_*^2$ where $\sigma_*^2 = 0.1$, and iii) $\sigma^2 = 0.1 \cdot \sigma_*^2$ where $\sigma_*^2 = 10$ .
 
 \textbf{Algorithms evaluated: }We evaluate the following algorithms: OFUL~\citep{ay11improved}, Lin-TS-Freq (Thompson sampling frequentest version)~\citep{agrawal14thompson}, Lin-TS-Bayes (Thompson sampling Bayesian version)~\citep{russo14learning}, Lin-IMED-1~\citep{bian24indexed}, Lin-IMED-3~\citep{bian24indexed}, LinMED-99 ($\alpha_{\mathrm{opt}} = 0.99$), LinMED-90 ($\alpha_{\mathrm{opt}} = 0.90$), and LinMED-50 ($\alpha_{\mathrm{opt}} = 0.50$). Note that, for Lin-TS-Bayes, we are still evaluating the frequentest regret as we do for every other algorithms. Lin-IMED-3 have a hyper-parameter $C$, which we set to $C= 30$ following~\citet{bian24indexed}.

\def\mygapapp{0.5}
\def\mygapappin{0.6}
\def \hecmmar{-2cm}
\def \hecmgap{-4cm}
	\begin{figure*}[h!]
	\centering
	\hspace{\hecmmar}\begin{subfigure}[b]{\mygapapp\textwidth}
		\centering
		\includegraphics[width=\mygapappin\textwidth]{EOPT_C_0005_CF_CCR.pdf}
		
		\caption{$\sigma^2 = \sigma^2_* $}
	\end{subfigure}
	\hspace{\hecmgap}
	\begin{subfigure}[b]{\mygapapp\textwidth}
		\centering
		\includegraphics[width=\mygapappin\textwidth]{EOPT_W_0005_CF_CCR.pdf}
		\caption{$\sigma^2 = 2\cdot \sigma^2_* $}
	\end{subfigure}
	\hspace{\hecmgap}
		\begin{subfigure}[b]{\mygapapp\textwidth}
		\centering
		\includegraphics[width=\mygapappin\textwidth]{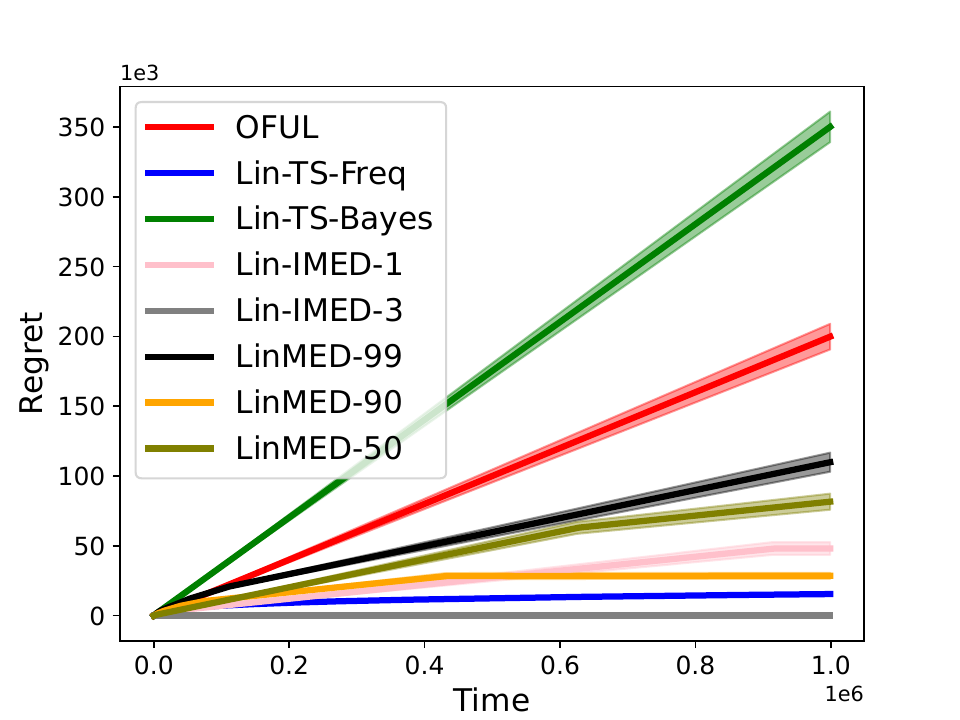}
		\caption{ $\sigma^2 = 0.1 \cdot \sigma^2_* $}
	\end{subfigure}\hspace{\hecmmar}
	\caption{End of optimism experiments $\eps = 0.005$ }
	\label{app-figure:end-of-opt-figure1}
\end{figure*}
\begin{figure*}[h!]
	\centering
	\hspace{\hecmmar}\begin{subfigure}[b]{\mygapapp\textwidth}
		\centering
		\includegraphics[width=\mygapappin\textwidth]{EOPT_C_001_CF_CCR.pdf}
		
		\caption{$\sigma^2 = \sigma^2_* $}
	\end{subfigure}
	\hspace{\hecmgap}
	\begin{subfigure}[b]{\mygapapp\textwidth}
		\centering
		\includegraphics[width=\mygapappin\textwidth]{EOPT_W_001_CF_CCR.pdf}
		
		\caption{$\sigma^2 = 2\cdot \sigma^2_* $}
	\end{subfigure}
	\hspace{\hecmgap}
	\begin{subfigure}[b]{\mygapapp\textwidth}
	\centering
	\includegraphics[width=\mygapappin\textwidth]{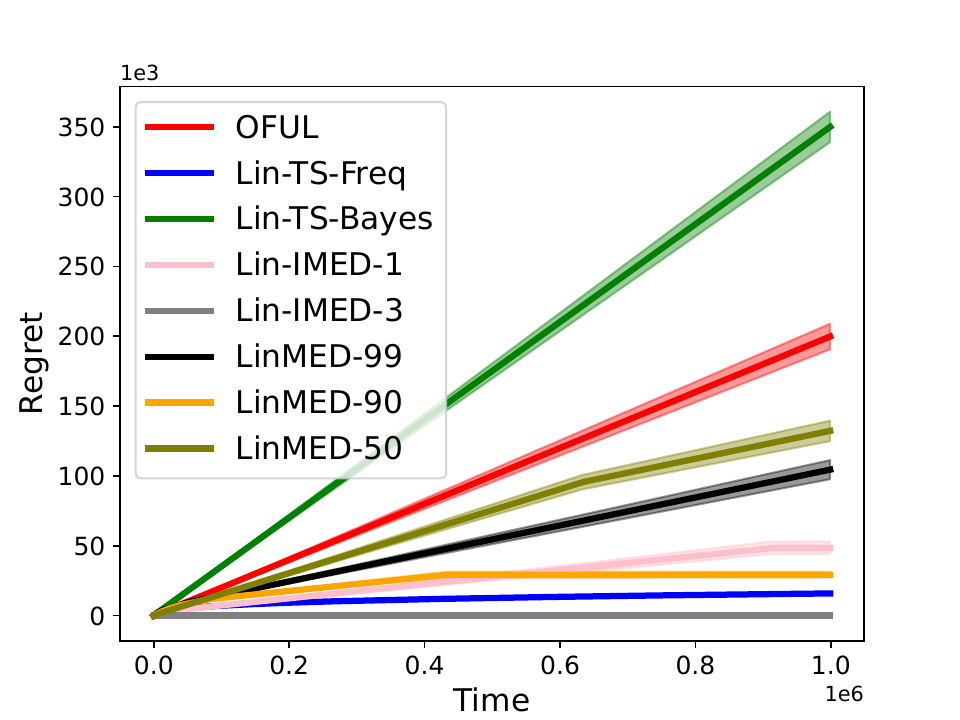}
	\caption{$\sigma^2 = 0.1 \cdot \sigma^2_* $}
	\end{subfigure}\hspace{\hecmmar}
	\caption{End of optimism experiments $\eps = 0.01$ }
	\label{app-figure:end-of-opt-figure2}
\end{figure*}
\begin{figure*}[h]
	\centering
	\hspace{\hecmmar}\begin{subfigure}[b]{\mygapapp\textwidth}
		\centering
		\includegraphics[width=\mygapappin\textwidth]{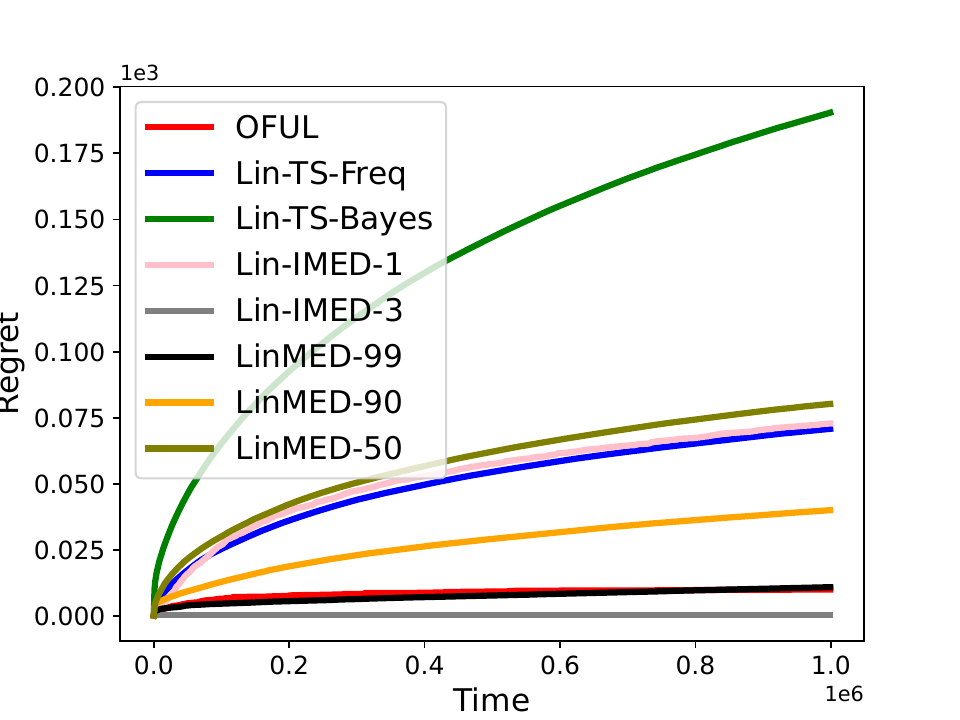}
		\caption{$\eps = 0.02$, $\sigma^2 = \sigma^2_* $}
	\end{subfigure}
	\hspace{\hecmgap}
	\begin{subfigure}[b]{\mygapapp\textwidth}
		\centering
		\includegraphics[width=\mygapappin\textwidth]{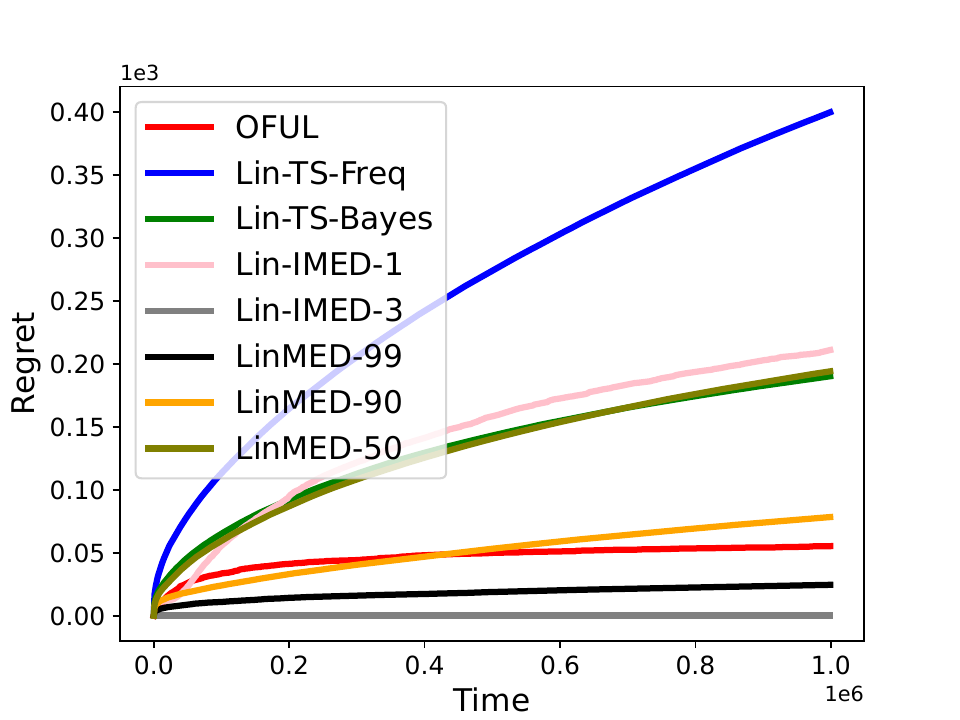}
		\caption{$\eps = 0.02$, $\sigma^2 = 2\cdot \sigma^2_* $}
	\end{subfigure}
	\hspace{\hecmgap}
	\begin{subfigure}[b]{\mygapapp\textwidth}
	\centering
	\includegraphics[width=\mygapappin\textwidth]{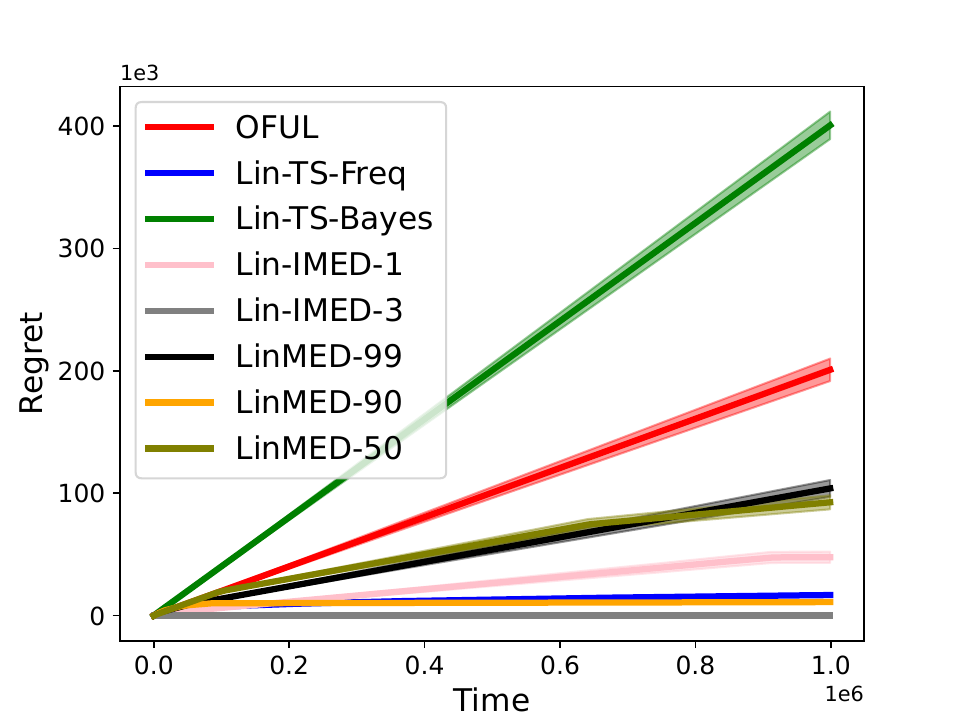}
	\caption{$\sigma^2 = 0.1 \cdot \sigma^2_* $}
	\end{subfigure}\hspace{\hecmmar}
	\caption{End of optimism experiments $\eps = 0.02$ }
	\label{app-figure:end-of-opt-figure3}
\end{figure*}
 \textbf{Remarks :} This experiment has been discussed already in the main body of the paper for  $\sigma^2 = \sigma^2_* $ and  $\sigma^2 = 2\cdot \sigma^2_* $. In addition to that, we observe the following: when the noise is under-specified ($\sigma^2 = 0.1\cdot \sigma^2_* $), the performance of both OFUL and Lin-TS-Bayes degrades significantly, resulting in nearly linear regret. On the other hand, Lin-TS-Freq performs well due to the fact that the degree of oversampling is lesser due to small $\sigma^2$. Among the LinMED variants, LinMED-99 shows a more pronounced deterioration compared to LinMED-90 and LinMED-50, indicating that higher exploration is required in such scenarios. Notably, Lin-IMED-3 maintains strong performance across all conditions.

\subsection{Delayed reward experiments on real-world data set}\label{app-subsection:delayed-reward-exp-subsection}

One of the advantages of randomized algorithms over deterministic algorithms is their superior performance in delayed reward settings, where immediate rewards are not accessible. To investigate this, we utilized the MovieLens real-world dataset. We extracted user and movie features and constructed our own true parameter $\theta^*$ for the reward calculation, as opposed to relying on the reported rewards in the dataset. This approach was necessary due to the sparsity of the reported rewards, where many users not providing ratings for numerous movies. Additionally, we performed Principal Component Analysis (PCA) to isolate the dominant features from both the movie and user datasets. Then we generated 2 dimensional movie feature vectors and 2 dimensional user feature vectors.

During the implementation we first fix randomly chosen $K=10$ movies and for each time step, we randomly select a user as the context and generate arm set of size $K = 10$ by taking outer product of each movie vectors with the user vector. Since we calculated the outer product between 2 dimensional user vector and 2 dimensional movie vector, the resulting dimension of the arm set is $d=4$. We set the noise parameter to $\sigma^2 = \sigma_{*}^2 =  1$ and varied the delay time across the set  $\{0,10,20\}$. The time horizon for each trial is $n=5,000$ and conduct $100$ such independent trials.

\textbf{Algorithms evaluated: }We evaluate the following algorithms: OFUL~\citep{ay11improved}, Lin-TS-Bayes (Thompson sampling Bayesian version)~\citep{russo14learning}, LinMED-99 ($\alpha_{\mathrm{opt}} = 0.99$), LinMED-90 ($\alpha_{\mathrm{opt}} = 0.90$), and LinMED-50 ($\alpha_{\mathrm{opt}} = 0.50$).  Note that, for Lin-TS-Bayes, we are still evaluating the frequentest regret as we do for every other algorithms.
\def\mygapapp{0.5}
\begin{figure*}[h]
	\centering
	\hspace{\hecmmar}\begin{subfigure}[b]{\mygapapp\textwidth}
		\centering
		\includegraphics[width=\mygapappin\textwidth]{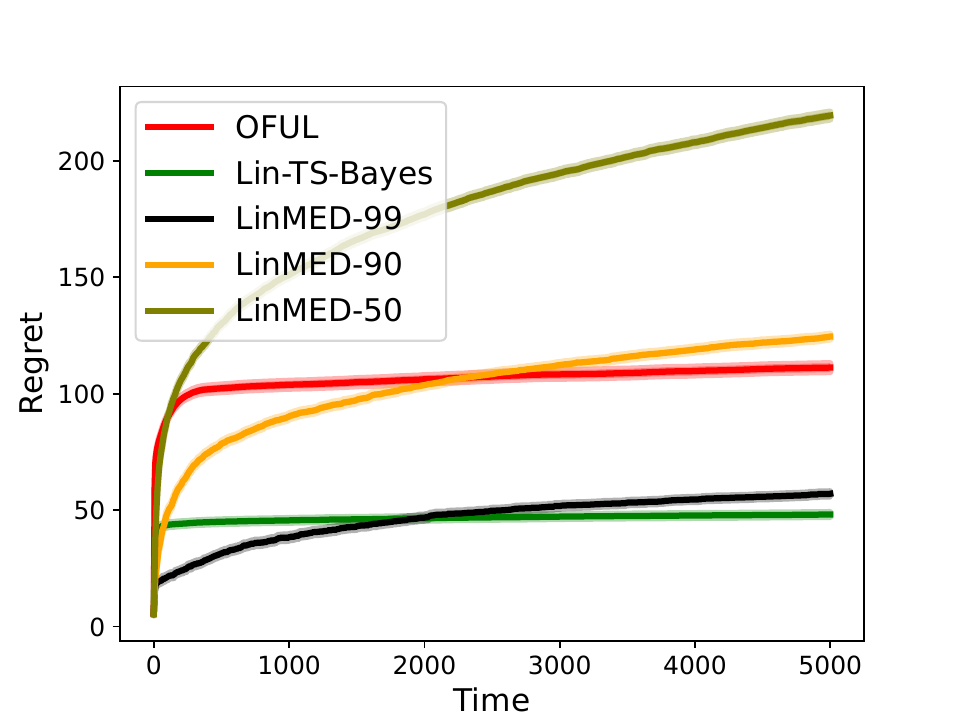}
		\caption{Delay = 0}
	\end{subfigure}
	\hspace{\hecmgap}
	\begin{subfigure}[b]{\mygapapp\textwidth}
		\centering
		\includegraphics[width=\mygapappin\textwidth]{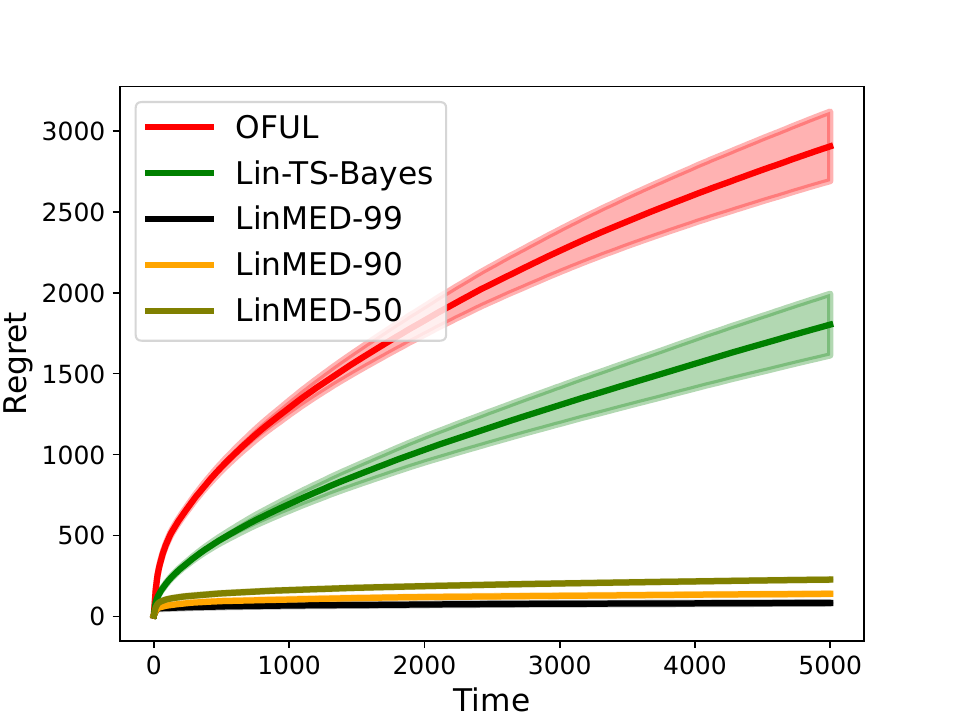}
		\caption{Delay = 10}
	\end{subfigure}
	\hspace{\hecmgap}
	\begin{subfigure}[b]{\mygapapp\textwidth}
		\centering
		\includegraphics[width=\mygapappin\textwidth]{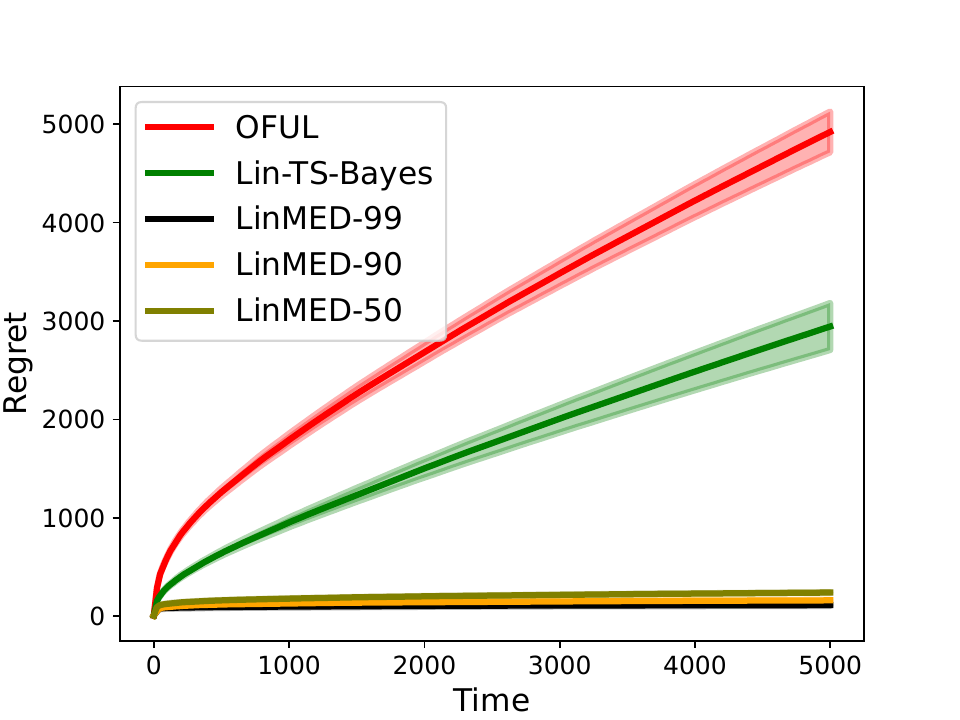}
		\caption{Delay = 20}
	\end{subfigure}\hspace{\hecmmar}
	\caption{Delayed reward experiments}
	\label{app-figure:DRE-figure}
\end{figure*}

As expected, the performance of OFUL significantly deteriorates when rewards are delayed, in contrast to the randomized algorithms, LinMED and Lin-TS-Bayes. Moreover, all three variants of LinMED demonstrate strong performance under these conditions.
\subsection{Offline evaluation experiments}\label{app-subsection:offline-eval-exp-subsection}
	
This section presents our simulation results on offline evaluation using logged data. We utilize the logged data generated by our algorithms, LinMED and Lin-TS-Freq (frequentest version), to estimate the expected reward of a policy (call it Uniform target policy) that selects arms uniformly at random from $\mathcal{A}$ (we use fixed arm set for this experiment). The logged data takes the form $(A_t, p_t(A_t) , Y_t)_{t=0}^{n}$, where $A_t$  represents the chosen arm,  $p_t(A_t)$ denotes the probability (either exact or approximate) of selecting that arm, and $Y_t$ indicates the received reward at time step $t$. We consider the Inverse Propensity Weighting (IPW) estimator to estimate the cumulative reward of the Uniform target policy as follows:
\begin{align}
	\text{IPW score} = \frac{1}{n} \sum_{t=1}^{n} \frac{\frac{1}{\lvert \mathcal{A} \rvert }}{p_t(A_t)}\cdot Y_t
\end{align}
LinMED assigns a closed-form probability to the chosen arm, whereas  Lin-TS-Freq estimates the probability of selecting an arm using Monte Carlo trials. We set the number of arms $K=2$, dimension $d=2$ and $\mathcal{A} = \{a_1 = (1,0)^\T, a_2 = (0.6,0.8)^\T \}$, $\theta^* = (1,0)^\T$ while varying the number of Monte Carlo samples for estimating the probability of arm selection in Lin-TS-Freq across the set $\{10^3,10^4,10^5\}$. The noise is modeled as $\mathcal{N}(0, \sigma_*^2)$, with $\sigma^2 = \sigma_*^2 = 0.1$. Throughout this experiment, we evaluate LinMED-50 ($\alpha_{\mathrm{opt}} = 0.5$). The time horizon for each trial is $n=1,000$ and conduct $5,000$ such independent trials to calculate the histogram representation of IPW scores. See Figure \ref{app-figure:OLE-figure}. 
\def\mygapapp{0.5}
\begin{figure*}[h]
	\centering
	\hspace{\hecmmar}\begin{subfigure}[b]{\mygapapp\textwidth}
		\centering
		\includegraphics[width=\mygapappin\textwidth]{OLE_1_CF_CCR.pdf}
		\caption{Monte-carlo samples = $10^3$}
	\end{subfigure}
	\hspace{\hecmgap}
	\begin{subfigure}[b]{\mygapapp\textwidth}
		\centering
		\includegraphics[width=\mygapappin\textwidth]{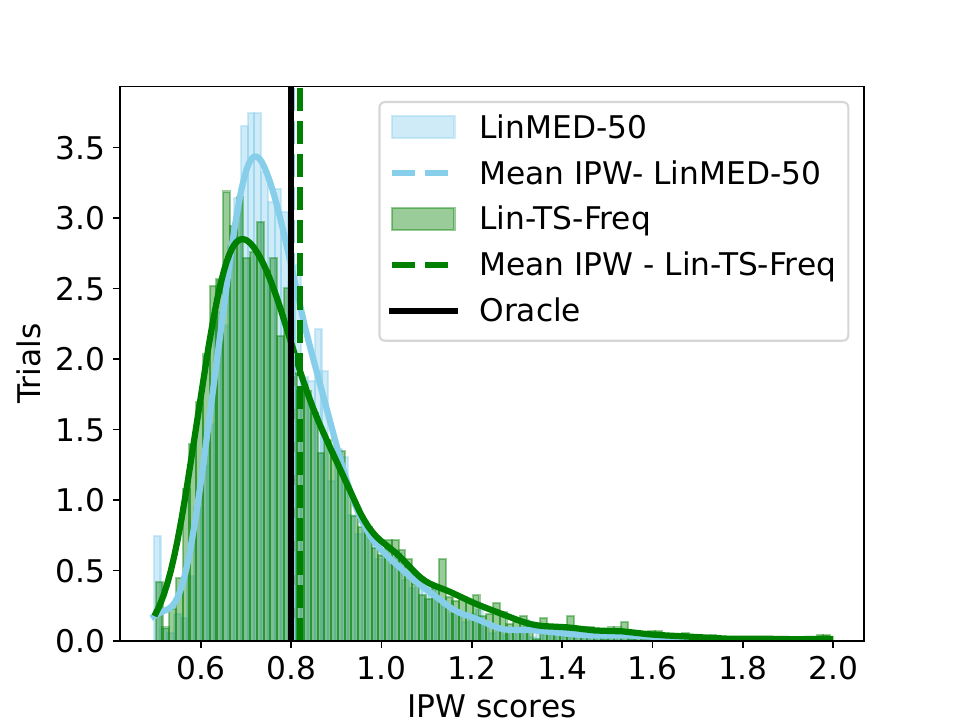}
		\caption{Monte-carlo samples = $10^4$}
	\end{subfigure}
	\hspace{\hecmgap}
	\begin{subfigure}[b]{\mygapapp\textwidth}
		\centering
		\includegraphics[width=\mygapappin\textwidth]{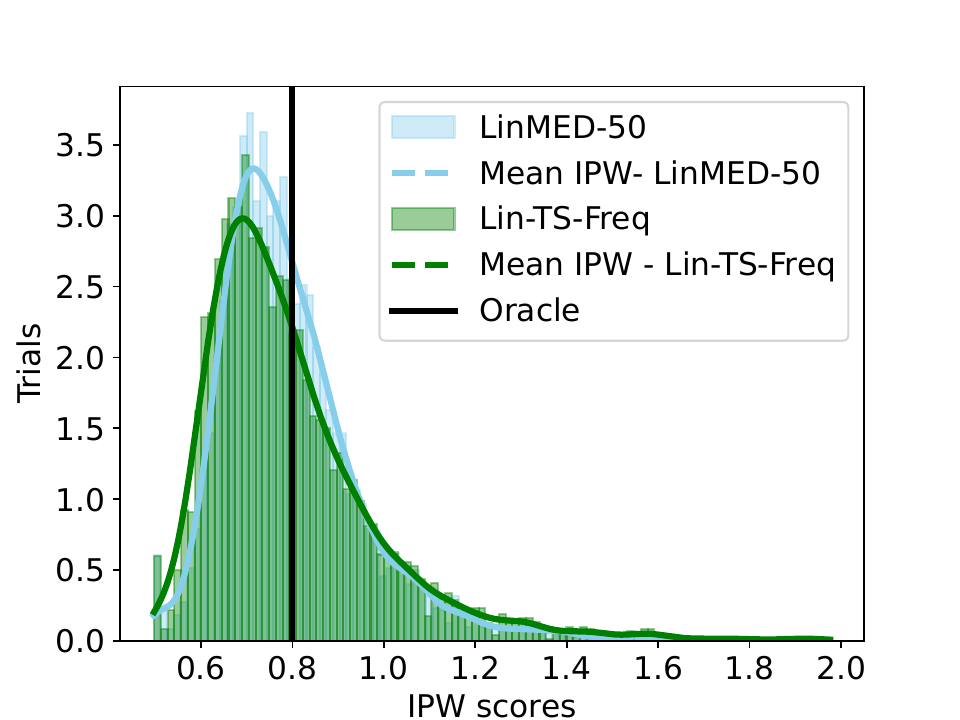}
		\caption{Monte-carlo samples = $10^5$}
	\end{subfigure}\hspace{\hecmmar}
	\caption{Offline evaluation experiments}
	\label{app-figure:OLE-figure}
\end{figure*}

\begin{table}
	\centering
		\begin{tabular}{||c c c c c||} 
			\hline
			Lin-TS-Feq & $M = 10^3 $ & $M = 10^4 $ & $M = 10^5 $ & Oracle \\ [0.5ex] 
			\hline\hline
			Mean & 0.906 & 0.819 & 0.799 & 0.800\\ 
			\hline
			Standard deviation & 0.099 & 0.069 & 0.039 & 0\\
			\hline
		\end{tabular}
		\caption{ Mean and standard deviation of rewards received by the Uniform policy using logged data from Lin-TS-Freq for offline evaluation. Here M stands for number of Monte-carlo samples}
		\label{app-table:LinTS-OLE-table}
\end{table}

In Figure \ref{app-figure:OLE-figure}, the oracle value of $0.8$ represents the expected reward of the Uniform policy when real-time data is used. The mean reward received by the Uniform policy using logged data from LinMED-50 for offline evaluation matches the oracle value to a two-decimal place accuracy, with a standard deviation of approximately $0.03$. The mean and standard deviation of Lin-TS-Freq are provided in Table \ref{app-table:LinTS-OLE-table}. Although the performance of Lin-TS-Freq approaches that of LinMED-50 when the number of Monte Carlo samples reaches $10^5$, there is a significant bias for sample sizes of $10^3$ and $10^4$. However, using $10^5$ samples for estimating probabilities is computationally expensive. Additionally, there are cases where Lin-TS-Freq assigns zero probability to some arms. Due to LinMED's ability to assign a closed-form probability to each arm, LinMED is more suitable than Lin-TS-Freq for offline evaluation.

\subsection{Synthetic unit ball arm set experiments}\label{app-subsection:spher-exp-subsection}

\subsubsection{Fixed number of arms (K), different dimensions (d)}
\textbf{Experimental setup: } We fix the number of arms $K=10$. For different $d \in \{2,20,50\}$, we randomly sample $K = 10$ arms from $d$ dimensional unit ball $S^{d-1}$. The noise follows $\mathcal{N}(0, \sigma_*^2)$ with $\sigma_*^2 = 1$.  The time horizon for each trial is $n=5,000$ and conduct $50$ such independent trials. Furthermore, we conduct experiments for the cases i) $\sigma^2 = \sigma_*^2$, ii) $\sigma^2 = 2\cdot \sigma_*^2$, and iii) $\sigma^2 = 0.1 \cdot \sigma_*^2$.

\textbf{Algorithms evaluated: }We evaluate the following algorithms: OFUL~\citep{ay11improved}, Lin-TS-Freq (Thompson sampling frequentest version)~\citep{agrawal14thompson}, Lin-TS-Bayes (Thompson sampling Bayesian version)~\citep{russo14learning}, Lin-IMED-1~\citep{bian24indexed}, Lin-IMED-3~\citep{bian24indexed}, LinMED-99 ($\alpha_{\mathrm{opt}} = 0.99$), LinMED-90 ($\alpha_{\mathrm{opt}} = 0.90$), LinMED-50 ($\alpha_{\mathrm{opt}} = 0.50$), SpannerIGW~\citep{zhu22contextual}, and SpannerIGW-AT~\citep{zhu22contextual}. Note that, for Lin-TS-Bayes, we are still evaluating the frequentest regret as we do for every other algorithms. Lin-IMED-3 have a hyper-parameter $C$, which we set to $C= 30$ following~\citet{bian24indexed}. Moreover SpannerIGW-AT is the anytime version of SpannerIGW. 

\def\mygapapp{0.5}
\def \hecmmar{-2cm}
\def \hecmgap{-4cm}
\begin{figure*}[h!]
	\centering
	\hspace{\hecmmar}\begin{subfigure}[b]{\mygapapp\textwidth}
		\centering
		\includegraphics[width=\mygapappin\textwidth]{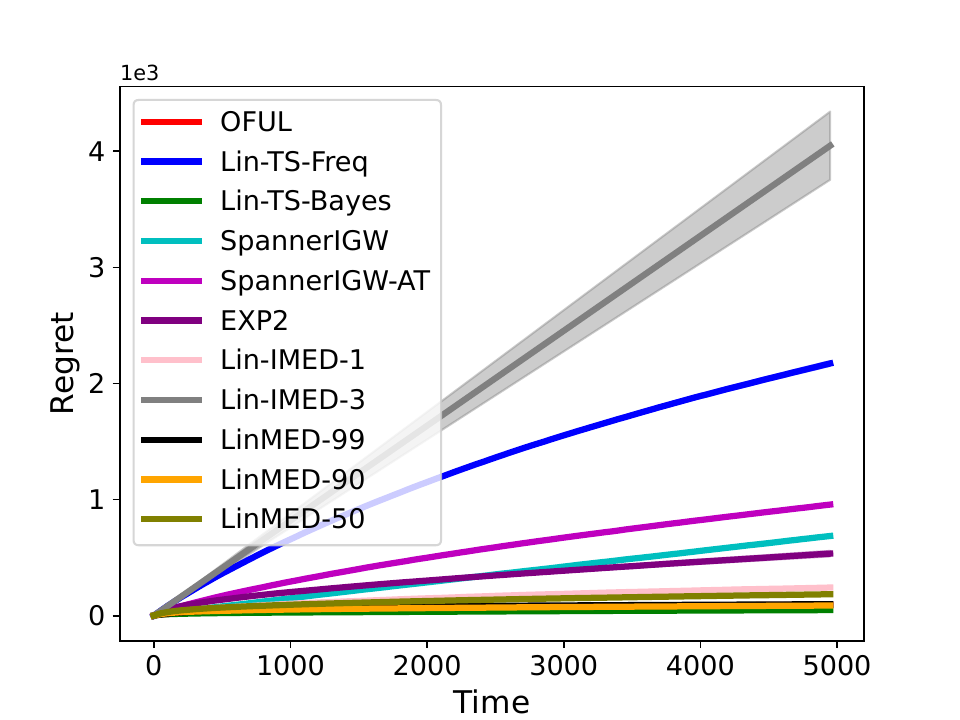}
		\caption{$d = 2$}
	\end{subfigure}
	\hspace{\hecmgap}
	\begin{subfigure}[b]{\mygapapp\textwidth}
		\centering
		\includegraphics[width=\mygapappin\textwidth]{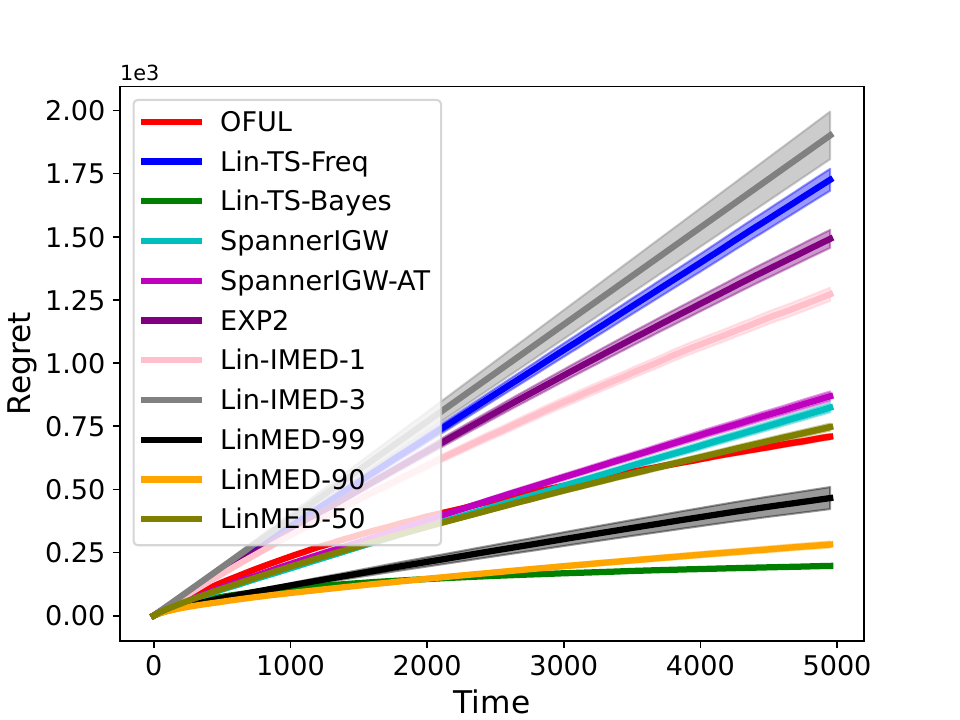}
		\caption{$ d = 20 $}
	\end{subfigure}
	\hspace{\hecmgap}
	\begin{subfigure}[b]{\mygapapp\textwidth}
		\centering
		\includegraphics[width=\mygapappin\textwidth]{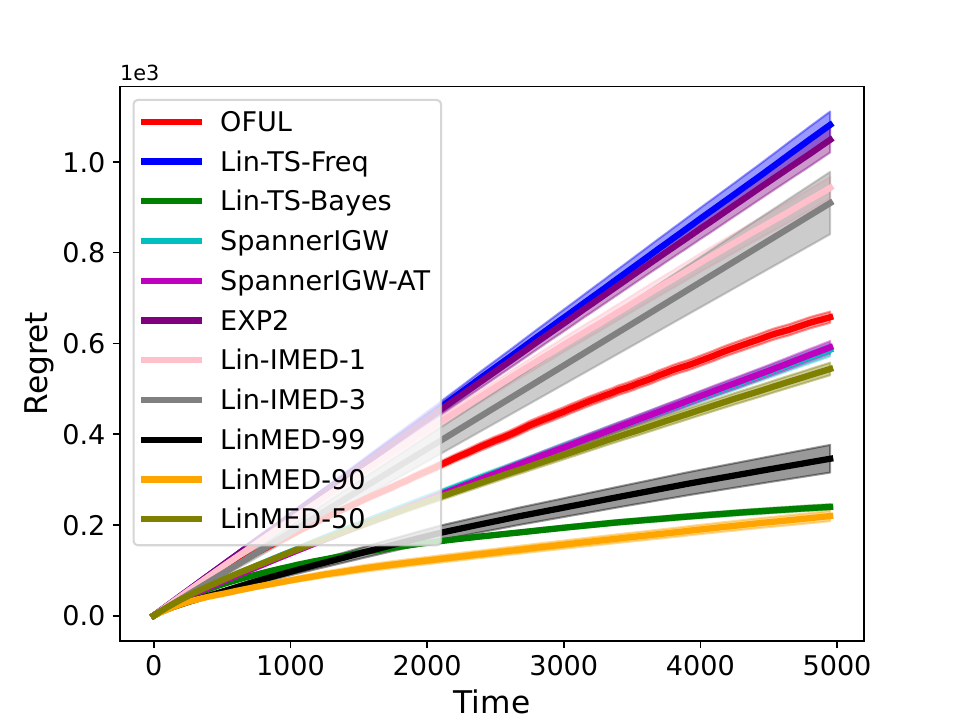}
		\caption{ $ d = 50 $}
	\end{subfigure}\hspace{\hecmmar}
	\caption{ $\sigma^2 = \sigma^2_* $ }
	\label{app-figure:SUBE-figure1}
\end{figure*}
\begin{figure*}[h!]
	\centering
	\hspace{\hecmmar}\begin{subfigure}[b]{\mygapapp\textwidth}
		\centering
		\includegraphics[width=\mygapappin\textwidth]{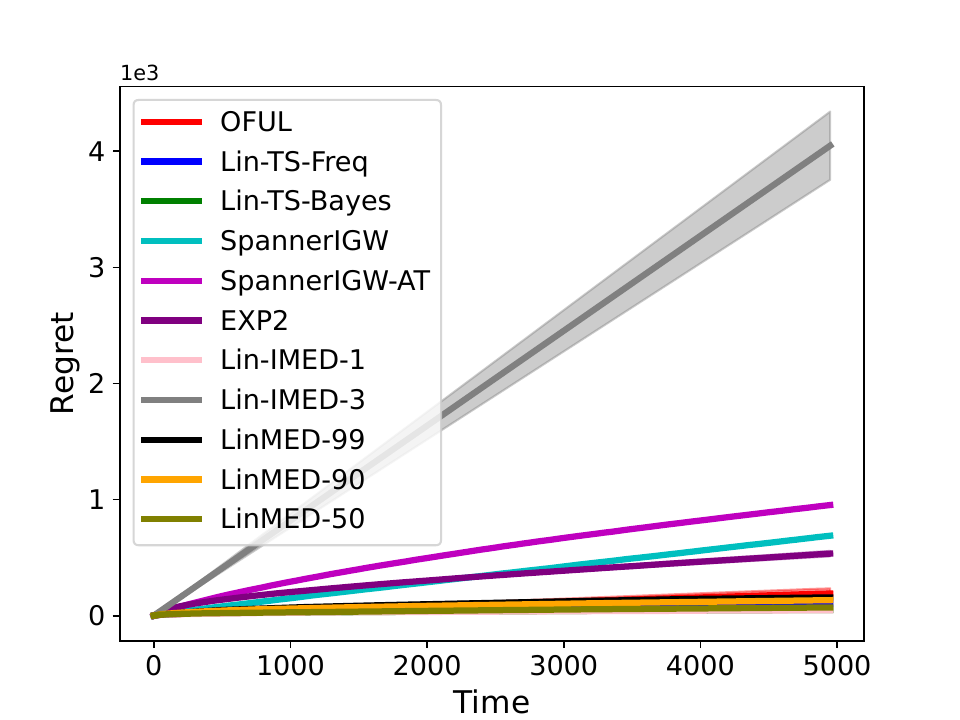}
		\caption{$ d = 2 $}
	\end{subfigure}
	\hspace{\hecmgap}
	\begin{subfigure}[b]{\mygapapp\textwidth}
		\centering
		\includegraphics[width=\mygapappin\textwidth]{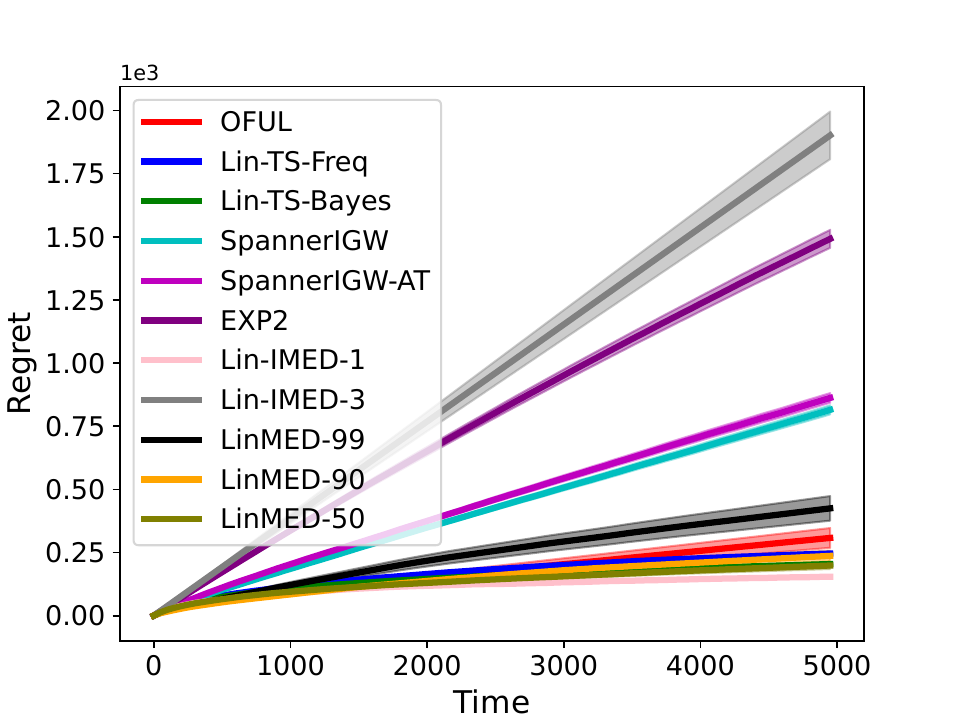}
		\caption{$ d = 20 $}
	\end{subfigure}
	\hspace{\hecmgap}
	\begin{subfigure}[b]{\mygapapp\textwidth}
		\centering
		\includegraphics[width=\mygapappin\textwidth]{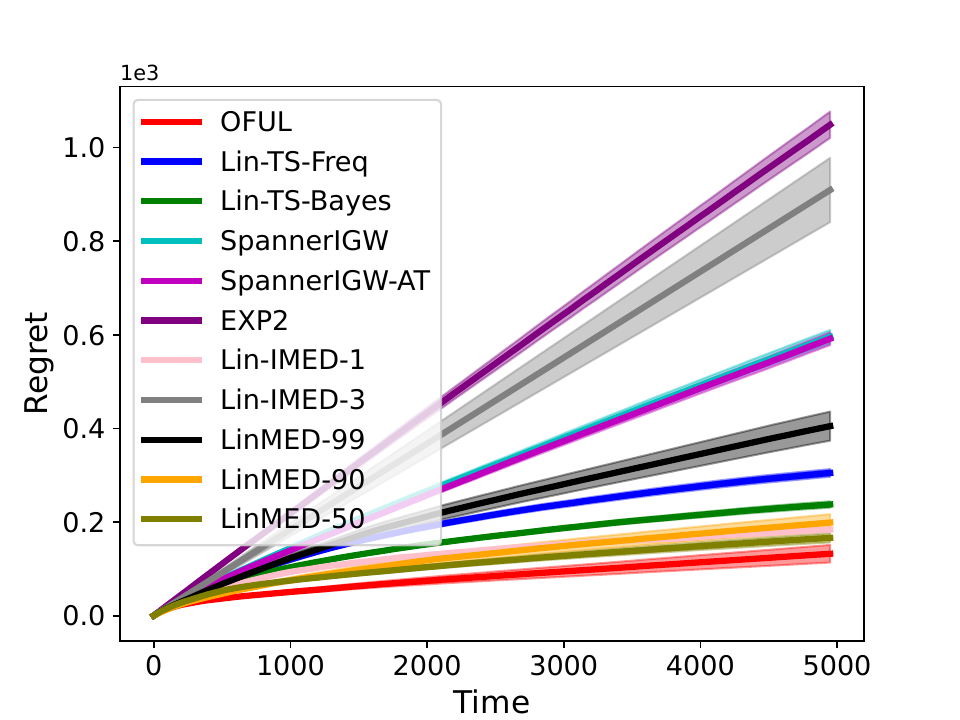}
		\caption{$ d = 50 $}
	\end{subfigure}\hspace{\hecmmar}
	\caption{$\sigma^2 = 0.1 \cdot \sigma^2_* $ }
	\label{app-figure:SUBE-figure2}
\end{figure*}
\begin{figure*}[h!]
	\centering
	\hspace{\hecmmar}\begin{subfigure}[b]{\mygapapp\textwidth}
		\centering
		\includegraphics[width=\mygapappin\textwidth]{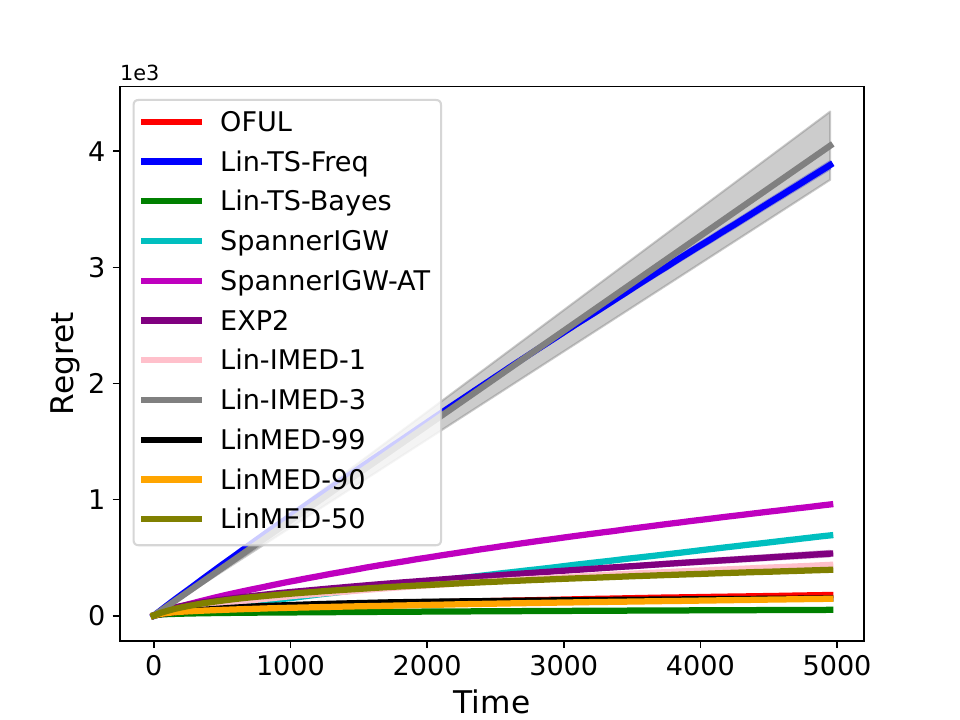}
		\caption{$ d = 2 $}
	\end{subfigure}
	\hspace{\hecmgap}
	\begin{subfigure}[b]{\mygapapp\textwidth}
		\centering
		\includegraphics[width=\mygapappin\textwidth]{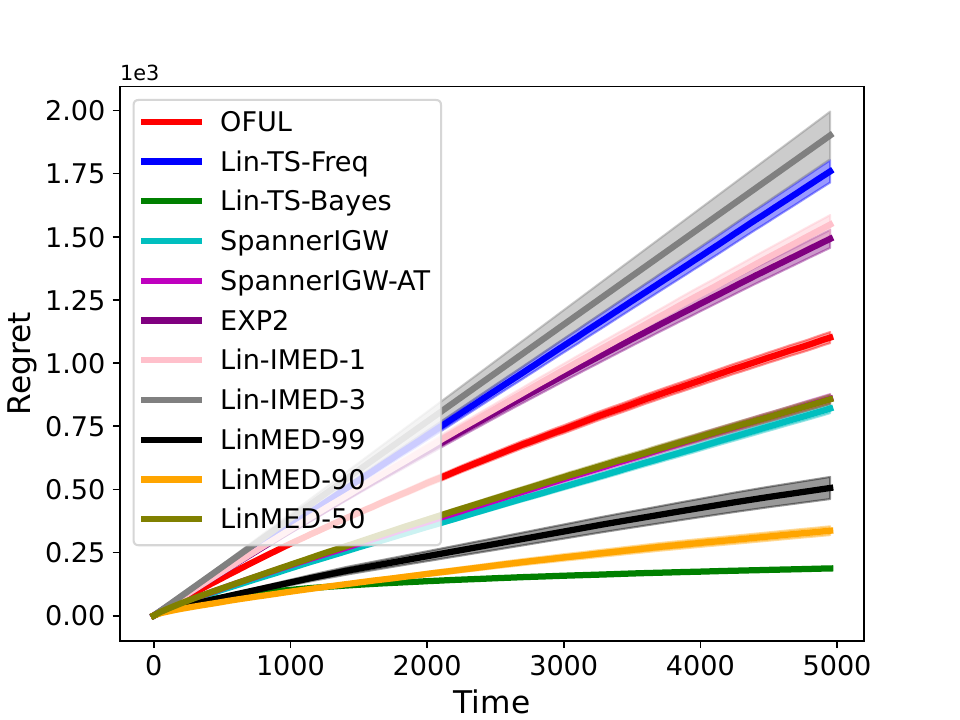}
		\caption{$ d = 20 $}
	\end{subfigure}
	\hspace{\hecmgap}
	\begin{subfigure}[b]{\mygapapp\textwidth}
		\centering
		\includegraphics[width=\mygapappin\textwidth]{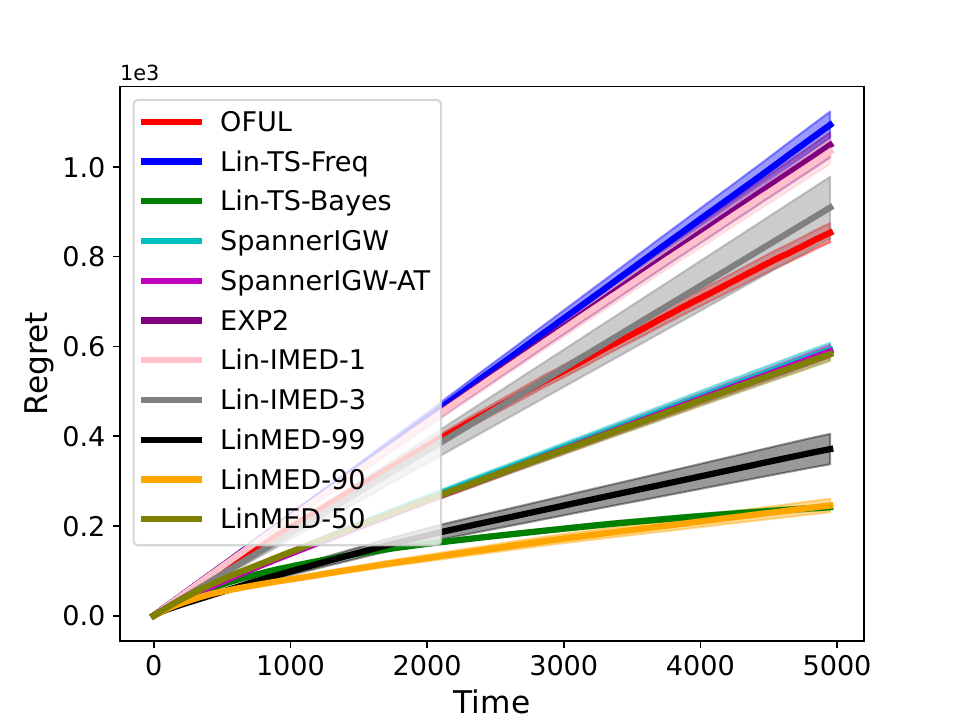}
		\caption{$ d = 50 $}
	\end{subfigure}\hspace{\hecmmar}
	\caption{$\sigma^2 = 2 \cdot \sigma^2_* $ }
	\label{app-figure:SUBE-figure3}
\end{figure*}

\textbf{Remarks: } Firstly, Lin-IMED-3, which demonstrated strong performance in most of the prior experiments, exhibits notably poor performance in this particular experiment. This decline could potentially be attributed to the choice of hyperparameter setting, specifically $C=30$. However, it is reasonable to retain this setting, as the same value was consistently applied in the previous experiments.

Secondly, although the performance of all algorithms deteriorates with increasing dimensionality, the decline in Lin-TS-Freq is particularly pronounced due to the $\sqrt{d}$ oversampling, a factor also reflected in its theoretical regret guarantee. This downward trend in Lin-TS-Freq becomes more severe when the sub-Gaussian parameter of the noise is over-specified. However, an improvement in performance is observed when the sub-Gaussian parameter of the noise is under-specified. The latter trend is expected, as the oversampling rate of Lin-TS-Freq  grows with $\sigma^2$.

OFUL performs well except when the sub-Gaussian parameter of the noise is over-specified.

All variants of LinMED perform competitively compared to other algorithms. Notably, the performance of LinMED is not significantly affected by noise misspecifications. Moreover, LinMED-90 outperforms LinMED-99 in high-dimensional contexts, suggesting that a higher degree of exploration is essential when dealing with such settings.

\subsubsection{Fixed dimension (d), different numbers of arms (K)}
\textbf{Experimental setup: } We fix the dimension $d=2$. For different $K \in \{10,100,500\}$, we randomly samples $K$ arms from $d$ dimensional unit ball $S^{d-1}$. The noise follows $\mathcal{N}(0, \sigma_*^2)$ with $\sigma_*^2 = 1$.  The time horizon for each trial is $n=5000$ and conduct $50$ such independent trials. Furthermore, we conduct experiments for the cases i) $\sigma^2 = \sigma_*^2$, ii) $\sigma^2 = 2\cdot \sigma_*^2$, and iii) $\sigma^2 = 0.1 \cdot \sigma_*^2$.

\textbf{Algorithms evaluated: }We evaluate the following algorithms: OFUL~\citep{ay11improved}, Lin-TS-Freq (Thompson sampling frequentest version)~\citep{agrawal14thompson}, Lin-TS-Bayes (Thompson sampling Bayesian version)~\citep{russo14learning}, Lin-IMED-1~\citep{bian24indexed}, Lin-IMED-3~\citep{bian24indexed}, LinMED-99 ($\alpha_{\mathrm{opt}} = 0.99$), LinMED-90 ($\alpha_{\mathrm{opt}} = 0.90$), LinMED-50 ($\alpha_{\mathrm{opt}} = 0.50$), SpannerIGW~\citep{zhu22contextual}, and SpannerIGW-AT~\citep{zhu22contextual}. Note that, for Lin-TS-Bayes, we are still evaluating the frequentest regret as we do for every other algorithms. Lin-IMED-3 have a hyper-parameter $C$, which we set to $C= 30$ following~\citet{bian24indexed}. Moreover SpannerIGW-AT is the anytime version of SpannerIGW.

\def\mygapapp{0.5}
\def \hecmmar{-2cm}
\def \hecmgap{-4cm}
\begin{figure*}[h]
	\centering
	\hspace{\hecmmar}\begin{subfigure}[b]{\mygapapp\textwidth}
		\centering
		\includegraphics[width=\mygapappin\textwidth]{SUBED2K10_CF_CCR.pdf}
		\caption{$K = 10$}
	\end{subfigure}
	\hspace{\hecmgap}
	\begin{subfigure}[b]{\mygapapp\textwidth}
		\centering
		\includegraphics[width=\mygapappin\textwidth]{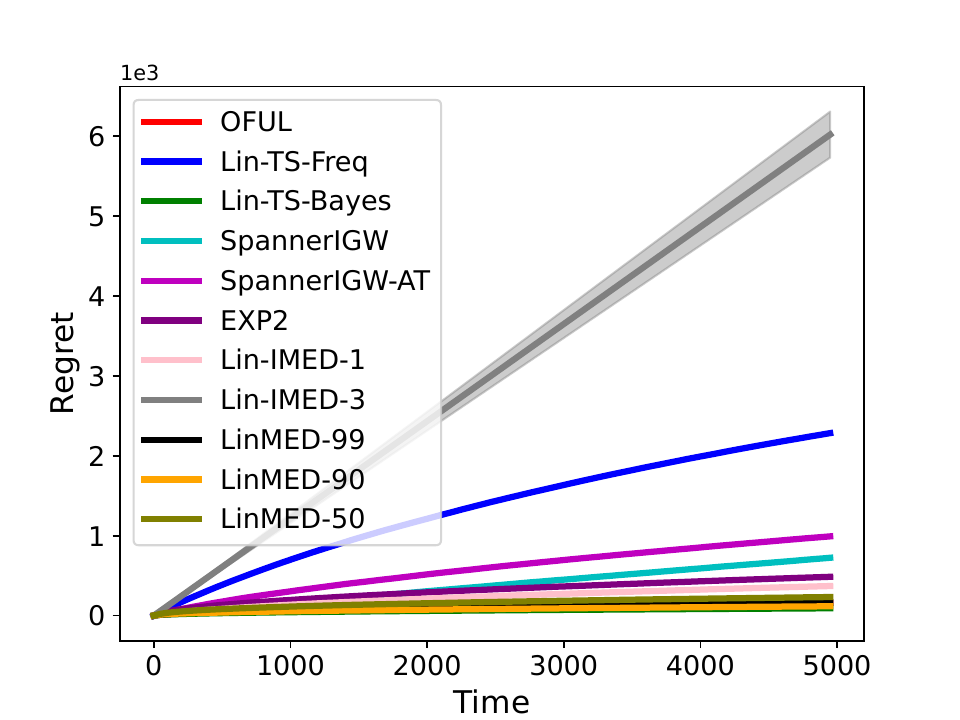}
		\caption{$ K = 100 $}
	\end{subfigure}
	\hspace{\hecmgap}
	\begin{subfigure}[b]{\mygapapp\textwidth}
		\centering
		\includegraphics[width=\mygapappin\textwidth]{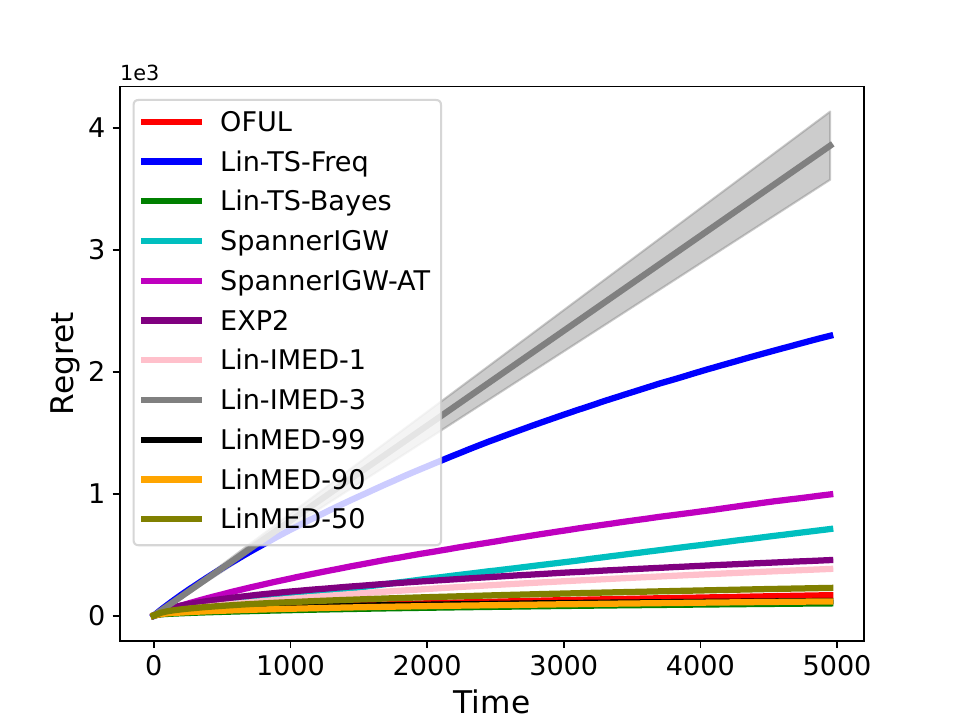}
		\caption{ $ K = 500 $}
	\end{subfigure}\hspace{\hecmmar}
	\caption{ $\sigma^2 = \sigma^2_* $ }
	\label{app-figure:SUBE-figure4}
\end{figure*}
\begin{figure*}[h]
	\centering
	\hspace{\hecmmar}\begin{subfigure}[b]{\mygapapp\textwidth}
		\centering
		\includegraphics[width=\mygapappin\textwidth]{SUBED2K10_US_CF_CCR.pdf}
		\caption{$K = 10$}
	\end{subfigure}
	\hspace{\hecmgap}
	\begin{subfigure}[b]{\mygapapp\textwidth}
		\centering
		\includegraphics[width=\mygapappin\textwidth]{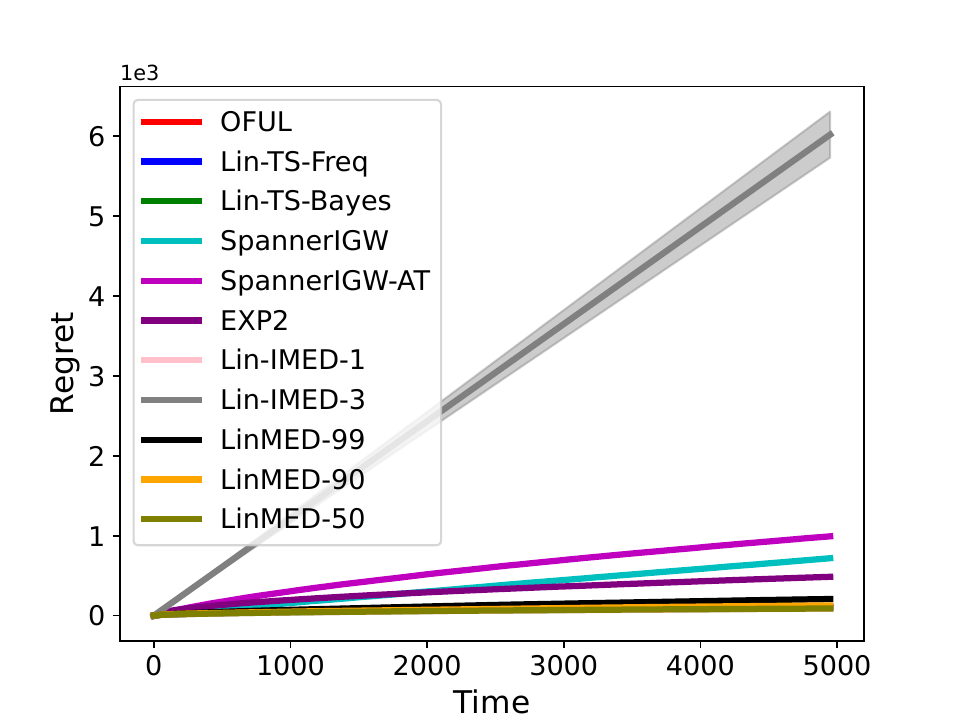}
		\caption{$ K = 100 $}
	\end{subfigure}
	\hspace{\hecmgap}
	\begin{subfigure}[b]{\mygapapp\textwidth}
		\centering
		\includegraphics[width=\mygapappin\textwidth]{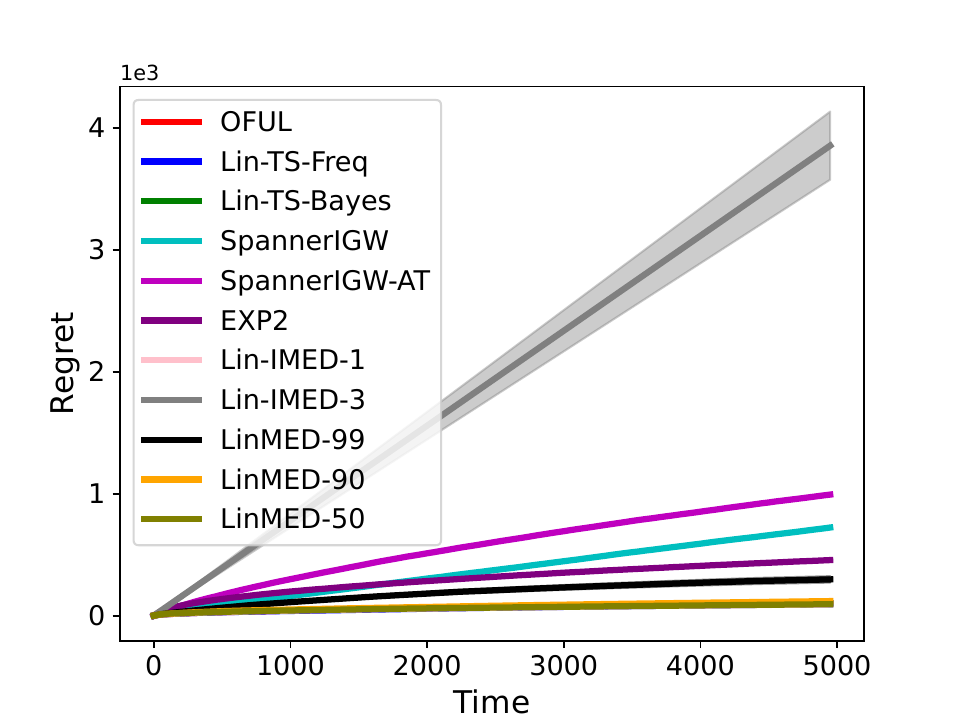}
		\caption{$ K = 500 $}
	\end{subfigure}\hspace{\hecmmar}
	\caption{$\sigma^2 = 0.1 \cdot \sigma^2_* $ }
	\label{app-figure:SUBE-figure5}
\end{figure*}
\begin{figure*}[h]
	\centering
	\hspace{\hecmmar}\begin{subfigure}[b]{\mygapapp\textwidth}
		\centering
		\includegraphics[width=\mygapappin\textwidth]{SUBED2K10_OS_CF_CCR.pdf}
		\caption{$K = 10$}
	\end{subfigure}
	\hspace{\hecmgap}
	\begin{subfigure}[b]{\mygapapp\textwidth}
		\centering
		\includegraphics[width=\mygapappin\textwidth]{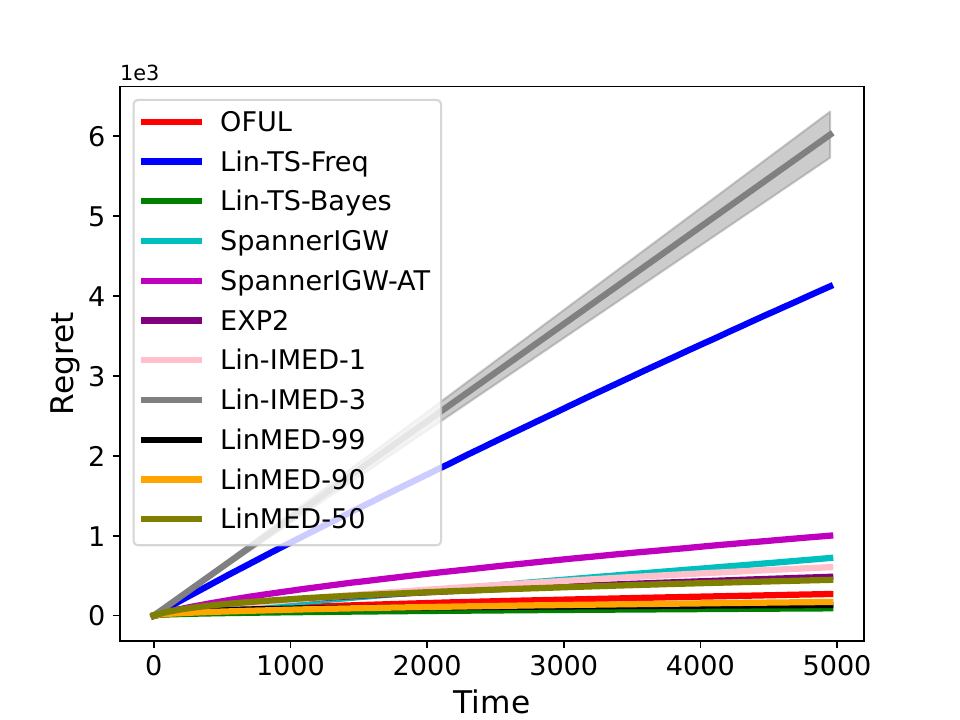}
		\caption{$ K = 100 $}
	\end{subfigure}
	\hspace{\hecmgap}
	\begin{subfigure}[b]{\mygapapp\textwidth}
		\centering
		\includegraphics[width=\mygapappin\textwidth]{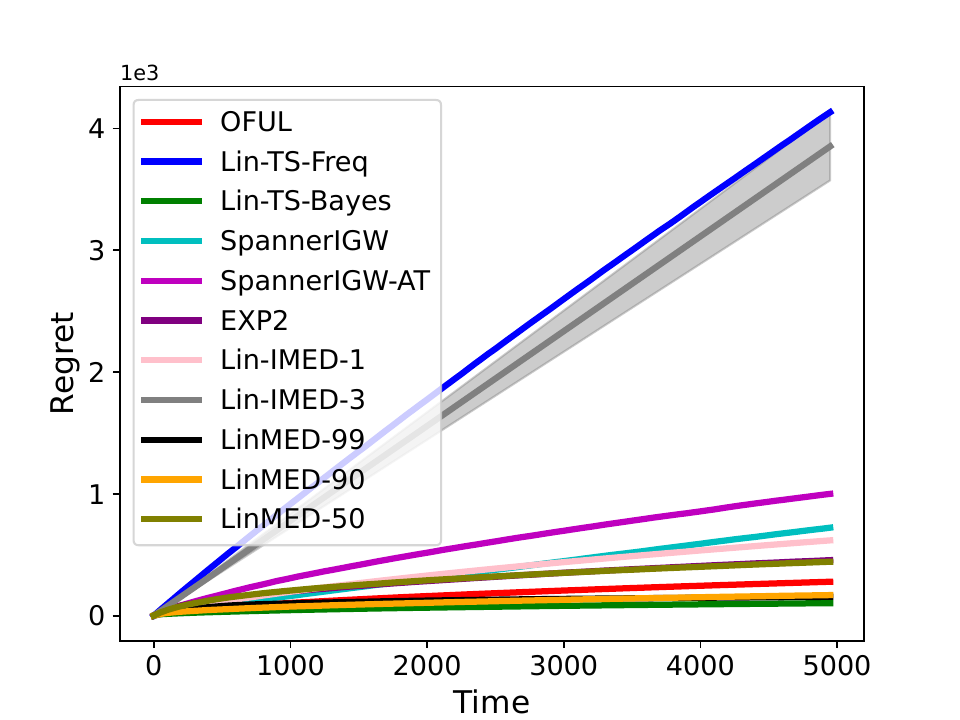}
		\caption{$ K = 500 $}
	\end{subfigure}\hspace{\hecmmar}
	\caption{$\sigma^2 = 2 \cdot \sigma^2_* $ }
	\label{app-figure:SUBE-figure6}
\end{figure*}

\textbf{Remarks: } Similar to the fixed $K$ setting we analyzed previously, Lin-IMED-3 demonstrates noticeably poor performance. Additionally, the performance of Lin-TS-Freq deteriorates when the sub-Gaussian parameter of the noise is over-specified. All variants of LinMED perform competitively compared to other algorithms. Moreover, across most algorithms, there are no substantial variations in performance with respect to $K$, suggesting that the regret does not exhibit dependency on $K$.

\end{document}